\documentclass[11pt]{article}

\usepackage{microtype}
\usepackage{graphicx}
\usepackage{booktabs} 

\usepackage{balance}
\usepackage{comment}
\usepackage{wrapfig}
\usepackage[margin=1in, letterpaper]{geometry}
\usepackage[utf8]{inputenc}
\usepackage{algorithm}
\usepackage{algorithmicx}
\algrenewcommand\algorithmicindent{1.0em}
\usepackage[noend]{algpseudocode}
\usepackage{amsmath,amssymb,amsthm,float,caption,mathtools,bbold}
\usepackage{enumerate}
\usepackage{enumitem}
\usepackage{graphicx}
\usepackage{caption}
\usepackage{subcaption}
\usepackage{wrapfig}
\usepackage{breqn}

\newtheorem{theorem}{Theorem}
\newtheorem*{theorem*}{Theorem}
\newtheorem{lemma}{Lemma}
\newtheorem{definition}{Definition}
\newtheorem*{definition*}{Definition}
\newtheorem*{lemma*}{Lemma}
\newtheorem{corollary}{Corollary}
\newtheorem*{corollary*}{Corollary}

\newtheorem*{claim*}{Claim}

\newtheorem*{proposition*}{Proposition}
\newtheorem{assumption}{Assumption}
\usepackage{hyperref}

\usepackage[T1]{fontenc}    
\usepackage{url}            
\usepackage{booktabs}       
\usepackage{amsfonts}       
\usepackage{nicefrac}       
\usepackage{microtype}      
\usepackage{xcolor}         

\title{Recoverability Landscape of Tree Structured Markov Random Fields under Symmetric Noise}

\author{Ashish Katiyar \\ \href{mailto:a.katiyar@utexas.edu}{a.katiyar@utexas.edu} \\
The University of Texas at Austin\\
\and 
Soumya Basu\\ \href{mailto:basusoumya@utexas.edu}{basusoumya@utexas.edu}\\
The University of Texas at Austin\\
\and  Vatsal Shah\\ \href{mailto:vatsalshah1106@utexas.edu}{vatsalshah1106@utexas.edu}\\
The University of Texas at Austin\\
\and Constantine Caramanis\\ \href{mailto:constantine@utexas.edu}{constantine@utexas.edu}\\
The University of Texas at Austin\\}

\begin{document}

\maketitle



\begin{abstract}
 We study the problem of learning tree-structured Markov random fields (MRF) on discrete random variables with common support when the observations are corrupted by a $k$-ary symmetric noise channel with unknown probability of error. 
 For Ising models (support size = 2), past work has shown that graph structure can only be recovered up to the leaf clusters (a leaf node, its parent, and its siblings form a leaf cluster) and exact recovery is impossible. No prior work has addressed the setting of support size of 3 or more, and indeed this setting is far richer. As we show, when the support size is 3 or more, the structure of the leaf clusters may be partially or fully identifiable. We provide a precise characterization of this phenomenon and show that the extent of recoverability is dictated by the joint PMF of the random variables.  In particular, we provide necessary and sufficient conditions for exact recoverability. Furthermore, we present a polynomial time, sample efficient algorithm that recovers the exact tree when this is possible, or up to the unidentifiability as promised by our characterization, when full recoverability is impossible. Finally, we  demonstrate the efficacy of our algorithm experimentally.
\end{abstract}

\newcommand{\ind}[4]{{#1}_{#2}^{#3, #4}}
\newcommand{\bs}[1]{\boldsymbol{#1}}
\newcommand{\E}[1]{\mathbb{E}{\left[#1\right]}}
\newcommand{\tempa}[2]{\mathbb{E}\left[{#1}_{#2}\right]}
\newcommand{\T}{\mathcal{T}_{T^*}}
\newcommand{\fa}{\textsc{FindEC}}
\newcommand{\fb}{\textsc{SplitTree}}
\newcommand{\fc}{\textsc{Recurse}}
\newcommand{\setx}{\mathcal{X}}
\newcommand{\sets}{\mathcal{S}}
\newcommand{\setl}{\mathcal{L}}
\newcommand{\sete}{\mathcal{E}}
\newcommand{\setp}[1]{\mathcal{P}_{i}^{#1}}
\newcommand{\Or}{\mathcal{O}}
\newcommand{\es}{\mathcal{E}^*}
\newcommand{\ep}{\mathcal{E}'}
\newcommand{\ee}{\mathcal{E}^e}
\newcommand{\corr}[2]{\rho_{#1, #2}}
\newcommand{\cov}[2]{\Sigma_{#1, #2}}


\section{Introduction}
Markov Random Fields (MRFs) provide a useful framework to model high dimensional probability distributions via an associated dependency graph $\mathbf{G}$, which captures the conditional independence relationships between random variables. Here, the nodes correspond to the random variables; edges represent the conditional independence relationships between these nodes. Any random variable conditioned on the random variables with which it shares an edge is independent of all the remaining random variables. 

This `Markov' property has encouraged the adoption of MRFs in a wide variety of fields such as computer vision, finance, biology, and social networks. Here, MRFs model various inference tasks via popular algorithms such as loopy belief propagation, message passing, etc. For a deeper understanding of these underlying ideas and applications, we refer the reader to \cite{lauritzen1996graphical, koller2009probabilistic, wainwright2008graphical, pearl2014probabilistic}. 

A special class of graphical model where the underlying graph is tree-structured is suited for applications where sample efficient learning, and time-efficient inference are required with strong theoretical guarantees. As a result, the problem of learning tree-structured graphical models from data has been well-studied since the 1960s.  In the seminal work \cite{chow1968approximating}, the authors propose the Chow-Liu algorithm, which shows that the maximum weight spanning tree of the empirical mutual information between all the pairs of random variables corresponds to the maximum-likelihood tree estimate. In practice, it is rare to observe the random variables without noise, as sources of noise are ubiquitous, e.g. errors in sensors, incorrect human labeling. In \cite{nikolakakis2019learning}, the authors present numerous motivating examples from social science, epidemiology, biology, differential privacy, and finance, where noise is present in the observations. Unfortunately, in the face of corruption by unequal noise in the nodes, the Chow-Liu algorithm breaks down. This occurs as the noise in the random variables alters the order of the pairwise mutual information. The noise also destroys the tree structure by adding fictitious edges. Moreover, as noise is unknown, the structure of a noisy graphical model could possibly originate from different tree structures. This brings the recoverability of the original tree structure into question. 

In this paper, we focus on learning the underlying tree-structured graphical model on non-noisy discrete random variables using samples that are corrupted by a $k$-ary symmetric noise channel (where $k$ is the size of the common support of all the random variables). Our work reveals a rich recoverability landscape for MRFs under symmetric noise. We discover that when $k\geq 3$, for a fixed underlying tree structure, the recoverability is determined by the pairwise PMF of the non-noisy random variables. This is in contrast to the Gaussian graphical model and Ising model results (\cite{katiyar2019robust}, \cite{katiyar2020robust}, \cite{tandon2021sga}) where, for a fixed tree structure, edges within a \textit{leaf cluster} (a leaf node, its parent, and its siblings) are never recoverable irrespective of the probability distribution of the non-noisy random variables. We completely characterize the recoverability for $k\geq 2$ by providing the necessary and sufficient conditions for the identifiability of the edges within a \textit{leaf cluster}.

Our contributions can be summarized as follows:
\begin{itemize} [leftmargin=*,noitemsep]\vspace{-5pt}
    \item[1.] \textbf{Identifiability Characterization}:
    In \textit{Theorem \ref{th:k_ary_iden}}, we completely characterize the recoverability of tree-structured MRF on support size $k$ when the observations come from unknown $k$-ary symmetric channel noise where each node has a different error probability. We show the identifiability depends on the PMF of the non-noisy random variables, which is unobserved. This dependence can then be translated to the PMF of the noisy random variables, which is observed, that provides the characterization.
    
    We show that for the special class of {\em Symmetric Graphical Models} (as defined in \textit{Section \ref{sec:symmetric}}), for any $k$, the nodes within a \textit{leaf cluster} are unidentifiable. On the other direction, we show for the class of Perturbed Symmetric Graphical Models (details in \textit{Section \ref{sec:pertured_symmetric}}) for $k\geq 4$, the exact tree is identifiable.
    
    \item[2.] \textbf{Algorithm}: We develop an algorithm that recovers the class of candidate trees that can explain the noisy observations. In the identifiable setting, this corresponds to recovering the exact tree. The algorithm is iterative where we recover one edge from the candidate tree per iteration. \textit{(Section \ref{sec:algo})}.
    
    \item[3.] \textbf{Sample Complexity Analysis:} We provide novel sample complexity lower bounds and upper bounds (\textit{Section \ref{sec:sample_complexity}}). Our upper bounds are shown to have orderwise tight dependence on underlying graph parameters, size of the graph, edge parameters (related to underlying conditional MF), and noise parameters. The lower bound proof relies on a novel construction of a class of graphical models including perturbed symmetric graphical models where part of the \textit{leaf clusters} are identifiable.  
    
    \item[4.] \textbf{Experiments:}\footnote{The code containing the implementation of the algorithm is available at \url{https://github.com/ashishkatiyar13/NoisyTreeMRF}} We demonstrate the efficacy of our algorithm via extensive numerical experiments for a variety of trees with different structures, edge parameters, corruption, and support sizes.
\end{itemize}


\section{Related Work}
We divide the related work into three main categories:\\
\textbf{Learning Generic Graphical Models from Non-Noisy Samples:} There exists a rich literature on the problem of learning graphical models on discrete random variables which assume access to non-noisy samples \cite{bresler2014structure, bresler2008reconstruction, bresler2015efficiently, bresler2014hardness, lee2007efficient, klivans2017learning, wu2019sparse, ravikumar2010high}. However, these models do not provide guarantees in the face of noise in the samples.\\
\textbf{Learning Tree-Structured Graphical Models:} The special class of tree-structured graphical models has also been extensively studied beginning with the classical Chow-Liu algorithm was proposed in \cite{chow1968approximating}. Chow-Liu algorithm's error exponents for Gaussian graphical models and graphical models on discrete random variables were analyzed in \cite{tan2010learning} and \cite{tan2011large} respectively. Results in \cite{tan2011large} were further refined in \cite{tandon2020exact} under additional assumptions of homogeneity and zero external field in tree-structured Ising models.  In \cite{bresler2016learning} the authors approximate the distribution of generic Ising models using tree-structured Ising models. More recently, in \cite{daskalakis2020tree}, the authors provide an algorithm to learn tree-structured Ising models providing total variation distance guarantees. In \cite{bhattacharyya2020near}, the authors provide finite sample guarantees for the Chow-Liu algorithm. As these algorithms assume access to non-noisy samples, no performance guarantees can be established when the samples have noise.\\
\textbf{Robust Estimation of Graphical Models:}  Robust estimation of graphical models has been studied in multiple prior works but they are unable to resolve our setting.  The algorithms in \cite{goel2019learning, lindgren2019robust, hamilton2017information}  learn graphical models on discrete random variables without the tree structure assumption but assume access to error probabilities. This is complementary to our setting as we have the tree structure constraint but do not require the knowledge of the error probabilities.
In \cite{tandon2020exact,nikolakakis2019learning,nikolakakis2020information}, the authors study the recovery of trees using noisy samples. Critically, they operate in the restricted regime where the Chow-Liu algorithm converges to the correct tree.
While these results are insightful in their own right, their assumptions are generally violated in our setting making their results inapplicable.


For Gaussian graphical models and Ising models, the unidentifiability  properties are established in \cite{katiyar2019robust} and \cite{katiyar2020robust}, respectively. In \cite{tandon2021sga} the authors extend the results in \cite{katiyar2019robust, katiyar2020robust}, providing better sample complexity results and a more efficient algorithm. The critical limitation of these results is that they do not extend to discrete random variables with support sizes larger than 2 and therefore fail to capture the nuanced identifiability properties demonstrated in our setting.

Finally, our problem can be posed as the latent tree graphical model estimation problem, where the noisy nodes are observed and non-noisy nodes are latent. Results for learning latent tree graphical models in \cite{pearl1986structuring, chang1996full, choi2011learning}, and {\em independently and concurrently} in \cite{ casanellas2021robust}, can be used to recover the underlying tree barring the nodes within leaf clusters. Importantly, these models do not assume any structure on the noise, and thereby, contrived noise models make it impossible to recover nodes within a leaf cluster. As a result they fail to uncover the possibility of identifiability within a leaf cluster when we consider the natural $k$-ary symmetric channel noise model.

\section{Problem Setup}
Let $\mathbf{X} = [X_1, X_2\dots X_n]$ be the vector of random variables with a common support set, $\sets = \{s_1, s_2, \dots s_k\}$ such that their graphical model structure is a tree $T^*$.
The vanilla learning problem is to recover the tree $T^*$ from i.i.d samples of $X_i$.

In this paper, we consider the problem of recovering $T^*$ but we do not get to observe samples of $X_i$. Instead, the samples of $X_i$ pass through a $k$-ary symmetric noise channel and we observe the output denoted by $X_i'$, that is,
\begin{equation} \label{eq:noise}
    X_i' = \begin{cases}X_i & \text{ w.p. }1 - q_i,\\
     U_i & \text{ w.p. } q_i,
    \end{cases}
\end{equation}
where $q_i$ is the probability of error for $X_i$ and $U_i$ is a discrete random variable independent of $\mathbf{X}$ and $U_j$ $\forall j\neq i$, distributed uniformly on $\sets$. Note that $q_i$ can be unequal for all $X_i$.
The vector of the noisy random variables is denoted by $\mathbf{X'} = [X_1', X_2'\dots X_n']$. Due to the noise in $X_i$, the graphical model of the nodes in $\mathbf{X'}$ is no longer given by $T^*$. In general, \textit{the graphical model on the noisy random variables can be a complete graph}.







\paragraph{Matrix PMF and Distance Notation:} We denote the joint PMF matrix for random variables ($X_a$, $X_b$), and ($X_a'$, $X_b'$) by the matrix $P_{a,b}$ and $P_{a',b'}$ respectively, such that:
\begin{equation*}
(P_{a,b})_{i,j} = P(X_a = s_i, X_b = s_j),
(P_{a',b'})_{i,j} = P(X_a' = s_i, X_b' = s_j).
\end{equation*} 
The conditional PMF of $X_a$ conditioned on $X_b$ is denoted by the matrix $P_{a|b}$ while the marginal distribution of random variables $X_a$ and $X_a'$ are denoted using diagonal matrices $P_a$ and $P_{a'}$ respectively such that:
\begin{equation*}
(P_{a|b})_{i,j} = P(X_a = s_i|X_b = s_j), (P_a)_{i,i} = P(X_a = s_i), (P_{a'})_{i,i} = P(X_a' = s_i).
\end{equation*} 
The information distance metric between proposed in \cite{lake1994reconstructing}, is defined as follows:
\begin{equation}\label{eq:dist}
    d_{i,j} =-\log\tfrac{|det(P_{i,j})|}{\sqrt{det(P_i)det(P_j)}}, d_{i',j'} =-\log\tfrac{|det(P_{i',j'})|}{\sqrt{det(P_{i'})det(P_{j'})}}.
\end{equation}

We require the following assumptions that are natural and standard in this line of literature (c.f. \cite{chang1996full,choi2011learning}).
\begin{assumption}\label{ass:pmf}
The probability mass at every support for each non-noisy random variable is bounded away from $0$ : $(P_a)_{i,i}\geq p_{min}>0$.
\end{assumption}
\begin{assumption}\label{ass:distance}
The distance $d_{i,j}$ between adjacent non-noisy random variables is bounded: $0<d_{min}<d_{i,j}<d_{max}$.
\end{assumption}
\begin{assumption}\label{ass:max_error}
The probability of error is upper bounded away from 1: $q_i \leq q_{max} < 1$.
\end{assumption}


Assumption \ref{ass:pmf} ensures that the probability mass at any support is not arbitrarily small for any random variable. The bounds on the distance in Assumption \ref{ass:distance} ensure that no adjacent random variables are duplicates or independent. Assumption \ref{ass:max_error} ensures that the noisy observations are not independent of the underlying random variables.  Our sample complexity lower bounds in Section \ref{sec:sample_complexity} show that the problem becomes infeasible if these assumptions are not satisfied.

Lastly, we also formally define a \textit{leaf cluster} as follows:
\begin{definition} \label{def:leafcluster}
The \textbf{leaf cluster} of any leaf node is the set containing that leaf node, its parent node and all its sibling leaf nodes.
\end{definition}



\section{Identifiability Results} \label{sec:identifiability}
In this section, we prove that the identifiability of the underlying tree is determined by the joint PMF of leaf parent pairs. The proof is divided in 3 parts - (i) prove that the only potential unidentifiability is within the leaf clusters of the tree, (ii) analyze the existence of valid probability of error for a tree on three nodes, (iii) extend the analysis to a generic tree and arrive at the necessary and sufficient condition for identifiability. 

\subsection{Potential unidentifiability is limited to leaf clusters}
For any tree $T^*$, \cite{katiyar2019robust} defined the equivalence class $\mathcal{T}_{T^*}$ to be the set of all the trees obtained by different permutations of nodes within a leaf cluster, and showed that in the Gaussian graphical model setting, $\mathcal{T}_{T^*}$ can be recovered. We show here that with a few new proof ideas, essentially the same is true for graphical models on discrete random variables with general support size $k$:
\begin{lemma}\label{le:lim_unid_gen}
Suppose the random variables in $\mathbf{X}$ form a tree graphical model $T^*$. Given samples from noisy random variables $X_i'$, it is possible to recover the equivalence class $\mathcal{T}_{T^*}$.
\end{lemma}
\textit{Proof Idea.} The proof of this lemma is similar in spirit to \cite{katiyar2019robust} and so we defer the details to Appendix \ref{ap:lemma1}. The proof depends on categorizing groups of 4 nodes as a {\em non-star} when 2 of the nodes lie in one subtree and the remaining 2 nodes lie in a disjoint subtree. The key new element we need for this categorization in the discrete setting for general $k$, is the information distance metric $d_{i,j}$ as defined in \eqref{eq:dist}.\\
\textbf{Remarks:} (i) Lemma \ref{le:lim_unid_gen} is not limited to the $k$-ary symmetric noise channel and holds for any noise channel such that when conditioned on $X_i$, $X_i'$ is independent of $X_j$ $\forall j \in [n] \neq i$ and $X_i$ and $X_i'$ are not independent. This result was independently and concurrently derived in \cite{casanellas2021robust}. (ii) If there are no restrictions on the noise channel, recovering $\mathcal{T}_{T^*}$ is the best we can do. That is, for every tree in $\mathcal{T}_{T^*}$, it is possible to construct a noise model that can produces the noisy observation. This analysis along with the proof of Lemma \ref{le:lim_unid_gen} is included in Appendix \ref{ap:lemma1}.

\subsection{Error Estimation for a Tree on 3 Nodes}

\paragraph{Additional Notation for $k$-ary Symmetric Channel:}
For each random variable $X_a$, we define a $k\times k$ error matrix $E_a$ as follows:
\begin{equation*}
    E_a = (1-q_a)I + \tfrac{q_a}{k}O,
\end{equation*}
where $O$ is a matrix of all ones. Recall that $k$ is the common support size for all the random variables and $q_a$ is the probability of error of $X_a$.\\
We denote the error estimated for node $X_a$ which enforces $X_b\perp X_c|X_a$ by $\Tilde{q}_{a}^{b,c}$ and we also define the matrix $\ind{\Tilde{E}}{a}{b}{c}$ as:
\begin{equation*}
    \ind{\Tilde{E}}{a}{b}{c} =  (1-\ind{\Tilde{q}}{a}{b}{c})I + \tfrac{\ind{\Tilde{q}}{a}{b}{c}}{k}O.
\end{equation*}
Note that $P_{a',b'}$ and $P_{a,b}$ are related as follows:
\begin{equation}\label{eq:noisy_joint_pmf}
    P_{a',b'} = E_aP_{a,b}E_b.
\end{equation}
It is also easy to see that:
\begin{equation}\label{eq:noisy_pmf}
P_{a'} = (1-q_a)P_a + \tfrac{q_a}{k} I.
\end{equation}
\paragraph{Error Estimation:}
Suppose there exist 3 nodes such that $X_1\perp X_3|X_2$ and we observe $X_1'$, $X_2'$ and $X_3'$ through a $k$-ary symmetric channel as defined in Equation \eqref{eq:noise}. The conditional independence relationship gives us:
\begin{equation}\label{eq:cond_ind}
    P_{1, 3} = P_{1,2}P_2^{-1}P_{2,3}.
\end{equation}
From Equation \eqref{eq:noisy_joint_pmf}, we have $P_{1',3'} = E_1P_{1,3}E_3$, $P_{1',2'} = E_1P_{1,2}E_2$, $P_{2',3'} = E_2P_{2,3}E_3$. From Equation \eqref{eq:noisy_pmf}, we have $P_{2'} = (1-q_2)P_2 + \frac{q_2}{k}I$. By substituting these in Equation \eqref{eq:cond_ind} we get the following quadratic equation with matrix coefficients in noise parameter $q_2$ (details in Appendix \ref{ap:quadratic}):
\begin{equation}\label{eq:err_est_quad}
\begin{aligned}
  &\frac{q_2^2}{k^2}(O - kI) - \frac{q_2}{k}(OP_{2'} + P_{2'}O - kP_{2'} - I) + 
  P_{2',3'} P_{1,' 3'}^{-1}P_{1',2'}-P_{2'} = 0,
 \end{aligned}
\end{equation}
where the $0$ on the RHS is a $k\times k$ matrix of all $0$s.
The key insight here is that, Equation \eqref{eq:err_est_quad} depends only on the noisy observations. Therefore, in the absence of the knowledge of conditional independence relation, it can be used as a test to check if the noisy observations can potentially be explained by $X_1\perp X_3|X_2$. 
Precisely, for a graph on 3 nodes $(X_1, X_2, X_3)$, $X_2$ is a potential middle node if the we can satisfy Equation~\eqref{eq:err_est_quad} for some noise parameter $q_2 \in [0,q_{max}]$. In other words, $X_2$ is a potential middle node if the following holds, with $\|\cdot\|_F$ as the Forbenius norm of a matrix:
\begin{equation}\label{eq:err_est_x}
\begin{aligned}
  &\min_{0\leq x\leq q_{max}} \|\frac{x^2}{k^2}(O - kI) - \frac{x}{k}(OP_{2'} + P_{2'}O - kP_{2'} - I) + 
  P_{2',3'} P_{1,' 3'}^{-1}P_{1',2'}-P_{2'}\|_F = 0.
 \end{aligned}
\end{equation}
This is equivalent to $k^2$ quadratic equations corresponding to each element of the matrix having a common root which lies between $0$ and $q_{max}$. These equations need not be unique.

\subsection{Extension to a generic tree}

Before presenting the identifiability result, we first establish some notation. Let $\setl$ be the set containing all the leaf nodes of the tree-structured graphical model $T^*$. Now, consider the subset of leaf nodes with the following property: the leaf node $X_2$, its parent node $X_1$, and any arbitrary node $X_3$ from the graph have a solution to Equation \eqref{eq:err_est_x}. We label this subset $\setl^{sub} \subseteq \setl$.    $\mathcal{T}_{T^*}^{sub}\subseteq \mathcal{T}_{T^*}$ represents the equivalence class where only leaves in $\setl^{sub}$ can exchange positions with their parents.\\
The next theorem completely characterizes the identifiability of the underlying tree for a $k$-ary symmetric noise channel.

\begin{theorem}\label{th:k_ary_iden}
Suppose the random variables in $\mathbf{X}$ form a tree-structured graphical model $T^*$. Let $\mathbf{X}'$ be the observed noisy output after passing $\mathbf{X}$ through a $k$-ary symmetric channel. Then, we show that for any leaf node $X_2 \in \mathcal{L}^{sub}$ and its parent node $X_1$, equation \eqref{eq:err_est_x} remains unchanged for any arbitrary third node $X_3$ from the graph. Using $\mathbf{X}'$, we can recover $\mathcal{T}_{T^*}^{sub}$. Moreover, for every tree $\Tilde{T}\in\mathcal{T}_{T^*}^{sub}$, there exist random variables $\Tilde{\mathbf{X}}$ and a  $k$-ary symmetric channels such that the graphical model of $\Tilde{\mathbf{X}}$ is $\Tilde{T}$ and the $k$-ary channel output is $\mathbf{X}'$.

\end{theorem}

\textit{Proof Idea:} As the unidentifiability is only between the nodes within a \textit{leaf cluster}, the key idea is to study a subset of 3 nodes comprising of a leaf parent pair and an arbitrary third node. It is clear that, Equation \eqref{eq:err_est_x} has a solution when the parent node is the middle node. Whenever Equation \eqref{eq:err_est_x} does not have a solution for a given node being a candidate center node, we can rule out the possibility of that node being a parent node. We further show that when the solution exists for a leaf node as a candidate center node, we can construct a tree where the parent node exchanges position with the leaf node. The details are presented in Appendix \ref{ap:k_ary_iden_proof}.


\subsection{Examples} \label{sec:examples}
In this section, we do not assume access to $q_{max}$ and analyse the solution to Equation \eqref{eq:err_est_x} with the constraint $0<x<1$. Extension to the setting of $0<x<q_{max}$ is straightforward where we reject any solution $x>q_{max}$. We first prove that symmetric graphical models are unidentifiable. Next, we present perturbed symmetric graphical models that are unidentifiable for $k=3$ but are identifiable for $k\geq4$. Finally, we show that our analysis recovers the existing results for $k=2$.

\paragraph{Symmetric graphical models:}\label{sec:symmetric} Symmetric graphical models are a class of graphical models where the marginals of all the random variables are uniform on the support and the conditional PMF matrix $P_{a|b}$ for random variables $X_a$, $X_b$ that have an edge between them, takes the following form:
$$
P_{a|b} = P_{b|a} = \alpha_{a,b}I+(1-\alpha_{a,b})\tfrac{O}{k}.
$$
Recall that $O$ is the matrix of all ones. The bounds on the distance in Assumption \ref{ass:distance} enforces $\exp{(-d_{max}/(k-1))}<\alpha_{a,b}<\exp{(-d_{min}/(k-1))}$. 
\begin{theorem}\label{th:symmetric}
Suppose the random variables in $\mathbf{X}$ form a tree graphical model $T^*$. Let $X_2$ be any leaf node and $X_1$ be its parent node. If $P_1 = P_2 = \frac{I}{k}$ and $P_{2|1} = \alpha_{2,1}I+(1-\alpha_{2,1})\frac{O}{k}$ such that $\exp{(-d_{max}/(k-1))}<\alpha_{2,1}<\exp{(-d_{min}/(k-1))}$, then Equation \eqref{eq:err_est_x} has a solution.
\end{theorem}
The proof is included in Appendix \ref{ap:symmetric}. Since, Equation \eqref{eq:err_est_x} has a solution for every leaf node $X_2$ as the candidate center node, using Theorem \ref{th:k_ary_iden}, we conclude that symmetric graphical models are unidentifiable.
\paragraph{Perturbed symmetric graphical models:}\label{sec:pertured_symmetric}
We first define a $k\times k$ perturbation matrix $\Delta_{a,b}$. For a given offset $0<c_{a,b}<k$, the term in the $(i,j)$ position of $\Delta_{a,b}$ is:
$$
\Delta_{a,b}(i,j) = \left\{\begin{array}{rl}
        \delta_{a,b}, & \text{for } j = ((i-1+c_{a,b})\mod k) + 1\\
        0, & \text{o/w}.
        \end{array}\right.
$$
In the perturbed symmetric model, the marginals continue to be uniform on the support but the conditional PMF matrix $P_{a|b}$ for adjacent $X_a$ and $X_b$ is modified to:
$$
P_{a|b} = (\alpha_{a,b}-\delta_{a,b})I+(1-\alpha_{a,b})\tfrac{O}{k}+\Delta_{a,b}.
$$
Here $\alpha_{a,b}$ and $\delta_{a,b}$ are chosen such that  Assumption \ref{ass:distance} is satisfied. We find that perturbed symmetric graphical models are unidentifiable for $k = 3$ but become identifiable for $k\geq 4$.
\begin{theorem}\label{th:perturbed_symmetric}
Suppose the random variables in $\mathbf{X}$ form a tree graphical model $T^*$. Let $X_2$ be any leaf node and $X_1$ be its parent node. Suppose $P_1 = P_2 = \frac{I}{k}$ and $P_{2|1} = (\alpha_{a,b}-\delta_{a,b})I+(1-\alpha_{a,b})\frac{O}{k}+\Delta_{a,b}$ such that $|\delta_{a,b}|>0, \alpha_{a,b}\neq\delta_{a,b}$, and $\alpha_{a,b}$, $\delta_{a,b}$ are such that the distance assumptions in \ref{ass:distance} are satisfied. Then, equation \eqref{eq:err_est_x} has a solution for $k=3$, but does not have a solution for $k\geq 4$.
\end{theorem}
\textit{Proof Idea.} The proof for $k\geq 4$ relies on lower bounding the Frobenius norm of the quadratic away from 0. In conjunction with Theorem \ref{th:k_ary_iden}, this implies that the exact tree is identifiable when $k\geq4$. For $k=3$, we explicitly calculate the solution to Equation \eqref{eq:err_est_x}. Note that, for $k=3$ the class of symmetric and perturbed symmetric graphical models together comprise all the joint PMF matrices that are circulant. In fact, for $k=3$, when the marginals are uniformly distributed, the joint PMF matrix being circulant is a necessary and sufficient condition for unidentifiability. These details are presented in Appendix \ref{ap:perturbed_symmetric}.

\paragraph{Unidentifiability when $k = 2$:} We now discuss the unidentifiability for $k=2$.
\begin{lemma}\label{le:bin_sol}
Suppose the random variables in $\mathbf{X}$ have support size $k=2$ and they form a tree graphical model $T^*$. The random variables in $\mathbf{X}$ pass through a binary symmetric channel with positive probability of error and we observe $\mathbf{X}'$. For any 3 nodes $(X_1, X_2, X_3)$, Equation \eqref{eq:err_est_x} always has a valid solution.
\end{lemma}
The proof of Lemma \ref{le:bin_sol} is in Appendix \ref{ap:bin_sol}. Corollary \ref{cor:ising} recovers the unidentifiability results of \cite{katiyar2020robust}.
\begin{corollary}\label{cor:ising}
When the random variables in $\mathbf{X}$ have a support size of 2 and all the parents of leaf nodes have non-zero noise, we have $\mathcal{T}_{T^*}^{sub} = \mathcal{T}_{T^*}$.
\end{corollary}

\section{Algorithm}\label{sec:algo}
In this section, we present the algorithm to recover a tree from $\mathcal{T}_{T^*}^{sub}$ given samples corrupted by a $k$-ary symmetric noise channel as inputs. \\
\textbf{Key Idea:}
The algorithm to recover the tree is an iterative one. During an iteration, we have an active set of nodes which are guaranteed to form a subtree. At each iteration, we find a leaf parent pair in the subtree, record that edge, and remove the leaf node from the active set of nodes. The algorithm to recover the tree structure is presented in Algorithm \ref{alg:find_tree}. 
\begin{figure}
    \centering
    \begin{minipage}{0.65\textwidth}
        \begin{algorithm}[H]
            \caption{Recover Tree Structure}\label{alg:find_tree}
            \textit{Input}: Pairwise noisy distributions, $P'_{i,j}$ $\forall{i,j} \in [n]$\\
            \textit{Output}: List of edges, $Edges$
            \begin{small}
            \begin{algorithmic}[1]
            \Procedure{FindTree}{$P'_{i,j}$ $\forall{i,j} \in [n]$} 
            \State $ActiveSet \gets \{1, 2, \dots n\}$, $Edges \gets \{\}$, $Parents\gets \{\}$
            \While{$|ActiveSet| > 2$}
            \State $leaf,parent \gets $  \textsc{GetLeafParent}($P'_{i,j}$, $ActiveSet$, $\dots$\\
            \hspace{18em} $Edges$, $Parents$)
            \State{$ActiveSet \gets ActiveSet\setminus leaf$}
            \State{$Edges\gets Edges\cup (leaf,parent)$}
            \State{$Parents\gets Parents\cup parent$}
            \EndWhile
            \State $Edges\gets Edges\cup(ActiveSet[0],ActiveSet[1])$\\
            \Return $Edges$
            \EndProcedure
            \end{algorithmic}
            \end{small}
        \end{algorithm}
    \end{minipage}~\hfill
    \begin{minipage}{0.33\textwidth}
        \centering
        \includegraphics[width=0.9\textwidth]{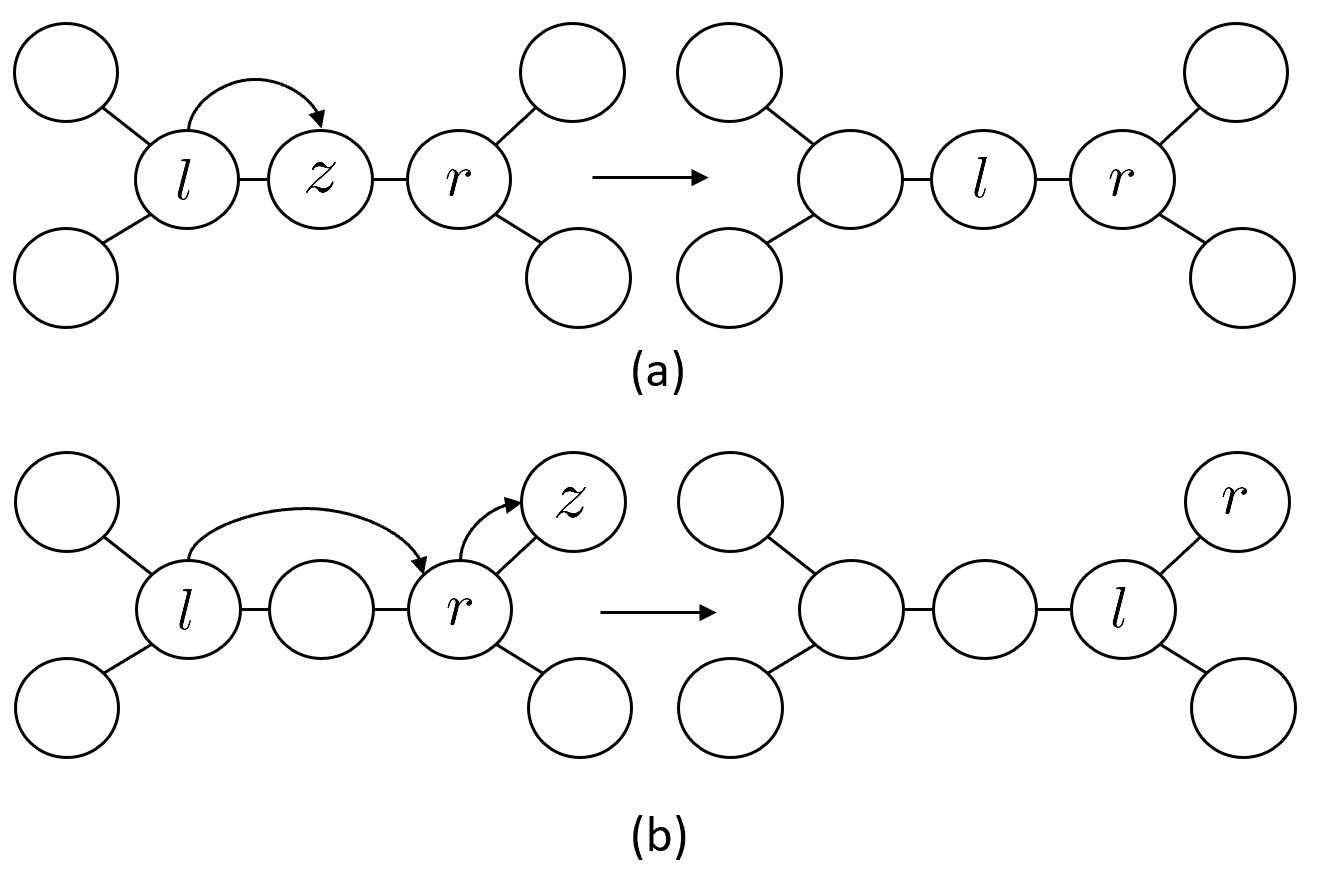} 
        \caption{(a) If the node $z$ lies between $l$ and $r$, $l$ becomes $z$, hence getting closer to $r$. (b) If the node $r$ lies between $l$ and $z$, both $l$ and $r$ shift towards the right with $l$ becoming $r$ and $r$ becoming $z$.}
        \label{fig:alg_step}
    \end{minipage}
 \vspace{-2pt}
\end{figure}
\\
\textbf{Finding a leaf parent pair:} 
We next describe the algorithm to find a leaf parent pair. We maintain two nodes -  a left node $l$, and a right node $r$.  The idea is to move both the nodes towards the right side till $r$ is a leaf node and $l$ is its parent node. In order to do this we consider a third node $z$ and perform the following operations:
\vspace{-0.7pc}
\begin{enumerate}[leftmargin = *]
    \item If the center node in $(l,r,z)$ is $z$, we shift node $l$ to node $z$,
    \item If the center node in $(l,r,z)$ is $r$, we shift node $l$ to node $r$ and node $r$ to node $z$.
\end{enumerate}
\vspace{-0.7pc}

This is illustrated in Figure (\ref{fig:alg_step}). Finding the center node can be done by checking the feasibility of Equation \eqref{eq:err_est_x} for different candidate center nodes. 

If Equation \eqref{eq:err_est_x} has a solution for more than one nodes, we use an alternative method which uses the 3 nodes in conjunction with different $4^{th}$ nodes. These 4 nodes are categorized as star/non-star to arrive at the center node. While doing the test for the center node, we only consider the nodes with pairwise distances smaller than $4d_{max} + 3\eta_{max}$. Here $\eta_{max}$ is an upper bound on the distance between a clean and noisy node. For a given $p_{min}$ and $q_{max}$ from Assumption \ref{ass:pmf} and \ref{ass:max_error} respectively, $\eta_{max} = (1-k)\log (1-q_{max}) - 0.5k \log (kp_{min})$ (details in Appendix \ref{ap:algo}).
This makes it easy to adapt the algorithm for the finite sample setting.
%
%
\paragraph{Finite sample algorithm:} 
The finite sample version of the algorithm uses the empirical estimate of the joint PMF of random variables to test for the center node given a set of three nodes. We only perform the test for nodes that whose empirical distance is small to avoid a sample complexity exponential in the diameter of the graph. For the test of center node by checking for existence of a solution to Equation \eqref{eq:err_est_x} using empirical PMF estimates, we need the following additional assumption:
\begin{assumption}\label{ass:fin_sample_err_est}
When Equation \eqref{eq:err_est_x} does not have a solution, we have the following inequality:
\begin{align*}
  \min_{0\leq x<q_{max}}& \|\frac{x^2}{k^2}(O - kI) - \frac{x}{k}(OP_{2'} + P_{2'}O - kP_{2'} - I) + 
  P_{2',3'} P_{1,' 3'}^{-1}P_{1',2'}-P_{2'}\|_F > t_0
\end{align*}
\end{assumption}
This assumption ensures that when Equation \eqref{eq:err_est_x} does not have a solution for a leaf node $X_2$ as a center node, it can be detected in the presence of perturbations due to finite samples.
In Appendix \ref{ap:algo}, we provide the details of the algorithm including finding the center node, and necessary modifications for executing the algorithm using finite samples. In addition, we also include the pseudocode and the proof of correctness of the algorithm.
\vspace{-5pt}
\paragraph{Insights into the input parameters of the algorithm:}
The algorithm in its vanilla form requires $d_{min},~d_{max},~ q_{max}, p_{min}$ and  $t_0$ in addition to the noisy samples as inputs. While the dependence on the knowledge of $q_{max}$ is necessary, it is possible to obtain estimates of bounds of $d_{min}$ and $d_{max}$ using the noisy samples. This comes at the cost of higher sample complexity. Dependence on $t_0$ can also be avoided at the cost of higher time complexity. This is detailed as follows:
\vspace{-5pt}
\begin{itemize}[leftmargin = *, noitemsep]
    \item The upper bound on $d_{max}$ is denoted by $\Tilde{d}_{max}$. It is defined as $\Tilde{d}_{max} = \max_i\min_{j\neq i}d_{i'j'}$. This bound can potentially be lose by $2\eta_{max}$.
    \item If the ground truth is such that $d_{min} - 2\eta_{max} > 0$ then a lower bound on $d_{min}$, denoted by $\Tilde{d}_{min}$, can be defined as $\Tilde{d}_{min} = \min_i\min_{j\neq i}d_{i'j'} - 2\eta_{max}$. This bound can also be loose by $2\eta_{max}$.
    \item If $p_{min}$ and $q_{max}$ are such that $p_{min}>q_{max}$ then a valid lower bound on $p_{min}$ is $\min_i(P_{a'})_{i,i} - q_{max}$ which can potentially be lose by $q_{max}$.
    \item In the absence of the knowledge of $t_0$, we can use the star/non-star test for finding the center node among 3 nodes as long as no 2 nodes belong to the same \textit{leaf cluster}. This increases the time complexity of finding the center node from $\mathcal{O}(1)$ to $\mathcal{O}(n)$. Once we get nodes within the same \textit{leaf cluster}, the potential center node with the minimum objective function in Equation \eqref{eq:err_est_x} is chosen as the center node.
\end{itemize}
\vspace{-10pt}

\section{Sample Complexity Results}\label{sec:sample_complexity}\vspace{-5pt}
In this section, we provide both the sample complexity upper bounds and sample complexity lower bounds for recovering the tree using our algorithm in presence of corrupted samples.
\begin{theorem}[\textbf{Sample Complexity Upper Bound}]\label{th:ub}
Suppose the random variables in $\mathbf{X}$ form a tree graphical model $T^*$ and we observe $\mathbf{X}'$ such that Assumptions \ref{ass:pmf}, \ref{ass:distance}, \ref{ass:max_error} and \ref{ass:fin_sample_err_est} are satisfied. Then, the finite sample Algorithm \ref{alg:find_tree} correctly recovers $\mathcal{T}_{T^*}^{sub}$ with probability at least $1-\delta$ if the number of samples $N$ satisfies
\begin{small}
\begin{align*}
    N = \mathcal{O}\Bigg(\max\Bigg\{&\tfrac{k^2\exp(8d_{\max})}{(1-q_{max})^{6(k-1)}(0.9p_{min}^{2.5})^{2k}(1-\exp{(-2d_{min})})^2(k-1)^{2(k-1)}}\Bigg. \Bigg.,\\
    &\Bigg.\Bigg.\tfrac{k \exp(16d_{\max})}{t_0^2 (1-q_{max})^{12(k-1)}(0.9p_{min}^{2.5})^{4k}(k-1)^{4(k-1)}}\Bigg\}\log\left(\tfrac{2nk(n-1)}{\delta}\right)\Bigg)
\end{align*}
\end{small}
\end{theorem}
In the unidentifiable setting, since Equation \eqref{eq:err_est_x} always has a solution, our algorithm finds more than one candidate center nodes and therefore resorts to the star/non-star test for finding the center node. In the sample complexity, the second term in the $\max$ comes from the quadratic test and therefore it can be dropped. As a result, since we have an easier learning problem of learning only $\mathcal{T}_{T^*}$, the sample complexity has better dependence on $d_{max}, q_{max}$ and $p_{min}$. 

\begin{theorem}[\textbf{Sample Complexity Lower Bound}]\label{th:lb}
Suppose the random variables in $\mathbf{X}$ form a tree graphical model $T^*$ and we observe $\mathbf{X}'$ such that Assumptions \ref{ass:pmf}, \ref{ass:distance}, \ref{ass:max_error} and \ref{ass:fin_sample_err_est} are satisfied. Then  any algorithm that correctly recovers $\mathcal{T}_{T^*}^{sub}$ with probability at least $1-\delta$ requires $N$ samples where 
$$
N = \Omega\left(\tfrac{\exp\left(\tfrac{2d_{\max}}{k-1}\right)}{(k-1)(1-q_{\max})^2\left(1-\exp\left(-\tfrac{d_{\min}}{k-1}\right)\right)} (1- \delta) \log(n)\right)
$$
Furthermore, for $k \geq 4$, $0< t_0 \leq \tfrac{k}{10}\exp(-2\tfrac{d_{\max}}{k-1})$, we additionally have 
$$
N = \Omega\left(\max_{d\in \{d_{\max}, d_{\min}\}}\exp\left(-\tfrac{2d}{k-1}\right)\left(1-\exp\left(-\tfrac{d}{k-1}\right)\right) \tfrac{k(1- \delta) \log(n)}{ t_0^2}\right)
$$
\end{theorem}
\vspace{-5pt}
We note that our lower bounds on sample complexity shows our certain dependence on the problem parameters cannot be improved orderwise. 
Firstly, we see the dependence on the graph size scales as $\Theta(\log(n))$ which is standard in graphical model learning. We observe that the sample complexity scales as ${\exp(\Theta(d_{\max}))}$ as a function of the $d_{max}$. Furthermore, for small enough $t_0$ and support size $4$ or more, the dependence on the lower bound for the quadratic term $Q(x)$, $t_0$, scales as $\Theta(\frac{1}{t_0^2})$ highlighting the significance of the term $Q(x)$ in the recovery of MRFs under unknown symmetric noise model.

Our lower bound proof for $t_0$ dependence in the (partially) identifiable case uses a family of $(n+1)$ star graphs with $n$ edges each, where one graph is a perturbed symmetric graphical model (Section \ref{sec:pertured_symmetric}), and for the other graphs we select one edge each and replace the conditional PMF with the one from a symmetric model.
Thus, the equivalence class $\mathcal{T}_{T^*}^{sub}$ for each graph in the family is unique. For the lower bounds in the unidentifiable scenario, we generalize the construction in \cite{tandon2021sga} to $k>2$ support size using symmetric graphical models. Our derivation for KL divergence for symmetric graphical model, and perturbed symmetric graphical models used in the lower bound proofs can be of independent interest.\vspace{-10pt}



\section{Experiments}\label{sec:exp}
 In this section, we present the experiments demonstrating the efficacy of our algorithm (The code can be found at \url{https://github.com/ashishkatiyar13/NoisyTreeMRF}.). We first demonstrate the performance of our algorithm for the $k = 2$ setting and demonstrate that our algorithm considerably outperforms the algorithm in \cite{tandon2021sga}. Next, we showcase the performance of our algorithm for the $k = 4$ setting with the perturbed symmetric model. As discussed in Section \ref{sec:examples}, the exact tree is identifiable in this scenario. \vspace{-8pt}
\begin{figure}
    \begin{subfigure}{\textwidth}
        \centering
        \includegraphics[scale = 0.27]{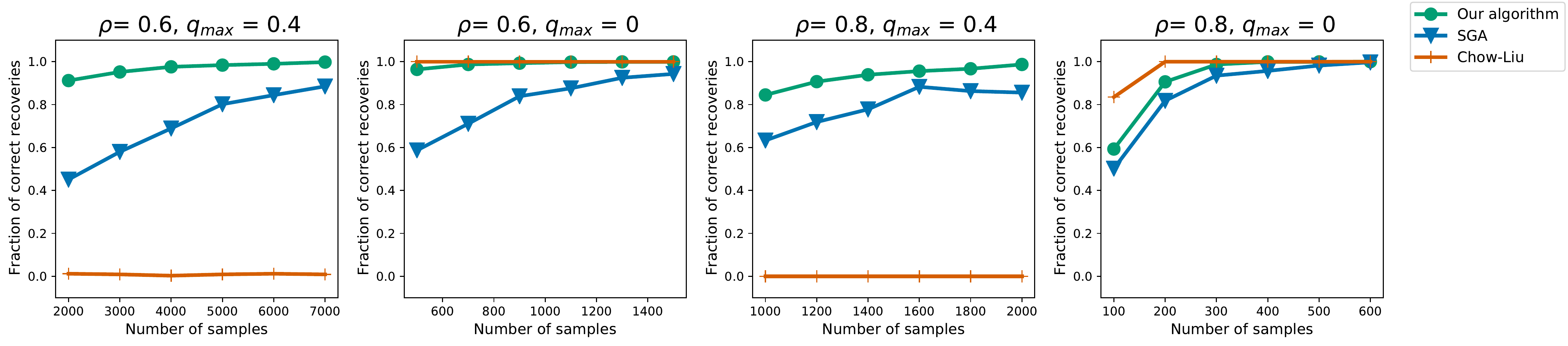}  
        \caption{Chain Graph}
        \label{fig:our_vs_sga_chain}
    \end{subfigure}
    \newline
    \begin{subfigure}{\textwidth}
        \centering
        \includegraphics[scale = 0.27]{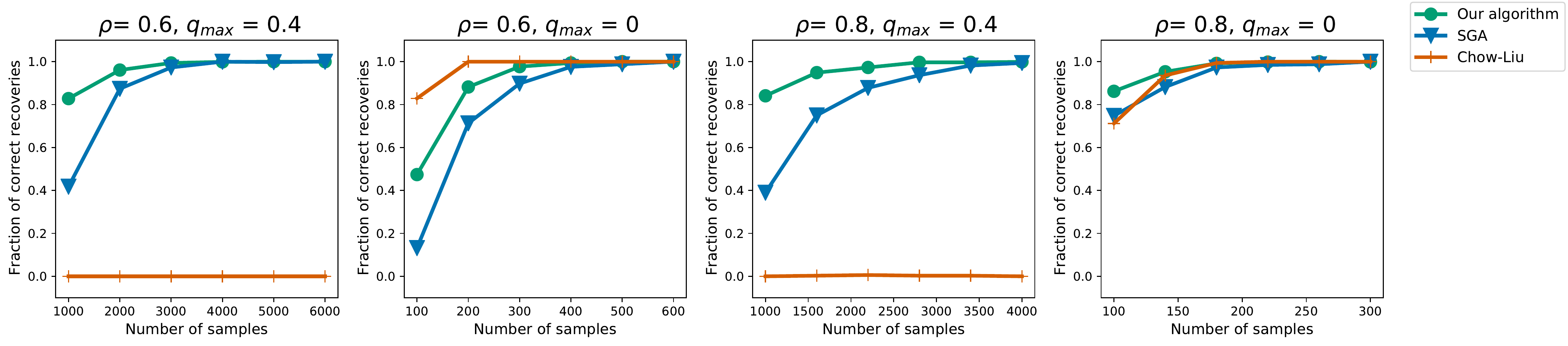}  
        \caption{Star Graph}
        \label{fig:our_vs_sga_star}
    \end{subfigure}
    \caption{For both chain and star graphs, our algorithm outperforms SGA for 4 different settings - (i) $\rho_{max} = 0.6, q_{max} = 0.4$, (ii) $\rho_{max} = 0.6, q_{max} = 0.0$, (iii) $\rho_{max} = 0.8, q_{max} = 0.4$, (iv) $\rho_{max} = 0.8, q_{max} = 0.0$}
    \label{fig:our_vs_sga}
\end{figure}

\subsection{Support size, $k = 2$ (Unidentifiable setting):}\vspace{-5pt}
In this part, we compare the performance of our algorithm for chain and star graphs to that of SGA proposed in \cite{tandon2021sga}. We use the exact same settings as in \cite{tandon2021sga} and demonstrate that we outperform SGA. \\
For chain graphs, the nodes are labeled $X_1$ to $X_{12}$ from left to right. The star graphs have $X_1$ as the center node and $X_2,\dots X_{12}$ are leaf nodes connected to $X_1$\vspace{-3pt}
\paragraph{Setting:} (i) Number of nodes = 12. (ii) Correlation of all the adjacent nodes = $\rho$. (iii) Alternate nodes have maximum noise ($q_i$ = 0 if $i~\%~2 = 0$, $q_i$ = $q_{max}$ if $i~\%2~ = 1$). (iv) Assume access to $\rho$. (v) Number of iterations = 1000 \\
For both, chain graphs and star graphs, we vary $\rho$ in $\{0.6, 0.8\}$ and $q_{max}$ in $\{0, 0.4\}$.

We would like to point out that $q_{max}$ is defined differently in our setting and in SGA; $q_{max}$ in our setting is twice the SGA's $q_{max}$. The final results are presented in Figures \ref{fig:our_vs_sga_chain} and \ref{fig:our_vs_sga_star} respectively. \vspace{-10pt}
\begin{figure}
    \centering
    \includegraphics[scale = 0.4]{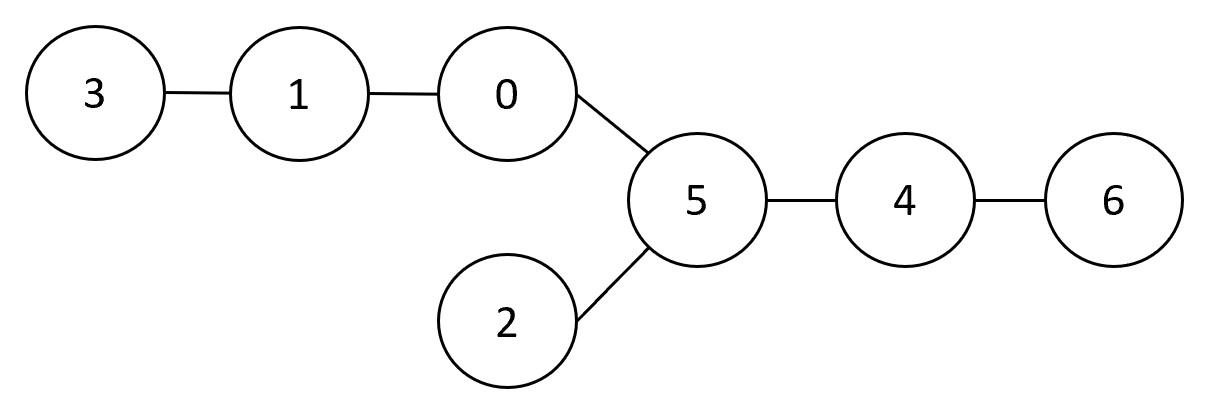}
    \caption{Randomly generated graph used for algorithm evaluation.}
    \label{fig:random_graph}
\end{figure}
\subsection{Support size, $k = 4$ (Identifiable Setting):} \vspace{-5pt}
In this part we see the impact of $\delta$ on the performance of the algorithm for different graphs. We execute the algorithm for a lot of randomly generated graphs and the algorithm converges to the correct output. We report the results for 3 different graph structures - star, chain and one of the many randomly generated graphs (Figure \ref{fig:random_graph}).
\paragraph{Setting}:
(i) Number of nodes = 7.\\
(ii) Graph Shape = \{Chain, Star, Random\}\\
(iii) Distance of all the adjacent nodes = $\exp(-0.7)$. \\
(iv) Error probability is uniformly sampled from $[0,0.2]$.\\
(v) $\delta\in \{0.00, 0.02, 0.04\}$\\
(vi) Assume access to $q_{max}$, $d_{min}$ but not to $d_{max}$, $t_0$.\\
(vii) Number of iterations = 100
\begin{figure}
    \centering
    \includegraphics[scale = 0.2]{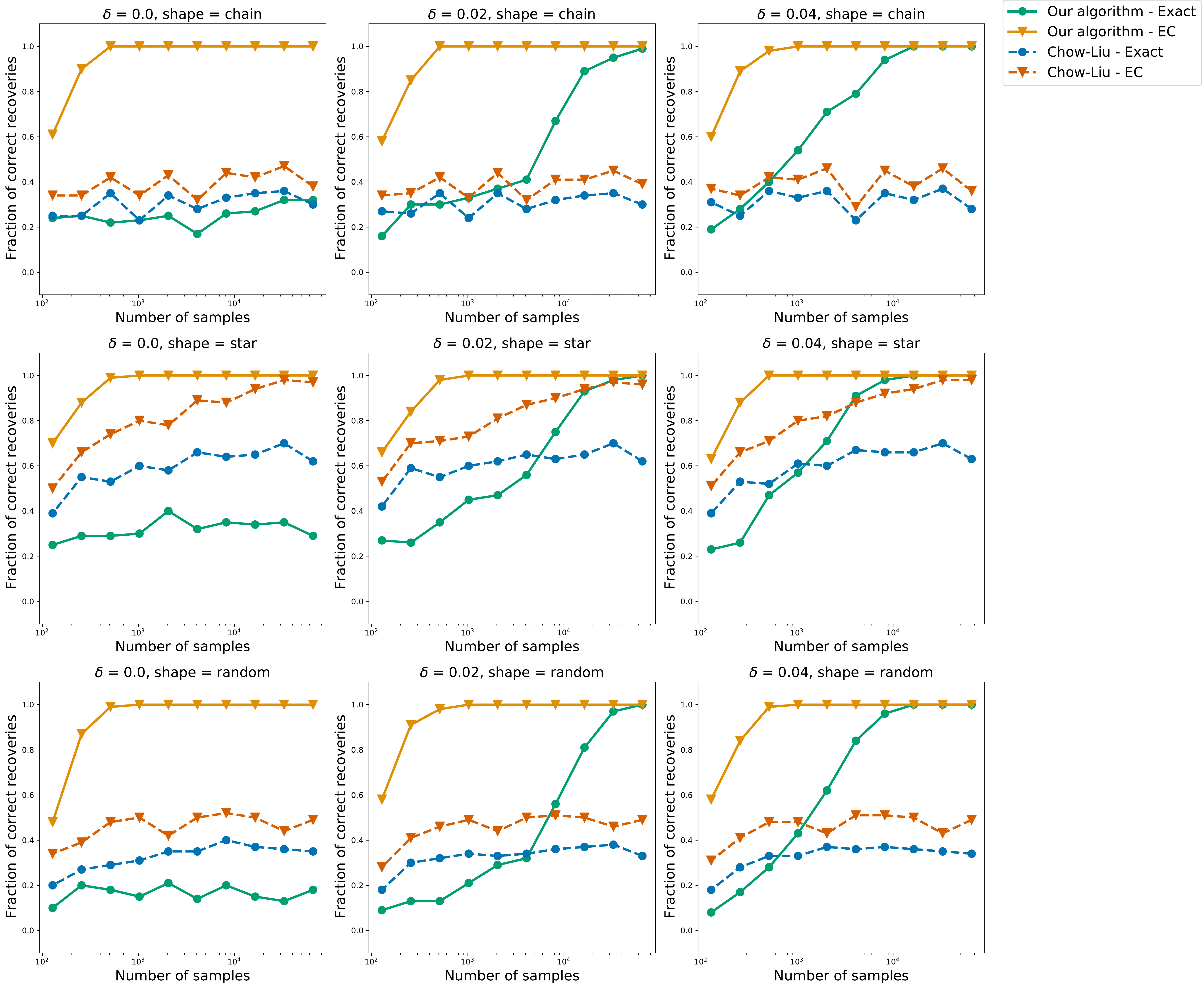}
    \caption{Comparing the performance of our
    algorithm and Chow-Liu over different values of $\delta_{i,j}\in \{0.00, 0.02, 0.04\}$ and different graph shapes - chain, star, random. Setting: $d_{min} = d_{max} = \exp(-0.7)$, $q_{max} = 0.2$, $\#$ of nodes$=7$. For both algorithms, we provide results for two cases: i) when the exact underlying tree is recovered, ii) when a tree from the equivalence class is recovered.}
    \label{fig:app_k_4_delta}
\end{figure}
\paragraph{Takeaways:}
\begin{enumerate}[leftmargin=*,noitemsep]\vspace{-5pt}
    \item We witness the transition from unidentifiability to identifiability. When $\delta = 0$, the exact graph cannot be recovered and hence the exact recovery fraction remains low consistently regardless of the number of samples. Higher $\delta$ has faster convergence to the correct graph.
    \item Learning a tree from the equivalence class requires much fewer samples.
    \item For the given noise model when the probability of error is randomly selected, for a significant number of realizations in the star shape, the Chow-Liu remains in the equivalence class. However, it lags behind considerably compared to our algorithm.
    \item Chow-Liu has high error for complete recovery.
\end{enumerate}
We also perform extensive experiments where we evaluate the impact of the probability of error and the distance between adjacent nodes and present the results in Appendix \ref{ap:exps}.
\vfill
\pagebreak
\bibliography{references.bib}

\begin{thebibliography}{10}

\bibitem{bhatia2007perturbation}
R.~Bhatia.
\newblock {\em Perturbation bounds for matrix eigenvalues}.
\newblock SIAM, 2007.

\bibitem{bhattacharyya2020near}
A.~Bhattacharyya, S.~Gayen, E.~Price, and N.~Vinodchandran.
\newblock Near-optimal learning of tree-structured distributions by
  {C}how-{L}iu.
\newblock {\em arXiv preprint arXiv:2011.04144}, 2020.

\bibitem{bresler2015efficiently}
G.~Bresler.
\newblock Efficiently learning ising models on arbitrary graphs.
\newblock In {\em Proceedings of the forty-seventh annual ACM symposium on
  Theory of computing}, pages 771--782. ACM, 2015.

\bibitem{bresler2014hardness}
G.~Bresler, D.~Gamarnik, and D.~Shah.
\newblock Hardness of parameter estimation in graphical models.
\newblock In {\em Advances in Neural Information Processing Systems}, pages
  1062--1070, 2014.

\bibitem{bresler2014structure}
G.~Bresler, D.~Gamarnik, and D.~Shah.
\newblock Structure learning of antiferromagnetic ising models.
\newblock In {\em Advances in Neural Information Processing Systems}, pages
  2852--2860, 2014.

\bibitem{bresler2016learning}
G.~Bresler and M.~Karzand.
\newblock Learning a tree-structured ising model in order to make predictions.
\newblock {\em arXiv preprint arXiv:1604.06749}, 2016.

\bibitem{bresler2020learning}
G.~Bresler, M.~Karzand, et~al.
\newblock Learning a tree-structured ising model in order to make predictions.
\newblock {\em Annals of Statistics}, 48(2):713--737, 2020.

\bibitem{bresler2008reconstruction}
G.~Bresler, E.~Mossel, and A.~Sly.
\newblock Reconstruction of markov random fields from samples: Some
  observations and algorithms.
\newblock In {\em Approximation, Randomization and Combinatorial Optimization.
  Algorithms and Techniques}, pages 343--356. Springer, 2008.

\bibitem{casanellas2021robust}
M.~Casanellas, M.~Garrote-L{\'o}pez, and P.~Zwiernik.
\newblock Robust estimation of tree structured models.
\newblock {\em arXiv preprint arXiv:2102.05472}, 2021.

\bibitem{chang1996full}
J.~T. Chang.
\newblock Full reconstruction of markov models on evolutionary trees:
  identifiability and consistency.
\newblock {\em Mathematical biosciences}, 137(1):51--73, 1996.

\bibitem{choi2011learning}
M.~J. Choi, V.~Y. Tan, A.~Anandkumar, and A.~S. Willsky.
\newblock Learning latent tree graphical models.
\newblock {\em Journal of Machine Learning Research}, 12(May):1771--1812, 2011.

\bibitem{chow1968approximating}
C.~Chow and C.~Liu.
\newblock Approximating discrete probability distributions with dependence
  trees.
\newblock {\em IEEE transactions on Information Theory}, 14(3):462--467, 1968.

\bibitem{daskalakis2020tree}
C.~Daskalakis and Q.~Pan.
\newblock Tree-structured ising models can be learned efficiently.
\newblock {\em arXiv preprint arXiv:2010.14864}, 2020.

\bibitem{goel2019learning}
S.~Goel, D.~M. Kane, and A.~R. Klivans.
\newblock Learning ising models with independent failures.
\newblock {\em arXiv preprint arXiv:1902.04728}, 2019.

\bibitem{hamilton2017information}
L.~Hamilton, F.~Koehler, and A.~Moitra.
\newblock Information theoretic properties of markov random fields, and their
  algorithmic applications.
\newblock In {\em Advances in Neural Information Processing Systems}, pages
  2463--2472, 2017.

\bibitem{circulant}
L.~S. (https://math.stackexchange.com/users/214617/leon sot).
\newblock Simple identity involving q-pochhammer symbol.
\newblock Mathematics Stack Exchange.
\newblock URL:https://math.stackexchange.com/q/2081765 (version: 2017-01-03).

\bibitem{katiyar2019robust}
A.~Katiyar, J.~Hoffmann, and C.~Caramanis.
\newblock Robust estimation of tree structured gaussian graphical models.
\newblock In {\em International Conference on Machine Learning}, pages
  3292--3300, 2019.

\bibitem{katiyar2020robust}
A.~Katiyar, V.~Shah, and C.~Caramanis.
\newblock Robust estimation of tree structured ising models.
\newblock {\em arXiv preprint arXiv:2006.05601}, 2020.

\bibitem{klivans2017learning}
A.~Klivans and R.~Meka.
\newblock Learning graphical models using multiplicative weights.
\newblock In {\em 2017 IEEE 58th Annual Symposium on Foundations of Computer
  Science (FOCS)}, pages 343--354. IEEE, 2017.

\bibitem{koller2009probabilistic}
D.~Koller and N.~Friedman.
\newblock {\em Probabilistic graphical models: principles and techniques}.
\newblock MIT press, 2009.

\bibitem{lake1994reconstructing}
J.~A. Lake.
\newblock Reconstructing evolutionary trees from dna and protein sequences:
  paralinear distances.
\newblock {\em Proceedings of the National Academy of Sciences},
  91(4):1455--1459, 1994.

\bibitem{lauritzen1996graphical}
S.~L. Lauritzen.
\newblock {\em Graphical models}, volume~17.
\newblock Clarendon Press, 1996.

\bibitem{lee2007efficient}
S.-I. Lee, V.~Ganapathi, and D.~Koller.
\newblock Efficient structure learning of markov networks using $ l\_1
  $-regularization.
\newblock In {\em Advances in neural Information processing systems}, pages
  817--824, 2007.

\bibitem{lindgren2019robust}
E.~M. Lindgren, V.~Shah, Y.~Shen, A.~G. Dimakis, and A.~Klivans.
\newblock On robust learning of ising models.
\newblock In {\em NeurIPS Workshop on Relational Representation Learning},
  2019.

\bibitem{nikolakakis2019learning}
K.~E. Nikolakakis, D.~S. Kalogerias, and A.~D. Sarwate.
\newblock Learning tree structures from noisy data.
\newblock In {\em The 22nd International Conference on Artificial Intelligence
  and Statistics}, pages 1771--1782, 2019.

\bibitem{nikolakakis2020information}
K.~E. Nikolakakis, D.~S. Kalogerias, and A.~D. Sarwate.
\newblock Information thresholds for non-parametric structure learning on tree
  graphical models, 2020.

\bibitem{pearl2014probabilistic}
J.~Pearl.
\newblock {\em Probabilistic reasoning in intelligent systems: networks of
  plausible inference}.
\newblock Elsevier, 2014.

\bibitem{pearl1986structuring}
J.~Pearl and M.~Tarsi.
\newblock Structuring causal trees.
\newblock {\em Journal of Complexity}, 2(1):60--77, 1986.

\bibitem{ravikumar2010high}
P.~Ravikumar, M.~J. Wainwright, J.~D. Lafferty, et~al.
\newblock High-dimensional ising model selection using l1-regularized logistic
  regression.
\newblock {\em The Annals of Statistics}, 38(3):1287--1319, 2010.

\bibitem{tan2011large}
V.~Y. Tan, A.~Anandkumar, L.~Tong, and A.~S. Willsky.
\newblock A large-deviation analysis of the maximum-likelihood learning of
  markov tree structures.
\newblock {\em IEEE Transactions on Information Theory}, 57(3):1714--1735,
  2011.

\bibitem{tan2010learning}
V.~Y. Tan, A.~Anandkumar, and A.~S. Willsky.
\newblock Learning gaussian tree models: Analysis of error exponents and
  extremal structures.
\newblock {\em IEEE Transactions on Signal Processing}, 58(5):2701--2714, 2010.

\bibitem{tandon2020exact}
A.~Tandon, V.~Y. Tan, and S.~Zhu.
\newblock Exact asymptotics for learning tree-structured graphical models with
  side information: Noiseless and noisy samples.
\newblock {\em arXiv preprint arXiv:2005.04354}, 2020.

\bibitem{tandon2021sga}
A.~Tandon, A.~H. Yuan, and V.~Y. Tan.
\newblock Sga: A robust algorithm for partial recovery of tree-structured
  graphical models with noisy samples.
\newblock {\em arXiv preprint arXiv:2101.08917}, 2021.

\bibitem{tropp2015introduction}
J.~A. Tropp.
\newblock An introduction to matrix concentration inequalities.
\newblock {\em arXiv preprint arXiv:1501.01571}, 2015.

\bibitem{wainwright2008graphical}
M.~J. Wainwright and M.~I. Jordan.
\newblock {\em Graphical models, exponential families, and variational
  inference}.
\newblock Now Publishers Inc, 2008.

\bibitem{wu2019sparse}
S.~Wu, S.~Sanghavi, and A.~G. Dimakis.
\newblock Sparse logistic regression learns all discrete pairwise graphical
  models.
\newblock In {\em Advances in Neural Information Processing Systems}, pages
  8071--8081, 2019.

\end{thebibliography}
\bibliographystyle{abbrv}
\pagebreak
\appendix

\section{Proof of Lemma 1}\label{ap:lemma1}
This proof relies on the classification of a set of 4 nodes as star/non-star. This graph theoretic concept was originally introduced in \cite{katiyar2019robust} where it was used to analyze Gaussian graphical models.
Any set of 4 nodes is classified as a non-star if the tree can be split into two subtrees with each subtree containing exactly 2 nodes. The nodes in the same subtree form a pair. If the 4 nodes do not form a non-star, they form a star. Figure \ref{fig:eq_cl} provides an example of the equivalence class and star/non-star classification.
\begin{figure}
    \centering
    \includegraphics[scale = 0.3]{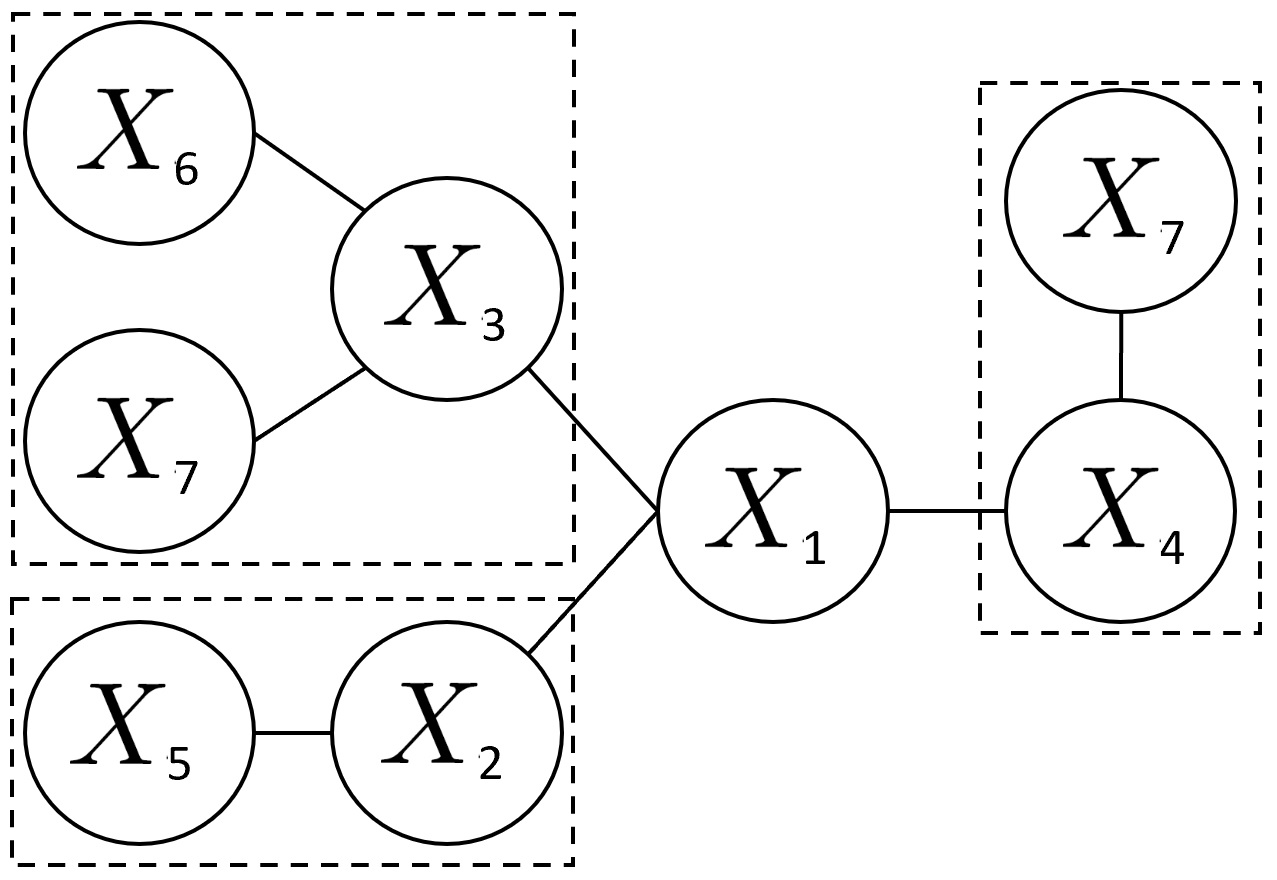}
    \caption{The equivalence class of this tree is given by all the permutations of nodes in the leaf clusters within the dotted regions. Nodes $(X_1, X_4, X_6, X_7)$ form a non-star. Nodes $(X_1, X_2, X_3, X_4)$ form a star.}
    \label{fig:eq_cl}
\end{figure}
Once a set of 4 nodes is classified as a star/non-star, Lemma \ref{le:lim_unid_gen} follows directly from the proof of Theorem 2 in \cite{katiyar2019robust}. The key idea is that the categorization of any set of 4 nodes as star/non-star completely defines all the splits of the tree into two subtrees with each subtree having at least 2 nodes. All of these splits can be combined to recover the equivalence class $\mathcal{T}_{T^*}$.

Next, we see how to categorize a set of 4 nodes as star non-star.
In \cite{katiyar2019robust} and \cite{katiyar2020robust}, the classification of a set of 4 nodes as star/non-star was done using pairwise correlations between random variables. Unfortunately, for random variables on support sizes larger than 2, correlation cannot be used to perform this classification. This is where we use the information distance metric $d_{i,j}$ as defined in Equation \eqref{eq:dist}.

A set of 4 nodes $(X_{1}, X_{2}, X_{3}, X_{4})$ forms a non-star with $(X_{1}, X_{2})$ forming a pair if:
\begin{equation*}
    d_{1',3'}+d_{2',4'} = d_{1',4'}+d_{2',3'} \neq d_{1',2'}+d_{3',4'}.
\end{equation*}
The set forms a star if:
\begin{equation*}
    d_{1',3'}+d_{2',4'} = d_{1',4'}+d_{2',3'} = d_{1',2'}+d_{3',4'}.
\end{equation*}
Next, we see why these conditions for star/non-star classification are correct.
\paragraph{Non-Star condition:} When any 4 nodes $(X_1, X_2, X_3, X_4)$ form a non-star such that $(X_1, X_2)$ form a pair, the 4 nodes can have one of the four configurations as shown in Figure \ref{fig:non_star}. There exist more configurations with $X_1$ and $X_2$ exchanging positions or $X_3$ and $X_4$ exchanging positions. Since $X_1$ and $X_2$ always occur interchangeably, the results continue to hold for the configurations where $X_1$ and $X_2$ exchange positions. Same argument holds for $X_3$ and $X_4$.

\begin{figure}
    \centering
    \includegraphics[scale = 0.4]{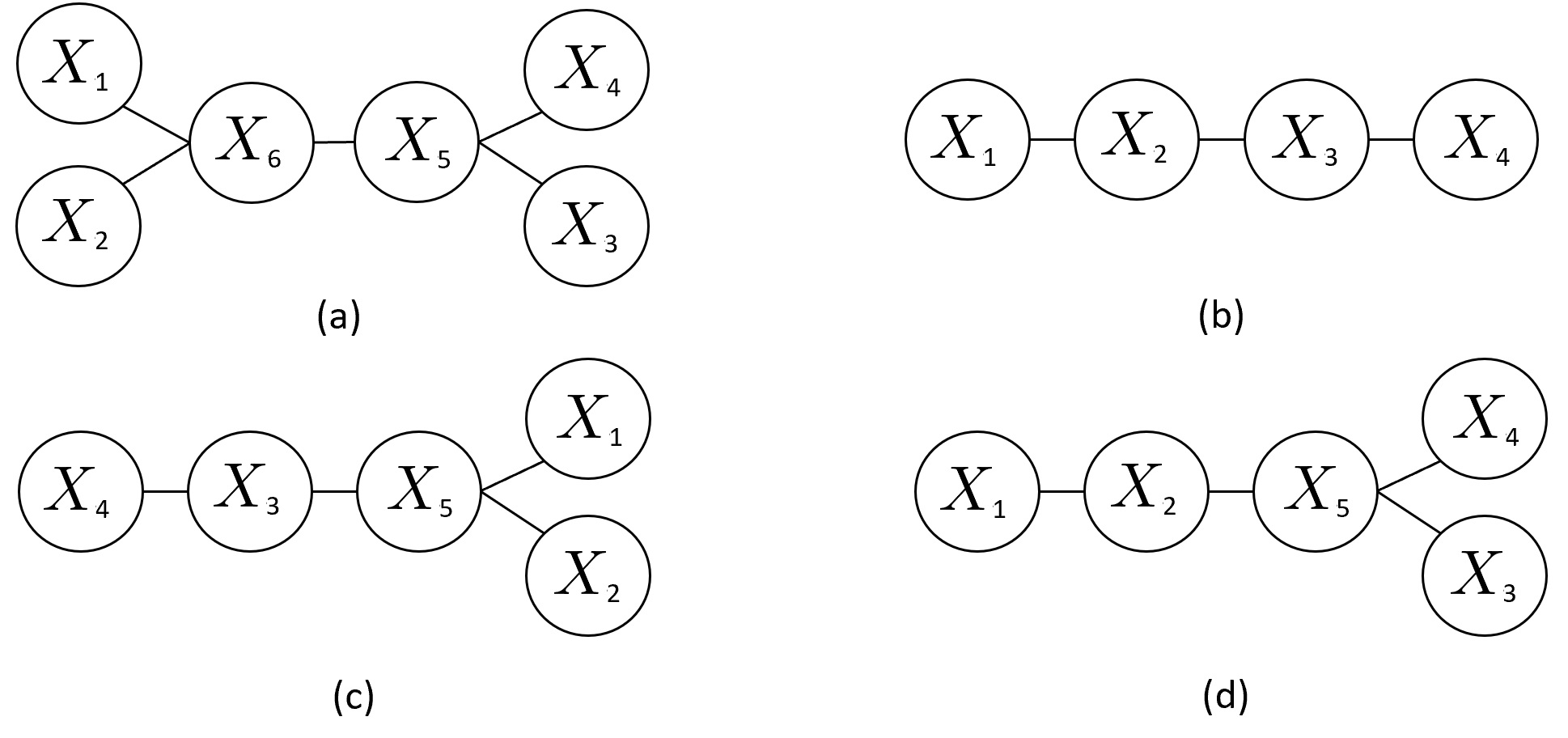}
    \caption{Four possible configurations of $(X_1, X_2, X_3, X_4)$ when they form a non-star such that $(X_1, X_2)$ form a pair.}
    \label{fig:non_star}
\end{figure}

Note that the distances $d_{i,j}$ are additive along the paths connecting $X_i$ and $X_j$. Therefore for all the cases, it is easy to see that: $$d_{1,3} + d_{2, 4} = d_{1,4} + d_{2, 3}.$$
Therefore we have that :
$$d_{1,3} + d_{2, 4} + d_{1,1'} + d_{2,2'} + d_{3,3'} + d_{4,4'}= d_{1,4} + d_{2, 3} + d_{1,1'} + d_{2,2'} + d_{3,3'} + d_{4,4'},$$ $$d_{1',3'} + d_{2', 4'} = d_{1',4'} + d_{2', 3'} (\text{As } d_{i',j'} = d_{i,i'} + d_{i,j} + d_{j,j'}).$$

Furthermore, one can see that $$d_{1,3} + d_{2, 4} - (d_{1,2} + d_{3,4})\geq 2d_{min}.$$
Adding and subtracting the noise distances again, we get that $$d_{1',3'} + d_{2', 4'} - (d_{1',2'} + d_{3',4'})\geq 2d_{min}.$$

\begin{figure}
    \centering
    \includegraphics[scale = 0.4]{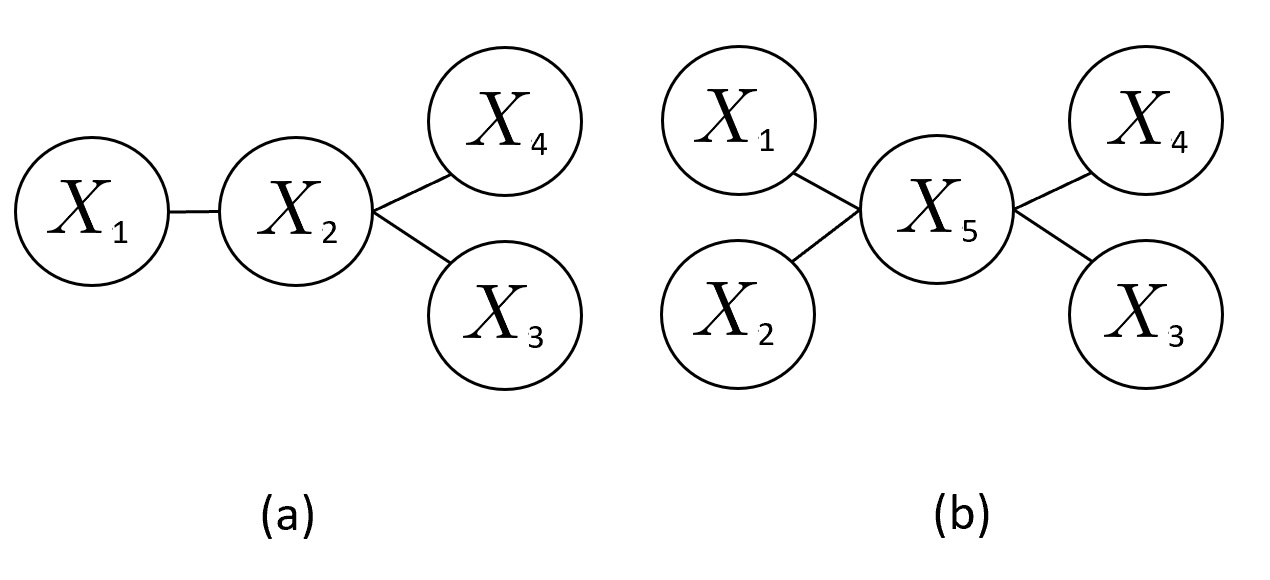}
    \caption{Two possible configurations of $(X_1, X_2, X_3, X_4)$ when they form a star.}
    \label{fig:star}
\end{figure}
\paragraph{Star condition:}
When the 4 nodes form a star, they can have either of the two configurations in Figure \ref{fig:star}. All the nodes are allowed to exchange positions with each other. Using the distance additivity  for this setting, it is easy to see that, for both the cases, $$d_{1,3} + d_{2, 4} = d_{1,4} + d_{2, 3} = d_{1,2} + d_{3,4}.$$
Furthermore using $d_{i',j'} = d_{i,i'} + d_{i,j} + d_{j,j'}$, we get that $$d_{1',3'} + d_{2', 4'} = d_{1',4'} + d_{2', 3'} = d_{1',2'} + d_{3',4'}.$$

This concludes the proof that the distances between noisy random variables can be used to classify a set of 4 nodes as star/non-star thereby proving that the only unidentifiability could possibly be within a leaf cluster.


\section{Obtaining Equation \eqref{eq:err_est_quad}}\label{ap:quadratic}
From Equation \eqref{eq:noisy_joint_pmf}, we have $P_{1',3'} = E_1P_{1,3}E_3$, $P_{1',2'} = E_1P_{1,2}E_2$, $P_{2',3'} = E_2P_{2,3}E_3$. From Equation \eqref{eq:noisy_pmf}, we have $P_{2'} = (1-q_2)P_2 + \frac{q_2}{k}I$. Substituting these in Equation \eqref{eq:cond_ind}, we get:
\begin{equation}\label{eq:err_est}
\begin{aligned}
   P_{2',3'} P_{1,' 3'}^{-1}P_{1',2'} =& E_2\frac{1}{(1-q_2)}\left(P_{2'} - \frac{q_2}{k}I\right)E_2 
   \end{aligned}
\end{equation}

\begin{equation}\label{eq:err_est_alt}
\begin{aligned}
   P_{2',3'} P_{1',3'}^{-1}P_{1',2'} =& E_2\frac{1}{(1-q_2)}\left(P_{2'} - \frac{q_2}{k}I\right)E_2
\end{aligned}
\end{equation}
\begin{equation}\label{eq:int1}
\begin{aligned}
    P_{2',3'} P_{1',3'}^{-1}P_{1',2'} =& \frac{E_2}{1-q_2}(1-q_2)\left(P_{2'} - \frac{q_2}{k}I\right)\frac{E_2}{1-q_2}\\
    \left(\frac{E_2}{1-q_2}\right)^{-1}P_{2',3'} P_{1',3'}^{-1}P_{1',2'}\left(\frac{E_2}{1-q_2}\right)^{-1} =& (1-q_2)\left(P_{2'} - \frac{q_2}{k}I\right)
   \end{aligned}
\end{equation}
Note that:
\begin{equation}
\begin{aligned}
    \frac{E_2}{1-q_2} = I + \frac{q_2O}{k(1-q_2)}\\
    \left(\frac{E_2}{1-q_2}\right)^{-1} = I - \frac{q_2O}{k}
\end{aligned}
\end{equation}
Substituting this back in Equation \eqref{eq:int1}
\begin{equation}\label{eq:int2}
    \begin{aligned}
    &\left(I - \frac{q_2O}{k}\right)P_{2',3'} P_{1',3'}^{-1}P_{1',2'}\left(I - \frac{q_2O}{k}\right) = (1-q_2)\left(P_{2'} - \frac{q_2}{k}I\right)\\
    &\frac{q_2^2}{k^2}(OP_{2',3'} P_{1',3'}^{-1}P_{1',2'}O - kI) - \frac{q_2}{k}(OP_{2',3'} P_{1',3'}^{-1}P_{1',2'} + P_{2',3'} P_{1',3'}^{-1}P_{1',2'}O - kP_{2'} - I) \\
    &\hspace{10em} + P_{2',3'} P_{1',3'}^{-1}P_{1',2'}-P_{2'} = 0
   \end{aligned}
\end{equation}
To simplify this, we observe that:
\begin{equation}
    \begin{aligned}
    OP_{2',3'} &= OP_{1',3'} = OP_{3}'\\
    P_{2',3'}O &= P_{2'}O\\
    P_{1',2'}O &= P_{1',3'}O = OP_{1}'\\
    OP_{1',2'} &= OP_{2'}
    \end{aligned}
\end{equation}
Substituting these back in Equation \eqref{eq:int2}, we get:
\begin{equation}
  \frac{q_2^2}{k^2}(O - kI) - \frac{q_2}{k}(OP_{2'} + P_{2'}O - kP_{2'} - I) + P_{2',3'} P_{1',3'}^{-1}P_{1',2'}-P_{2'} = 0
\end{equation}


\section{Proof of Theorem \ref{th:k_ary_iden}}\label{ap:k_ary_iden_proof}
\begin{proof}
Note that a graphical model on any subset of 3 nodes comprising of a leaf node $X_2$, it's parent $X_1$ and an arbitrary third node $X_3$ always forms a tree and satisfies $X_2\perp X_3|X_1$. However, due to the unidentifiability between $X_2$ and $X_1$, we don't know a priori whether $X_2\perp X_3|X_1$ or $X_1\perp X_3|X_2$. Therefore, we attempt to estimate the probability of error for both the cases using an equation equivalent to Equation (\ref{eq:err_est_x}). All the cases for which the equation has a feasible solution can explain the noisy observations. 

Clearly, the case corresponding to the ground truth $X_2\perp X_3|X_1$ has a solution. Now we see what happens when we check whether node $X_2$ is the middle node by solving Equation (\ref{eq:err_est}) when the ground truth has node 1 in the middle. That is, we try to estimate $\ind{\Tilde{q}}{2}{1}{3}$ when $X_2\perp X_3|X_1$.

In the current setting, we have:
\begin{equation*}
    P_{2,3} = P_{2, 1}P_1^{-1}P_{1,3}.
\end{equation*}
We also have:
\begin{equation*}
    P'_{2,3} = E_2P_{2,3}E_3, P'_{1,3} = E_1P_{1,3}E_3, P'_{1,2} = E_1P_{1,2}E_2.
\end{equation*}
Substituting these in Equations (\ref{eq:err_est}) and (\ref{eq:err_est_quad}), we get:
\begin{equation}\label{eq:err_est2}
\begin{aligned}
E_2P_{2,1}P_1^{-1}P_{1,2}E_2 = &\ind{\Tilde{E}}{2}{1}{3}\frac{1}{(1-\ind{\Tilde{q}}{2}{1}{3})}\left(P_{2'} - \frac{\ind{\Tilde{q}}{2}{1}{3}}{k}I\right)\ind{\Tilde{E}}{2}{1}{3}\\
& \text{ s.t.} 0\leq \ind{\Tilde{q}}{2}{1}{3} <1,
\end{aligned}
\end{equation}
\begin{equation}
\begin{aligned}\label{eq:quad2}
&\frac{(\ind{\Tilde{q}}{2}{1}{3})^2}{k^2}\left(O-kI\right) - \frac{\ind{\Tilde{q}}{2}{1}{3}}{k}\left(OP_{2'} + P_{2'}O-kP_{2'} -I\right)\\
&+E_2P_{2,1}P_1^{-1}P_{1,2}E_2 - P_{2'} = 0 \text{ s.t. }0\leq \ind{\Tilde{q}}{2}{1}{3}<1.
\end{aligned}
\end{equation}

Note that this equation does not depend on the random variable $X_3$. Therefore, whether a leaf node and its parent are unidentifiable depends solely on the joint distribution of the parent node $X_1$ and the noisy leaf node $X_2'$. When this equation does not have a solution, we can conclude that $X_2$ is a leaf node. Thus any tree in $\mathcal{T}_{T^*}$ which has $X_1$ as a leaf node can be ruled out.

Now, let us focus on the case when Equation (\ref{eq:quad2}) has a solution. We aim to obtain $\Tilde{\mathbf{X}}$ whose graphical model is $\Tilde{T}$. In order to do that, we assign the probability of error $\Tilde{q}_i$ which resulted in each of the observed noisy random variable $X_i'$ as follows:
\begin{equation}\label{eq:model_gen}
    \Tilde{q}_1 = 0, \Tilde{q}_2 = \ind{\Tilde{q}}{2}{1}{3}, \Tilde{q_i} = q_i \text{ } \forall i\notin \{1,2\}.
\end{equation}
Therefore we have that $\Tilde{X_i} = X_i$ $\forall i\notin \{1,2\}$.
Note that, by construction, this results in $\Tilde{X_1}\perp {X_i}|\Tilde{X_2}$ $\forall i\notin \{1,2\}$. We next prove that for any pair of nodes such that $X_{k_1}\perp X_{k_2}|X_1$ and $k_1, k_2\notin\{1,2\}$, we have that ${X}_{k_1}\perp {X}_{k_2}|\Tilde{X}_2$. This is equivalent to proving that $P_{k_1, k_2} = P_{k_1, \Tilde{2}}P_{\Tilde{2}}^{-1}P_{\Tilde{2},k_2}$ where $P_{k_1, \Tilde{2}}, P_{\Tilde{2}}$ and $ P_{\Tilde{2},k_2}$ are the joint PMF matrix of $X_{k_1}$ and  $\Tilde{X}_2$, diagonal marginal of $\Tilde{X}_2$, and the joint PMF matrix of $\Tilde{X}_2$ and $X_{k_2}$ respectively. We have that:
\begin{align*}
    P_{k_1, 2} = P_{k_1, 1}P_1^{-1}P_{1,2},\text{ }
    P_{2, k_2} = P_{2, 1}P_1^{-1}P_{k_2, 1}.
\end{align*}
Substituting these in $P_{k_1, k_2} = P_{k_1, 1}P_1^{-1}P_{1, k_2}$, we get:
\begin{align*}
    P_{k_1, k_2} = P_{k_1, 2}P_{1,2}^{-1}P_1P_{2,1}^{-1}P_{2, k_2}.
\end{align*}
Note that $P_{k_1, 2}E_2 = P_{k_1, \Tilde{2}}\ind{\Tilde{E}}{2}{1}{3}=P_{k_1, 2'}$. Using this along with Equation (\ref{eq:err_est2}) we get $P_{k_1, k_2} = P_{k_1, \Tilde{2}}P_{\Tilde{2}}^{-1}P_{\Tilde{2},k_2}$.

The above analysis of ruling out the trees with $X_1$ as a leaf node when Equation (\ref{eq:quad2}) does not have a solution and constructing $\Tilde{\mathbf{X}}$ when Equation (\ref{eq:quad2}) has a solution, holds true for every pair of parent and leaf nodes. Thus any tree in $\mathcal{T}_{T^*}\setminus \mathcal{T}_{T^*}^{sub}$ can be ruled out. Furthermore, for any tree $\Tilde{T}\in  \mathcal{T}_{T^*}^{sub}$ in which leaf nodes $\setl_{\Tilde{T}} \subseteq \setl^{sub}$ exchange positions with their parents, we can define the probability of error for $\Tilde{q}_i$ for every node $\Tilde{X}_i \in\Tilde{\mathbf{X}}$ as follows:

\begin{equation*}
\begin{aligned}
    \Tilde{q}_i &= \ind{\Tilde{q}}{i}{p_i}{3} \text{ }\forall i \in \setl_{\Tilde{T}},\\
    \Tilde{q}_{p_i} &= 0 \text{ }\forall i \in \setl_{\Tilde{T}},\\
    \Tilde{q}_i &= q_i \text{ otherwise},
\end{aligned}
\end{equation*}
where $X_{p_i}$ is the parent node of $X_i$. It is straightforward to see that the graphical model of $\Tilde{\mathbf{X}}$ is $\Tilde{T}$.
\end{proof}


\section{Proof of Theorem \ref{th:symmetric}}\label{ap:symmetric}
We first present a simple equation that helps in working with symmetric and perturbed symmetric models:
\begin{equation}\label{eq:mat_mul}
    \left(\alpha_1I + (1-\alpha_1)\frac{O}{k}\right) \left(\alpha_2I + (1-\alpha_2)\frac{O}{k}\right) = \left(\alpha_1\alpha_2I + (1-\alpha_1\alpha_2)\frac{O}{k}\right).
\end{equation}
When $X_2$ is a leaf node, $X_1$ is its parent node and $X_3$ is an arbitrary third node, $X_3\perp X_2|X_1$. This gives us:
\begin{equation*}
    P_{2,3} = P_{2,1}P_1^{-1}P_{1,2}.
\end{equation*}
Substituting this in $P_{2',3'}P_{1',3'}^{-1}P_{1',2'}$ while noting that $P_{a', b'} = E_aP_{a,b}E_b$, we get that:
\begin{equation}\label{eq:temp1}
P_{2',3'}P_{1',3'}^{-1}P_{1',2'} = E_2 P_{2,1}P_1^{-1}P_{1,2}E_2.
\end{equation}
Now, using $P_1 = I/k$, $P_{2|1} = \alpha_{1,2}I + (1-\alpha_{1,2})\frac{O}{k}$, $E_2 = (1-q_2)I + q_2\frac{O}{k}$ and Equation \ref{eq:mat_mul}, we get that:
$$
P_{2',3'}P_{1',3'}^{-1}P_{1',2'} = E_2 P_{2,1}P_1^{-1}P_{1,2}E_2 = \frac{1}{k}\left((1-q_2)^2\alpha_{1,2}^2I + (1 - (1-q_2)^2\alpha_{1,2}^2)\frac{O}{k}\right).
$$
With these expressions, along with $P_{2'} = \frac{I}{k}$, we now look at the quadratic in Equation \eqref{eq:err_est_x}.
\begin{align*}
    &\frac{x^2}{k^2}(O - kI) - \frac{x}{k}(OP_{2'} + P_{2'}O - kP_{2'} - I) + 
  P_{2',3'} P_{1,' 3'}^{-1}P_{1',2'}-P_{2'}\\
  =  &\frac{x^2}{k^2}(O - kI) - \frac{2x}{k}(O/k - I) + 
  \frac{1}{k}\left((1-q_2)^2\alpha_{1,2}^2I + (1 - (1-q_2)^2\alpha_{1,2}^2)\frac{O}{k}\right)-\frac{I}{k}\\
  = &\frac{(x-1)^2 - (1-q_2)^2\alpha_{2,1}^2}{k}(O-kI).
\end{align*}
Thus, Equation \ref{eq:err_est_x} has a solution $x = 1-(1-q_2)\alpha_{1,2}$.
\qed


\section{Proof of Theorem \ref{th:perturbed_symmetric}}\label{ap:perturbed_symmetric}
Using Equation \eqref{eq:temp1}, and recalling that $P_1 = P_{1'} = P_{2} = P_{2'} = \frac{I}{k}$, we have that:
\begin{align}
&\frac{x^2}{k^2}(O - kI) - \frac{x}{k}(OP_{2'} + P_{2'}O - kP_{2'} - I) + P_{2',3'} P_{1,' 3'}^{-1}P_{1',2'}-P_{2'}\\
=&\frac{x^2}{k^2}(O - kI) - \frac{2x}{k^2}(O -kI) +E_2 P_{2,1}P_1^{-1}P_{1,2}E_2 -\frac{I}{k}\\
=& \left(\frac{x-1}{k}\right)^2(O - kI) - \frac{O}{k^2} + kE_2P_{2,1}P_{1,2}E_2.
\end{align}
Substituting $E_2 = (1-q_2)I + q_2\frac{O}{k}$ and $P_{2,1} = (\alpha_{a,b} - \delta_{a,b})I + (1-\alpha_{a,b})\frac{O}{k}+\Delta_{a,b}$, we get:
\begin{align}
    E_2P_{2,1}P_{1,2}E_2 =& \left((1-q_2)I + q_2\frac{O}{k}\right)\left((\alpha_{a,b} - \delta_{a,b})I + (1-\alpha_{a,b})\frac{O}{k}+\Delta_{a,b}\right)\\
    &\left((\alpha_{a,b} - \delta_{a,b})I + (1-\alpha_{a,b})\frac{O}{k}+\Delta_{a,b}^T\right)\left((1-q_2)I + q_2\frac{O}{k}\right).
\end{align}

Now we have:
\begin{align*}
    E_2P_{2,1} = &\left((1-q_2)I + q_2\frac{O}{k}\right)\left((\alpha_{a,b} - \delta_{a,b})I + (1-\alpha_{a,b})\frac{O}{k}+\Delta_{a,b}\right)\\
    =& (1-q_2)(\alpha_{a,b} - \delta_{a,b})I+(1-q_2)(1-\alpha_{a,b})\frac{O}{k}+(1-q_2)\Delta_{a,b}\\
    &+q_2(\alpha_{a,b} - \delta_{a,b})\frac{O}{k}+q_2(1-\alpha_{a,b})\frac{O}{k}+q_2\delta_{a,b}\frac{O}{k}\\
    &= (1-q_2)(\alpha_{a,b} - \delta_{a,b})I+(1-(1-q_2)\alpha_{a,b})\frac{O}{k}+(1-q_2)\Delta_{a,b}
\end{align*}
Define $\alpha'_{a,b} \triangleq (1-q_2)\alpha_{a,b}, {\delta'}_{a,b} \triangleq (1-q_2)\delta'_{a,b}$ and $\Delta'_{a,b} = (1-q_2)\Delta_{a,b}$, we get:
\begin{align*}
    E_2P_{2,1}= (\alpha'_{a,b} - {\delta'}_{a,b})I + (1-\alpha'_{a,b})\frac{O}{k} + \Delta'_{a,b}.
\end{align*}
Noting that $P_{1,2}E_2 = (E_2P_{2,1})^T$, we get:
\begin{align*}
    E_2P_{2,1}P_{1,2}E_2= \frac{1}{k^2}\left(((\alpha'_{a,b}-\delta'_{a,b})^2+(\delta'_{a,b})^2)I+\frac{O}{k}(1-(\alpha'_{a,b})^2)+(\alpha'_{a,b} - {\delta'}_{a,b})((\Delta'_{a,b})^T+\Delta'_{a,b})\right)
\end{align*}
This gives us:
\begin{align*}
    Q^2(x) =& \|\left(\frac{x-1}{k}\right)^2(O - kI) - \frac{O}{k^2} + k E_b P_{b,a}P_{a,b}E_b\|_F^2\\
    =& \|\left(\frac{x-1}{k}\right)^2(O - kI) + (({\alpha'}_{a,b}-{\delta'}_{a,b})^2+{\delta'}_{a,b}^2)\frac{I}{k}-{\alpha'}_{a,b}^2\frac{O}{k^2}+\frac{({\alpha'}_{a,b} - {\delta'}_{a,b})}{k}({\Delta'}_{a,b}^T+{\Delta'}_{a,b}) \|_F^2\\
\end{align*}

Each diagonal element (total $k$) of the matrix is $\left(\frac{x-1}{k}\right)^2 - \frac{(x-1)^2}{k}+\frac{(\alpha'_{a,b}- {\delta'}_{a,b})^2+ {\delta'}_{a,b}^2}{k}-\frac{ {\alpha'}_{a,b}^2}{k^2}$.\\
Each element at the positions of the support $(\Delta'_{a,b} + {\Delta'}_{a,b}^T)$ (total $2k$) is $\left(\frac{x-1}{k}\right)^2-\frac{ {\alpha'}_{a,b}^2}{k^2}+\frac{ {\delta'}_{a,b}(\alpha'_{a,b}- {\delta'}_{a,b})}{k}$.\\
Every remaining element (total $k^2-3k$) is $\left(\frac{x-1}{k}\right)^2-\frac{ {\alpha'}_{a,b}^2}{k^2}$.
To simplify the above equation, we define $\gamma = (1-x)^2 -  {\alpha'}_{a,b}^2$, $e = {\delta'}_{a,b}(\alpha'_{a,b}-\delta'_{a,b})$.
Each diagonal element is $\frac{\gamma}{k^2} - \frac{\gamma}{k}-\frac{2e}{k}$.\\
Each element at the positions of the support $(\Delta'_{a,b} + {\Delta'}_{a,b}^T)$ (total $2k$) is $\frac{\gamma}{k^2} + \frac{e}{k}$.\\
Every remaining element (total $k^2-3k$) is $\frac{\gamma}{k^2}$.
Thus, we get:
\begin{align*}
    Q^2(x) =& k\left(\frac{\gamma}{k^2} - \frac{\gamma}{k}-\frac{2e}{k}\right)^2 + 2k\left(\frac{\gamma}{k^2} + \frac{e}{k}\right)^2 + (k^2-3k)\frac{\gamma^2}{k^4}\\
    =&  \tfrac{1}{k^3}\left((k-1)\gamma + 2k e\right)^2 + \tfrac{2}{k^3}\left(\gamma + ke\right)^2 + \tfrac{k-3}{k^3}\gamma^2
\end{align*}
$Q^2(x)$ is minimized for $\gamma = -\frac{2ke}{k-1}$. Substituting this, we get:
$$
Q^2(x)\geq \frac{2(k-3)e^2k^2}{k-1}.
$$
When $k>4$, $Q^2(x)\geq 0$. This completes the proof that when $k>4$, Equation \eqref{eq:err_est_x} does not have a solution. 

Next we look at the case when $k = 3$. For  $k = 3$, when $\gamma = -3e$, we get $Q^2(x) = 0$. The only thing that remains is to check that $\gamma = -3e$ corresponds to a valid solution of $x$.
\begin{align*}
    &(1-x)^2 -  {\alpha'}_{a,b}^2 = \gamma\\
    &(1-x)^2 -  {\alpha'}_{a,b}^2 + 3e  = 0\\
    &(1-x)^2 =  {\alpha'}_{a,b}^2 - 3 {\delta'}_{a,b}(\alpha'_{a,b}- {\delta'}_{a,b})\\
\end{align*}
Note that $ {\alpha'}_{a,b}^2 - 3 {\delta'}_{a,b}(\alpha'_{a,b}- {\delta'}_{a,b})\geq \frac{ {\alpha'}_{a,b}^2}{4}$. Also note that for $P_{2|1}$ to be a valid PMF, we need that $\alpha>\delta, 0<\alpha<1$. Under these constraints, it is easy to see that $ {\alpha'}_{a,b}^2 - 3 {\delta'}_{a,b}(\alpha'_{a,b}-\delta'_{a,b})\leq 1$.
Therefore $(1-x)^2 =  {\alpha'}_{a,b}^2 - 3\delta'_{a,b}(\alpha'_{a,b}-\delta'_{a,b})$ has a solution for $0\leq x\leq 1$. This concludes the proof that for $k=3$, solution to Equation \eqref{eq:err_est_x} always exists. In other words, for $k=3$ the joint PMF matrix being circulant is a sufficient condition for unidentifiability.

Next we go on to prove that for $k=3$, the joint PMF matrix being circulant is also a necessary condition for unidentifiability. In order to arrive at this, note that, from Equation \ref{eq:err_est_alt}, a solution exists for Equation \eqref{eq:err_est_x} if and only if it exists for:
\begin{equation}\label{eq:err_est_alt_1}
   P_{2',3'} P_{1',3'}^{-1}P_{1',2'} = \tilde{E}_{2}^{1,3}\frac{1}{(1-\tilde{q}_2^{1,3})}\left(P_{2'} - \frac{\tilde{q}_2^{1,3}}{k}I\right)\tilde{E}_{2}^{1,3}\text{ s.t. } 0\leq \tilde{q}_2^{1,3}<1. 
\end{equation}
Recall that $\tilde{E}_{2}^{1,3} = (1-\tilde{q}_2^{1,3})I + \tilde{q}_2^{1,3}\frac{O}{k}$
We would like to prove that if Equation \eqref{eq:err_est_alt_1} has a solution then the matrix $P_{2,1}$ is circulant.
Since $P_{2'} = \frac{I}{k},P_1 = \frac{I}{k}, P_{2',3'} P_{1',3'}^{-1}P_{1',2'} = E_2P_{2,1}P_1^{-1}P_{1,2}E_2$, we have that for some $0\leq \tilde{q}_2^{1,3}<1$:
\begin{equation}\label{eq:int3}
   9P_{2,1}P_{1,2} = E_2^{-1}\tilde{E}_{2}^{1,3}\tilde{E}_{2}^{1,3}E_2^{-1}.
\end{equation}
Note that $E_2^{-1} = ((1-q_2)I + q_2\frac{O}{k})^{-1} = \frac{1}{1-q_2}(I + \frac{q_2}{1-q_2}\frac{O}{k})^{-1} = \frac{1}{1-q_2}(I - \frac{\frac{q_2}{1-q_2}\frac{O}{k}}{1+\frac{q_2}{1-q_2}})$ (using Woodbury Matrix Identity).
Simplifying, we get:
$$
E_2^{-1} = \frac{1}{1-q_2}(I - q_2\frac{O}{k}) = \frac{1}{1-q_2}I + (1-\frac{1}{1-q_2})\frac{O}{k}.
$$
Now, using Equation \eqref{eq:mat_mul}, we get:
$$
E_2^{-1}\tilde{E}_{2}^{1,3} = \frac{1-\tilde{q}_2^{1,3}}{1-q_2}I + \left(1-\frac{1-\tilde{q}_2^{1,3}}{1-q_2}\right)\frac{O}{k}
$$
Again, using Equation \eqref{eq:mat_mul}, we get:
$$
E_2^{-1}\tilde{E}_{2}^{1,3}\tilde{E}_{2}^{1,3}E_2^{-1} =(E_2^{-1}\tilde{E}_{2}^{1,3})^2= \left(\frac{1-\tilde{q}_2^{1,3}}{1-q_2}\right)^2I + \left(1-\left(\frac{1-\tilde{q}_2^{1,3}}{1-q_2}\right)^2\right)\frac{O}{k}.
$$
We note that in Equation \eqref{eq:int3}, the RHS has equal off-diagonal elements and equal diagonal elements.\\ 
Before proceeding further, for the ease of notation, we define $M = 3P_{1,2}$ and $M_i$ is the  $i^{th}$ column of $M$.\\
Since Equation \eqref{eq:int3} has a solution, we have the following properties of $M$:
\begin{enumerate}
    \item $M$ is doubly stochastic (as $P_1 = P_2 = I/3$),
    \item $||M_i||_2 = ||M_j||_2$ $\forall i,j\in \{1,2,3\}$ (as the diagonal elements of $M^T M$ are equal),
    \item $<M_i, M_j>$ is equal $\forall i\neq j\in \{1,2,3\}$ (as the off-diagonal elements of $M^T M$ are equal).
\end{enumerate}
These properties can hold true only if the columns of $M$ are circulant. In order to see this, note that:
\begin{enumerate}
    \item A necessary condition for property 1 is that $M_1, M_2$ and $M_3$ lie on the probability simplex.
    \item For property 2 to hold, $M_1, M_2$ and $M_3$ lie on a circle on the plane of the probability simplex with center at $(1/3, 1/3, 1/3)$.
    \item For property 3 to hold, $M_1, M_2$ and $M_3$ lie on an equilateral triangle of this circle.
\end{enumerate}

\begin{figure}
    \centering
    \includegraphics[scale = 0.6]{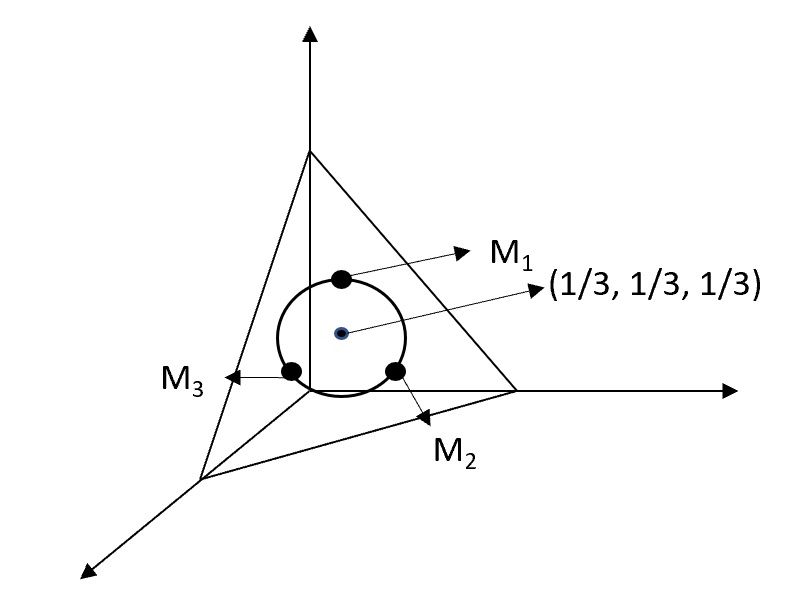}
    \caption{Position of the three column vectors of matrix $M$ for unidentifiability.}
    \label{fig:M}
\end{figure}

This can be visualized in Figure (\ref{fig:M}).
In order to see that they would be circulant, note that once we are given the vector $M_1$, vectors $M_2$ and $M_3$ are also determined. Given that we know that circulated versions of $M_1$ satisfy 1, 2 and 3, vectors $M_2$ and $M_3$ have to be the circulated $M_1$.

\section{Proof of Lemma \ref{le:bin_sol}}\label{ap:bin_sol}
We first analyze what happens to the solution of Equation \eqref{eq:err_est_x} for 3 nodes $(X_1, X_2, X_3)$ such that no 2 nodes are independent conditioned on the third. That is, their marginal distribution is not tree structured. We perform this analysis for general support size $k>2$. In this case, there exists another node, say $X_4$, such that $X_1\perp X_2\perp X_3|X_4$. This analysis is going to be useful in the proof of Lemma \ref{le:bin_sol} as well as the algorithm design.

\begin{lemma}\label{le:non_tree}
Consider any three nodes $(X_1, X_2, X_3)$ in a tree graphical model whose marginals are not tree structured. Then there exists a node $X_4$ such that $X_1\perp X_2\perp X_3|X_4$. Solving Equation \eqref{eq:err_est_x} outputs $X_2$ as a potential center node among $(X_1, X_2, X_3)$ if and only if it outputs $X_2$ as a potential center node among $(X_4, X_2, X_3)$
\end{lemma}
\begin{proof}
In this setting, we would like to estimate the probability of error of $X_2$ using Equation \eqref{eq:err_est_x}. We have that:
\begin{equation*}
\begin{aligned}
    P_{2,3} = P_{2,4}P_4^{-1}P_{4,3},\\ P_{1,3} = P_{1,4}P_4^{-1}P_{4,3},\\ P_{1,2} = P_{1,4}P_4^{-1}P_{1,2}
    \end{aligned}
\end{equation*}
Using these expressions coupled with Equation \eqref{eq:noisy_joint_pmf} and substituting them in Equation \eqref{eq:err_est_quad} we get the following quadratic equation:
\begin{equation}
\begin{aligned}\label{eq:quad3}
&\frac{(\ind{\Tilde{q}}{2}{1}{3})^2}{k^2}\left(O-kI\right) - \frac{\ind{\Tilde{q}}{2}{1}{3}}{k}\left(OP_2' + P_2'O-kP_2' -I\right)+E_2P_{2,4}P_4^{-1}P_{4,2}E_2 - P_2' = 0.
\end{aligned}
\end{equation}
This is the same equation with $X_1$ replaced by $X_4$.
\end{proof}
Next we go on to prove Lemma \ref{le:bin_sol}
\begin{proof}
First, let us look at the case when $(X_1, X_2, X_3)$ form a tree. If $X_1\perp X_3|X_2$, solution to Equation \eqref{eq:err_est_x} exists and it recovers the true error for $X_2$. We see what happens when $X_2\perp X_3|X_1$. We consider the case when there is no noise in $X_2$ and $X_3$. This is analysis is sufficient, as even if there was independent noise in $X_2$ and $X_3$, we would have had $X_2'\perp X_3'|X_2$. Thus we can assume that $X_2$ and $X_3$ already have the noise factored in.

For this case, we know that Equation (\ref{eq:err_est_x}) boils down to Equation (\ref{eq:quad2}) with $E_2 = I$. Using basic algebra, we see that all the quadratic equations corresponding to the different matrix components are equal to the following:
\begin{equation}\label{eq:bin_quad}
\begin{aligned}
    \frac{(\ind{\Tilde{q}}{2}{1}{3})^2}{4} -  \frac{(\ind{\Tilde{q}}{2}{1}{3})}{2} + \frac{(P_{2,1})_{0,0}(P_{2,1})_{1,0}}{(P_{2,1})_{0,0}+(P_{2,1})_{1,0}} + \frac{(P_{2,1})_{0,1}(P_{2,1})_{1,1}}{(P_{2,1})_{0,1}+(P_{2,1})_{1,1}} = 0 \text{ s.t. } 0\leq \ind{\Tilde{q}}{2}{1}{3}<1
    \end{aligned}
\end{equation}

Since the entries of $P_{2,1}$ are positive and sum up to 1, the smallest root of this equation is 0 (when one of $(P_{2,1})_{0,0},(P_{2,1})_{1,0}$ and one of $(P_{2,1})_{0,1},(P_{2,1})_{1,1}$ are 0) and the largest root is 1 (when all entries of $P_{2,1}$ are $1/4$). Since $P_{2,1}$ is full rank, we can conclude that Equation (\ref{eq:bin_quad}) has a solution.

Next, consider the case when $(X_1, X_2, X_3)$ do not form a tree. There exists a node $X_4$ such that $X_1\perp X_2\perp X_3| X_4$. Using the above result, we know that Equation (\ref{eq:err_est_x}) has a solution when we estimate the probability of error of $X_2$ which enforces $X_4\perp X_3|X_2$. Using Lemma \ref{le:non_tree}, we conclude that Equation \eqref{eq:err_est_x} has a solution which enforces $X_1\perp X_3|X_2$.
\end{proof}

\section{Algorithm Details}\label{ap:algo}
In this section, we provide the details of the algorithm to recover the tree upto unidentifiability. When we have access to $t_0$ (Assumption \ref{ass:fin_sample_err_est}), we can recover $\mathcal{T}_{T^*}^{sub}$. In the absence of the knowledge of $t_0$ , the algorithm returns one tree from $\mathcal{T}_{T^*}^{sub}$. We discuss the details after presenting the pseudocode. Also, if we have prior knowledge that the tree is identifiable only upto $\mathcal{T}_{T^*}$ (for instance, when $k=2$ or for symmetric models), we can gain in runtime by $\mathcal{O}(n)$.\\

\paragraph{Obtaining $\eta_{max}$} We first prove that $\eta_{max} = (1-k)\log (1-q_{max})-0.5 k\log (kp_{min})$. First note that for any node $X_i$, we have that:
$$
P_{i'|i} = (1-q_i) I + q_i\frac{O}{k}.
$$
Note that:
$$
d_{i',i} = -\log\left(\frac{|det(P_{i',i})|}{\sqrt{det(P_{i'})det(P_{i})}}\right) = -\log \left(det(P_{i'|i})\sqrt{\frac{det(P_i)}{det(P_{i'})}}\right).
$$
Using the matrix determinant lemma, we get $det(P_{i'|i}) = (1-q_i)^{k-1}$. Also ${det(P_{i'})}<(1/k)^k$ and ${det(P_{i})}\geq p_{min}^k$. This gives us:

$$
d_{i',i}\leq (1-k)\log (1-q_i) - 0.5k\log (kp_{min})  \triangleq \eta_{max}
$$
\paragraph{Neighborhood Vectors} We define for each node $X_i$, a neighborhood vector $N(X_i)$, which is the array of nodes $X_j$ sorted by $d_{i',j'}$ in ascending order and only contains nodes such that $d_{i',j'}$ is smaller than a threshold $t_{real}$. This is given as follows:
\begin{equation}
    N(X_i) = sort({X_j: d_{i',j'} \leq t_{real}}, \text{ key } = d_{i',j'})
\end{equation}
The threshold is $t_{real} = 4d_{max} + 3\eta_{max}$. 

\subsection{Pseudocode and runtime analysis}

We first provide the pseudocode for the two building blocks - \textsc{FindCenter} and \textsc{QuadraticError}. \textsc{FindCenter} returns the center node among 3 nodes as long as no 2 nodes are in the same leaf cluster. Otherwise it returns the nodes that belong to the same leaf cluster.
\textsc{QuadraticError} is used by the \textsc{LeafClusterResolution} routine to find the parent node within a leaf cluster.
Using these, we present the \textsc{FindLeafParent} subroutine that returns a leaf parent pair given an active set of nodes that form a subtree. 

\subsubsection{\textsc{QuadraticError}}
In this subroutine, we test if Equation \eqref{eq:err_est_x} has a solution. Note that the quadratic in Equation \eqref{eq:err_est_x} with matrix coefficients is equivalent to having $k^2$ quadratic equations. Equation \eqref{eq:err_est_x} has a solution if all the $k^2$ quadratic equations have a common root in $[0,q_{max}]$. Since we are working with the finite sample empirical estimates of the PMFs, we do not get an exact solution. To work in the finite sample domain, we find the mean of the root of all the $k^2$ quadratic equations and use that as an estimate of the common root. We return the Frobenius norm of the quadratic with the estimated root plugged in.

\begin{algorithm}[H]
\caption{Find the Error of the quadratic in Equation \eqref{eq:err_est_x}}
Input - Pairwise noisy distributions, a set of 3 nodes, test center node among the three nodes.\\
Output - Error of the quadratic in Equation \eqref{eq:err_est_x}.
\begin{algorithmic}[1]
\Procedure{QuadraticError}{$P_{i',j'}, NodeTriplet, TestCenter$} 
\State $A,B,C\gets $ Matrix Quadratic Coefficients from Equation \eqref{eq:err_est_x} for given $ NodeTriplet, TestCenter$.
\State $MeanRoot\gets 0$
\For{$i_1$ in $1\dots k$}
\For{$i_2$ in $1\dots k$}
\State $MeanRoot\gets MeanRoot + \frac{root(A[i_1,i_2]x^2+B[i_1, i_2]x+C[i_1, i_2])}{k^2}$
\EndFor
\EndFor
\Return $\|A (MeanRoot)^2 + B(MeanRoot) + C\|_F$
\EndProcedure
\end{algorithmic}
\end{algorithm}

\subsubsection{\textsc{FindCenter}}
The key idea is based on the observation that for any 3 nodes $(X_1, X_2, X_3)$, if $X_2$ is the center node, then any set of 4 nodes $(X_1, X_2, X_3, j)$ which forms a non-star, never has $(X_2, j)$ as a pair. Thus we can scan through all the nodes $j$ and rule out the nodes that pair with $j$. This procedure could potentially detect a leaf node as the center node if its parent is the center node. However, this is as expected since using the star/non-star procedure, it is impossible to differentiate between leaf and parent nodes.

\begin{algorithm}[H]
\caption{Recover Center Node in the Unidentifiable setting} \label{alg:center_unidentifiable}
Input - Pairwise noisy distributions and 3 nodes\\
Output - Candidate Center Nodes
\begin{algorithmic}[1]
\Procedure{FindCenter}{$P_{i',j'}, NodeTriplet$} 
\State $x\gets NodeTriplet[0], y\gets NodeTriplet[1], z\gets NodeTriplet[2]$
\State $CenterCand\gets \{x,y,z\}$
\For{$j\in N(x)\cap N(y)\cap N(z)$}
\If{$(x,y,z,j)$- Non-star and $pair(j)\in CenterCand$}
\State $CenterCand\gets CenterCand\setminus pair(j)$
\EndIf
\EndFor
\Return $CenterCand$
\EndProcedure
\end{algorithmic}
\end{algorithm}



\subsubsection{\textsc{GetLeafParent}}
This routine finds a leaf parent pair given an active set of nodes that form a subtree. We maintain two nodes -  a left node $l$, and a right node $r$.  The idea is to move both the nodes towards the right side till $r$ is a leaf node and $l$ is its parent node. 
In order to do this we consider a third node $z$ and perform the following operations:
\vspace{-0.7pc}
\begin{enumerate}[leftmargin = *]
    \item If the center node in $(l,r,z)$ is $z$, we shift node $l$ to node $z$,
    \item If the center node in $(l,r,z)$ is $r$, we shift node $l$ to node $r$ and node $r$ to node $z$.
\end{enumerate}
\vspace{-0.7pc}
\paragraph{Selecting nodes $l$, $r$ and $z$:}
When the \textsc{GetLeafParent} subroutine is called for the first time, node $r$ is randomly initialized. For any subsequent calls to \textsc{GetLeafParent}, node $r$ is initialized to one of the nodes that was detected as a parent node in the previous iterations and is still in the active set. $l$ is initialized to the node closest to $r$ in terms of $d_{i',j'}$. $z$ is obtained by iterating through $ N(X_i)\setminus l$ in the increasing order of distance.

When for a given $(l,r,z)$, there are more than one candidate center nodes, we conclude that they belong to the same leaf cluster. We check if we have already discovered the right node in one of the previous iterations if we have, we return the leaf parent pair. Otherwise, we attempt to find the parent node in that leaf cluster using the \textsc{LeafClusterResolution} routine.
\paragraph{Further robustifying \textsc{FindCenter}:}
At any point in the algorithm, suppose in the previous iterations we have recovered the edges $\{z,z_1\},\{z,z_2\}, \dots \{z,z_j\}$, then all the star/non-star tests involving $(l,r,z,z_i)$ $\forall i\in \{1, 2, \dots j\}$ are have the same star/non-star characterization and if they are non-star then $z_i$ pairs with $z$ in all the tests. We have the same phenomena for the already recovered edges of $l$ and $r$. Thus, when executing the algorithm with finite samples, we can robustify the \textsc{FindCenter} subroutine by considering all the nodes whose edge with node $z$ has been recovered and assign them the same star/non-star classification as the majority. We do the same for nodes $l$ and $r$ also.

\begin{algorithm}[H]
\caption{Find a leaf parent pair.}\label{alg:find_lp}
Input - Pairwise noisy distributions and Active nodes\\
Output - Leaf Node and its parent in the subtree of Active Nodes.
\begin{algorithmic}[1]
\Procedure{GetLeafParent}{$P_{i',j'}$, $ActiveSet$, $Edges$, $Parents$} 
\If{$|ActiveSet\cap Parents| > 0$}
\State $r \gets ActiveSet\cap Parents [0]$
\Else
\State $r\gets ActiveSet[0]$
\EndIf
\State $l \gets N(r)[0]\cap ActiveSet$
\State $i\gets 1, visited\gets \{l,r\}$
\While{$i < len(N(r))$}
\State $z \gets N(r)[i]$
\If {$z\in visited$ or $z\notin ActiveSet$}
\State{$i \gets i+1$}
\State continue
\EndIf
\State $visited\gets visited\cup z$
\State $C\gets $\textsc{FindCenter}$(P_{i',j'}, (l,r,z))$
\If{$|C| == 1$}
\State $l\_r\_order = True$
\EndIf
\If{$C == z$}
\State{$l\gets z$}
\ElsIf{$C == r$}
\State $l\gets r, r\gets z, i\gets 0$
\ElsIf{$|C| >1 $}
\If{$l\_r\_order == True$ and $r, l \in C$}
\State break
\EndIf
\State $r,l \gets \textsc{LeafClusterResolution}(C, Parents, ActiveSet)$
\State break
\EndIf
\EndWhile
\Return $r,l$
\EndProcedure
\end{algorithmic}
\end{algorithm}

\subsubsection{\textsc{LeafClusterResolution}}
When we have more than one nodes from the same leaf cluster, we find the parent node of that leaf cluster. If one of the nodes has been detected as a parent node in an earlier iteration, it is selected as the parent node. Otherwise, we perform the following operation on every subset of two nodes $X_{i_1}, X_{i_2}$ in $C$:
\begin{enumerate}
    \item Consider a third node $X_{i_3}\in X_{i_1}\cap X_{i_2}$. 
    \item Check if $X_{i_3}$ also belongs to the same leaf cluster as $X_{i_1} $ and $ X_{i_2}$.
    \begin{enumerate}
        \item If $X_{i_3}$ is not in the same leaf cluster, record the value $Q^2(x)$ in Equation \eqref{eq:err_est_x}, for two cases - (i) if $X_{i_1}$ is the center node, (ii) if $X_{i_2}$ is the center node.
        \item If $X_{i_3}$ is in the same leaf cluster, record the value $Q^2(x)$ in Equation \eqref{eq:err_est_x}, for three cases - (i) if $X_{i_1}$ is the center node, (ii) if $X_{i_2}$ is the center node, (iii) if $X_{i_3}$ is the center node. 
    \end{enumerate}
\end{enumerate}
Select the center node with the lowest value of the residual $Q^2(x)$ as the parent node. Note that in order to check if 3 nodes are in the same leaf cluster, we attempt to find the center node using the star/non-star subroutine. If we cannot eliminate the possibility of any node being a center node, all the nodes are in the same leaf cluster.

\begin{algorithm}[H]
\caption{Find the parent node in a leaf cluster}\label{alg:lc_resolve}
Input - Nodes of the leaf cluster, parents.\\
Output - A parent leaf pair from the leaf cluster.
\begin{algorithmic}[1]
\Procedure{LeafClusterResolution}{$P_{i',j'}$, $C$, $Parents$} 
\If{$|C\cap Parents|>0$}
\State $l \gets C\cap Parents[0]$
\Return $C\setminus \{l\}[0]$, $l$
\EndIf
\State $MinError \gets \infty$
\For{$(X_{i_1}, X_{i_2}) \in C$}
\For{$X_{i_3} \in N(X_{i_1})\cap N(X_{i_2})$}
\If{$X_{i_3}\in \textsc{FindCenter}(P_{i',j'}, (X_{i_1}, X_{i_2},X_{i_3}))$ and $d_{X_{i_3}',X_{i_1}'}, d_{X_{i_3}',X_{i_2}'}\leq d_{max} + 2\eta_{max}$}
\State $CandidateParent\gets (X_{i_1}, X_{i_2},X_{i_3})$
\Else \text{ }$CandidateParent\gets (X_{i_1}, X_{i_2})$
\EndIf
\For{$X_i \in  CandidateParent$}
\State $err \gets \textsc{QuadraticError}((X_{i_1}, X_{i_2},X_{i_3}), X_i)$
\If{$err < MinError$}
\State $MinError\gets err, l \gets X_i$
\EndIf
\EndFor
\EndFor
\EndFor
\State $r\gets C\setminus\{l\}[0]$\\
\Return $r,l$
\EndProcedure
\end{algorithmic}
\end{algorithm}

\subsubsection{Runtime Analysis}
Following are the runtime for constant $k$:
\begin{enumerate}
    \item \textsc{QuadraticError}: $\mathcal{O}(1)$.
    \item \textsc{FindCenter}: $\mathcal{O}(n)$ as in the worst case, the intersection of the neighborhood can contain $\mathcal{O}(n)$ nodes. The star/non-star test is $\mathcal{O}(1)$.
    \item \textsc{LeafClusterResolution}: The for loop on line 6 can execute $n$ times in the worst case calling \textsc{FindCenter} in each iteration. Thus the total time complexity is $\mathcal{O} (n^2)$.
    \item \textsc{FindLeafParent}: In the worst case \textsc{LeafClusterResolution} is called $\mathcal{O}(n)$ times thereby making the sample complexity $\mathcal{O}(n^3)$.
    \item \textsc{FindTree}: This calls \textsc{FindLeafParent} $\mathcal{O}(n)$ times. Thus the sample complexity of the algorithm is $\mathcal{O}(n^4)$.
\end{enumerate}
Note that when we know apriori that all the nodes within leaf clusters are unidentifiable, we only use the \textsc{LeafClusterResolution} subroutine to check if the parent node was already selected in the previous iteration (lines 1-5). We do not use the \textsc{QuadraticError} subroutine, thereby making it \textsc{LeafClusterResolution} an $\mathcal{O}(1)$ operation. In that case, \textsc{FindLeafParent} is now dominated by  \textsc{FindCenter} and becomes an $\mathcal{O}(n^2)$ making \textsc{FindTree} an $\mathcal{O}(n^3)$ operation (a gain of $\mathcal{O}(n)$ )
\subsubsection{Recovering $\mathcal{T}_{T^*}^{sub}$}
Once we recover a tree from $\mathcal{T}_{T^*}^{sub}$, we can obtain the complete set $\mathcal{T}_{T^*}^{sub}$ by considering all the parent leaf pairs within every cluster along with an arbitrary third node. We call the function $\textsc{QuadraticError}$ with this triplet and only $TestCenter$ node with $err<t_0/2$ is a candidate parent node. This operation does not increase the time complexity as it is an $\mathcal{O}(n^3)$ operation in the worst case.
\subsubsection{Modifications for the unidentifiable setting}
If we know apriori that the nodes within a leaf cluster are unidentifiable, we do not hope to achieve anything from the \textsc{QuadraticError} subroutine. Therefore, we do not execute any for loops in the \textsc{LeafClusterResolution} subroutine, thereby making it an $\mathcal{O}(1)$ operation. Therefore, the \textsc{GetLeafParent} subroutine becomes an $\mathcal{O}(n^2)$ operation making \textsc{FindTree} an $\mathcal{O}(n^3)$ operation.

\subsection{Proof of correctness}
\subsubsection{Proof of correctness of \textsc{FindLeafParent} subroutine}
We first prove that while no two nodes among $(l,r,z)$ are in the same leaf cluster, the subroutine \textsc{FindCenter} returns $C$ such that $|C| \leq 1$. 
For the next part, we assume that no two nodes among $(l,r,z)$ are in the same leaf cluster.\\
\textbf{Notation:} For any node, the adjacent node on its left is denoted with subscript $-$ and the adjacent node on the right is denoted by subscript $+$. $l^{t+1}, r^{t+1}$ and $z^{t+1}$ are the selection of nodes $l,r$ and $z$ in the next iteration respectively.

We have already proved the correctness of the star/non-star routine in the proof of Lemma \ref{le:lim_unid_gen}. Recall from the functionality of \textsc{FindCenter} that when we consider nodes $(l,r,z)$ with another node $j$, if $(l,r,z,j)$ forms a non-star, we eliminate the node that pairs with node $j$ from the candidate center nodes. 

With this in mind, we enumerate all the possible configurations of nodes $(l,r,z)$ such that no two of these nodes are in the same leaf cluster. For each case, we present two nodes which, when considered with $(l,r,z)$ would eliminate different nodes from $(l,r,z)$. This is equivalent to proving that $|C|\leq 1$.

\textbf{Claim:}  $d_{r, l}, d_{r, z} \leq d_{max} + \eta_{max}$\\
We first show that this holds true in the initialization of $l,r,z$. When $r$ is an internal node, we have that: 
$$
d_{r, l}\leq d_{r,l'}\leq d_{r, r_-'}\leq d_{max} + \eta_{max},d_{r, z}\leq d_{r, z'}\leq d_{r,r_+'}\leq d_{max} + \eta_{max}.
$$
When $r$ is a leaf node, since $l, z$ are not in the same leaf cluster as $r$, $l\neq z\neq r_-$. Therefore, we have that:
$$
d_{r,l}\leq d_{r, l'}\leq d_{r, r_-'}\leq d_{max} + \eta_{max}, d_{r, z}\leq d_{r, z'}\leq d_{r,r_-}\leq d_{max} + \eta_{max}.
$$
Now, we assume that $d_{r, l'}, d_{r, z'} \leq d_{max} + \eta_{max}$ is true at the beginning of any iteration and prove that it will continue to hold true at the end of every iteration. \\
\begin{figure}
    \centering
    \includegraphics[scale = 0.35]{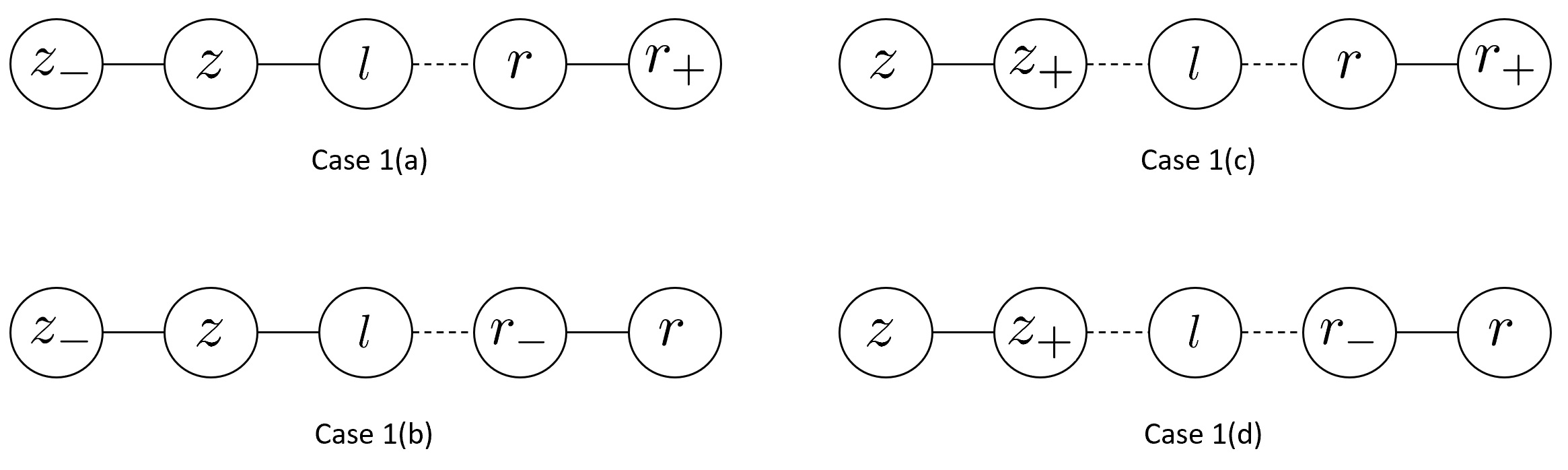}
    \caption{All the possible when node $z$ lies to the left of node $l$}
    \label{fig:case1}
\end{figure}
\textbf{Case 1:} We first enumerate all the cases when node $z$ lies to the left of node $l$. These are presented in Figure \ref{fig:case1}.

\textbf{Case 1(a):} Node $z$ lies to the left of node $l$ and is adjacent to it and $r_+$ exists.\\
In the case there exists a node $z_-$ to the left of $z$ such that there is an edge between $z$ and $z_-$. (If such a node did not exist, node $l$ and $z$ would have been in the same leaf cluster.) 
\begin{align*}
d_{r', z_-'} =& d_{r, r'}+d_{r, z} + d_{z, z_-'}\\
\leq& \eta_{max} + (d_{max} + \eta_{max}) +(d_{max} + \eta_{max})\\
=& 2d_{max} + 3\eta_{max}\\
d_{l', z_-'} =& d_{l, l'}+d_{l, z} + d_{z, z_-'}\\
\leq& \eta_{max} + (d_{max} + \eta_{max}) +(d_{max} + \eta_{max})\\
=& 2d_{max} + 3\eta_{max}\\
d_{z',r_+'} =& d_{z',z} + d_{z,r} + d_{r,r_+'}\\
\leq& 2d_{max} + 3\eta_{max}\\
d_{l',r_+'} =& d_{l,l'} + d_{l,r} + d_{r,r_+'}\\
\leq& 2d_{max} + 3\eta_{max}
\end{align*}
Thus $z_-,r_+\in N(r)\cap N(l)\cap N(z)$. $z_-$ eliminates $z$ and $r_+$ eliminates $r$.
In this case, nodes $l$ and $r$ do not change in this iteration. Therefore, $d_{l^{t+1},r^{t+1}} = d_{l,r} \leq d_{max} + \eta_{max}$. Also, $d_{z^{t+1},r^{t+1}} \leq d_{r_+',r} \leq d_{max} + \eta_{max}$.

\textbf{Case 1(b):} Node $z$ lies to the left of node $l$ and is adjacent to it and $r_+$ does not exists.\\
When $r_+$ does not exist, it is easy to see that $\exists r_-\neq l, z$.
The first 2 inequalities continue to hold true. We also have:
\begin{align*}
    d_{z',r_-'} =& d_{z',z} + d_{z,r_-} + d_{r_-,r_-'}\\
    \leq& d_{max} + 3\eta_{max}\\
    d_{l',r_-'} \leq& d_{max} + 3\eta_{max}
\end{align*}
Thus $r_-, z_-\in N(r)\cap N(z)\cap N(l)$. $r_-$ eliminates $r$ and $z_-$ eliminates $z$.
In this case, nodes $l$ and $r$ do not change in this iteration. Therefore, $d_{l^{t+1},r^{t+1}} = d_{l,r} \leq d_{max} + \eta_{max}$. Also, $d_{z^{t+1},r^{t+1}} \leq d_{r_-',r} \leq d_{max} + \eta_{max}$.

\textbf{Case 1(c):} Node $z$ lies to the left of node $l$ and there exists a node between $l$ and $z$. Also, $r_+$ exists.\\
We consider the nodes $z_+$ and $r_+$.
\begin{gather*}
    d_{r', z_+'} = d{r',r} + d_{r,z_+} + d_{z_+, z_+'}
    \leq  d_{max} + 3\eta_{max},\\
    d_{l', z_+'} = d{l',l} + d_{l,z_+} + d_{z_+, z_+'}
    \leq  d_{max} + 3\eta_{max}.
\end{gather*}
For $d_{z',r_+'}$ and $d_{l',r_+'}$, Case 1(a) calculations are valid.\\
Thus $r_+, z_+\in N(r)\cap N(z)\cap N(l)$. $r_+$ eliminates $r$ and $z_+$ eliminates $z$.
In this case, nodes $l$ and $r$ do not change in this iteration. Therefore, $d_{l^{t+1},r^{t+1}} = d_{l,r} \leq d_{max} + \eta_{max}$. Also, $d_{z^{t+1},r^{t+1}} \leq d_{r_+',r} \leq d_{max} + \eta_{max}$.

\textbf{Case 1(d):} Node $z$ lies to the left of node $l$ and there exists a node between $l$ and $z$. $r_+$ does not exist.\\
In this case, we have $z_+', r_-'\in N(r)\cap N(l)\cap N(z)$. The derivation comes from Case 1(b) and 1(c). $r_-$ eliminates $r$ and $z_+$ eliminates $z$.
In this case, nodes $l$ and $r$ do not change in this iteration. Therefore, $d_{l^{t+1},r^{t+1}} = d_{l,r} \leq d_{max} + \eta_{max}$. Also, $d_{z^{t+1},r^{t+1}} \leq d_{r_-',r} \leq d_{max} + \eta_{max}$.
 
\textbf{Case 2:} We next enumerate all the cases when node $z$ lies to the right of node $r$. These are presented in Figure \ref{fig:case2}.
\begin{figure}
    \centering
    \includegraphics[scale = 0.35]{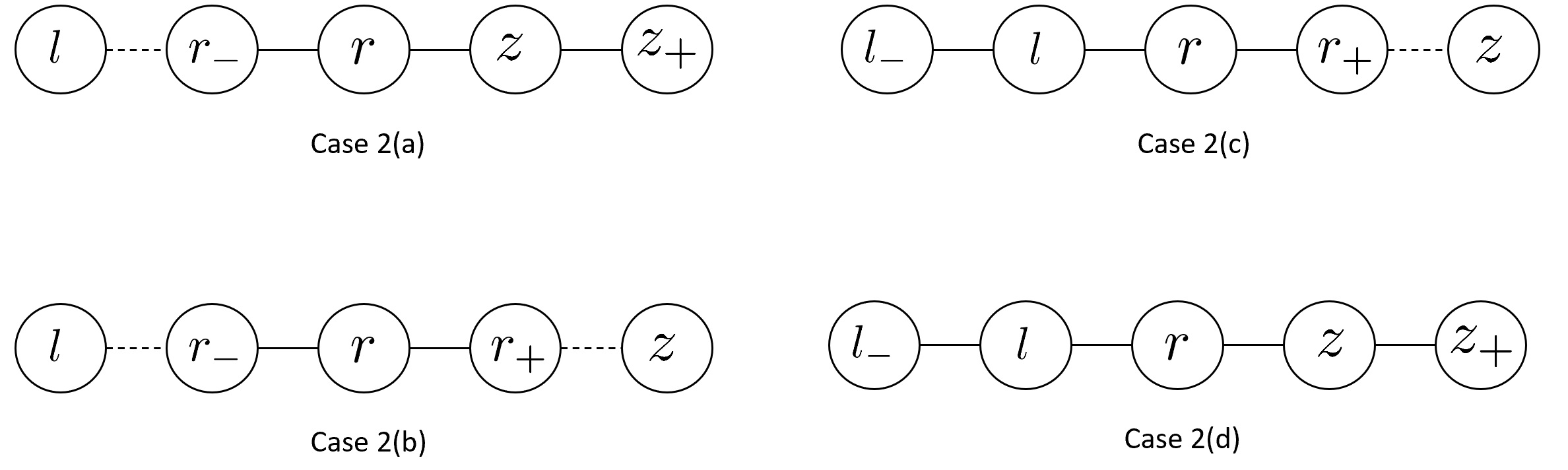}
    \caption{All the possible when node $z$ lies to the right of node $r$}
    \label{fig:case2}
\end{figure}

\textbf{Case 2(a):} $z$ lies to the right of $r$ and there exists at least one node between $l$ and $r$ and but no node between $r$ and $z$.
\begin{gather*}
    d_{l',z_+'} = d_{l',l} + d_{l,r} + d_{r,z_+'}
    \leq 3d_{max} + 3\eta_{max},\\
    d_{r',z_+'} \leq 2d_{max} + 2\eta_{max},\\
    d_{r_-',z'} \leq 2d_{max} + 2\eta_{max},\\
    d_{l', r_-'} = d_{l', l} + d_{l,r_-} + d{r_-, r_-'}
    \leq d_{max} + 3\eta_{max}.
\end{gather*}
Thus $r_-, z_+\in N(r)\cap N(z)\cap N(l)$. $r_-$ eliminates $l$ and $z_+$ eliminates $z$.
In this case, $l^{t+1} = r, r^{t+1} = z$ Therefore, $d_{l^{t+1},r^{t+1}} = d_{z,r} \leq d_{max} + \eta_{max}$. Also, $d_{z^{t+1},r^{t+1}} \leq d_{z_+',z} \leq d_{max} + \eta_{max}$.

\textbf{Case 2(b):} $z$ lies to the right of $r$ and there exists at least one node between $l$ and $r$ and also between $r$ and $z$.\\
Nodes of interest - $r_+, r_-$. $d_{l',r_-'}$ is the same as case 2(a).
\begin{gather*}
    d_{l',r_+'} = d_{l',r} + d_{r, r_+'}
    \leq 2(d_{max} + \eta_{max})
\end{gather*}
Similarly, $d_{z', r_-'}\leq 2(d_{max} + \eta_{max})$, $d_{z', r_+'}\leq d_{max} + 3\eta_{max}$.
Thus $r_-, r_+\in N(r)\cap N(z)\cap N(l)$. $r_-$ eliminates $l$ and $r_+$ eliminates $z$.
In this case, $l^{t+1} = r, r^{t+1} = z$ Therefore, $d_{l^{t+1},r^{t+1}} = d_{z,r} \leq d_{max} + \eta_{max}$. Also, $d_{z^{t+1},r^{t+1}} \leq d_{z_-',z} \leq d_{max} + \eta_{max}$.

\textbf{Case 2(c):} $z$ lies to the right of $r$ and there exists at least one node between $r$ and $z$ but no node between $r$ and $l$.\\
This is symmetric to Case 2(a). 
Thus $l_-, r_+\in N(r)\cap N(z)\cap N(l)$. $l_-$ eliminates $l$ and $r_+$ eliminates $z$.
In this case, $l^{t+1} = r, r^{t+1} = z$ Therefore, $d_{l^{t+1},r^{t+1}} = d_{z,r} \leq d_{max} + \eta_{max}$. Also, $d_{z^{t+1},r^{t+1}} \leq d_{z_-',z} \leq d_{max} + \eta_{max}$.

\textbf{Case 2(d):} $z$ lies to the right of $r$ and no nodes exist between $r$ and $z$ or $r$ and $l$.\\
Since all the nodes are within a radius of 3, it is easy to see that $l_-, r_+\in N(r)\cap N(z)\cap N(l)$. $l_-$ eliminates $l$ and $r_+$ eliminates $z$.
In this case, $l^{t+1} = r, r^{t+1} = z$ Therefore, $d_{l^{t+1},r^{t+1}} = d_{z,r} \leq d_{max} + \eta_{max}$. Also, $d_{z^{t+1},r^{t+1}} \leq d_{z_+',z} \leq d_{max} + \eta_{max}$.

\begin{figure}
    \centering
    \includegraphics[scale = 0.35]{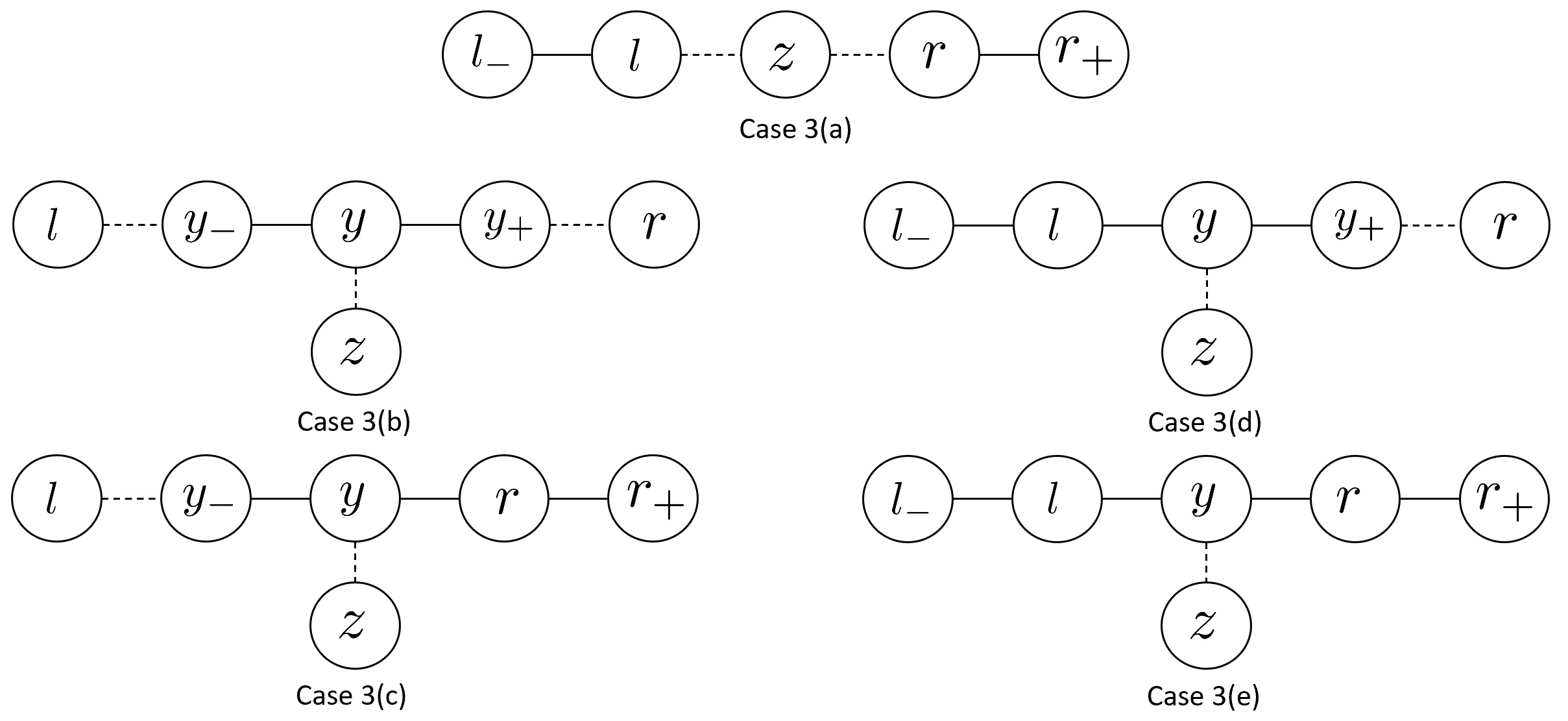}
    \caption{All the possible when node $z$ does not lie to the left of $l$ or right of $r$}
    \label{fig:case3}
\end{figure}

\textbf{Case 3(a):} $z$ lies between $l$ and $r$.
Consider $l_-$ and $r_+$.
\begin{gather*}
    d_{l_-', r'} = d_{l_-', l} + d_{l,r} + d_{r,r'}
    \leq 2d_{max} + 3\eta_{max}\\
    d_{l_-', z'} = d_{l_-', l} + d_{l,z} + d_{z,z'}
    \leq 2d_{max} + 3\eta_{max}\\
    d_{l', r_+'} = d_{l',l}+d_{l,r}+d_{r,r_+'}s
    \leq 2d_{max} + 3\eta_{max}\\
    d_{z', r_+'} = d_{z',z}+d_{z,r}+d_{r,r_+'}
    \leq 2d_{max} + 3\eta_{max}
\end{gather*}
Thus $l_-, r_+\in N(r)\cap N(z)\cap N(l)$. $l_-$ eliminates $l$ and $r_+$ eliminates $r$.
In this case, $l^{t+1} = z, r^{t+1} = r$ Therefore, $d_{l^{t+1},r^{t+1}} = d_{z,r} \leq d_{max} + \eta_{max}$. Also, $d_{z^{t+1},r^{t+1}} \leq d_{r_+',r} \leq d_{max} + \eta_{max}$.

If $l_-$ does not exist, we use $l_+$. Similarly, if $r_+$ does not exist, we use $r_-$. 

\textbf{Case 3(b):} Nodes $l, r, z$ form a Y-shape, that is, there exists a node $y$ such that $l\perp r\perp z|y$. There exists at least one node between $l$ and $y$ as well as between $y$ and $r$.\\
Consider nodes $y_-$, $y_+$.
\begin{gather*}
    d_{y_-', z'} = d_{z,z'} + d_{y,z} + d_{y, y_-'}
    \leq  2d_{max} + 3\eta_{max}\\
    d_{y_-', l'} = d_{l',l} + d_{l,y_-} + d_{y_-, y_-'}
    \leq  d_{max} + 3\eta_{max}\\
    d_{y_-', r'} = d_{r',r} + d_{r, y_-} + d_{y_-, y_-'}
    \leq  d_{max} + 3\eta_{max}\\
    d_{y_+', z'} = d_{z,z'} + d_{y,z} + d_{y, y_+'}
    \leq  2d_{max} + 3\eta_{max}\\
    d_{y_+', l'} = d_{l',l} + d_{l,y_+} + d_{y_+, y_+'}
    \leq  d_{max} + 3\eta_{max}\\
    d_{y_+', r'} = d_{r',r} + d_{r, y_+} + d_{y_+, y_+'}
    \leq  d_{max} + 3\eta_{max}
\end{gather*}

Thus $y_-, y_+\in N(r)\cap N(z)\cap N(l)$. $y_-$ eliminates $l$ and $y_+$ eliminates $r$. If $z$ is also eliminated, $l^{t+1} = l, r^{t+1} = r$ and $d_{z^{t+1}, r^{t+1}}\leq d_{r_-',r}\leq d_{max} + \eta_{max}$. If $z$ is not eliminated, $l^{t+1} = z, r^{t+1} = r$, $d_{l^{t+1},r^{t+1}} = d_{r,z}\leq d_{max} + \eta_{max}$  $d_{z^{t+1}, r^{t+1}}\leq d_{r_-',r}\leq d_{max} + \eta_{max}$.

\textbf{Case 3(c):} Nodes $l, r, z$ form a Y-shape, that is, there exists a node $y$ such that $l\perp r\perp z|y$. There exists at least one node between $l$ and $y$ but no node between $y$ and $r$.\\
Consider nodes $y_-$, $r_+$.
Analysis for $y_-$ is the same as in case 3(b).
\begin{gather*}
    d_{r_+', z'} = d_{z,z'} + d_{r,z} + d_{r, r_+'}
    \leq  2d_{max} + 3\eta_{max}\\
    d_{r_+', l'} = d_{l',l} + d_{l,r} + d_{r, r_+'}
    \leq  2d_{max} + 3\eta_{max}
\end{gather*}

Thus $y_-, r_+\in N(r)\cap N(z)\cap N(l)$. $y_-$ eliminates $l$ and $r_+$ eliminates $r$. If $z$ is also eliminated, $l^{t+1} = l, r^{t+1} = r$ and $d_{z^{t+1}, r^{t+1}}\leq d_{r_+',r}\leq d_{max} + \eta_{max}$. If $z$ is not eliminated, $l^{t+1} = z, r^{t+1} = r$, $d_{l^{t+1},r^{t+1}} = d_{r,z}\leq d_{max} + \eta_{max}$  $d_{z^{t+1}, r^{t+1}}\leq d_{r_+',r}\leq d_{max} + \eta_{max}$.

\textbf{Case 3(d):} Nodes $l, r, z$ form a Y-shape, that is, there exists a node $y$ such that $l\perp r\perp z|y$. There exists at least one node between $r$ and $y$ but no node between $y$ and $l$.\\
Consider nodes $l_-$, $y_+$. Analysis for $y_+$ is the same as Case 3(b).
\begin{gather*}
    d_{l_-', z'} = d_{z,z'} + d_{z,y} + d_{y, l_-'}
    \leq  d_{r,z'} + d_{y,l_-'}
    \leq  3d_{max} + 3\eta_{max}\\
    d_{l_-', r'} = d_{r',r} + d_{r, l} + d_{l, l_-'}
    \leq  2d_{max} + 3\eta_{max}
\end{gather*}
Thus $y_+, l_-\in N(r)\cap N(z)\cap N(l)$. $y_+$ eliminates $r$ and $l_-$ eliminates $l$. If $z$ is also eliminated, $l^{t+1} = l, r^{t+1} = r$ and $d_{z^{t+1}, r^{t+1}}\leq d_{r_-',r}\leq d_{max} + \eta_{max}$. If $z$ is not eliminated, $l^{t+1} = z, r^{t+1} = r$, $d_{l^{t+1},r^{t+1}} = d_{r,z}\leq d_{max} + \eta_{max}$  $d_{z^{t+1}, r^{t+1}}\leq d_{r_-',r}\leq d_{max} + \eta_{max}$.

\textbf{Case 3(e):} Nodes $l, r, z$ form a Y-shape, that is, there exists a node $y$ such that $l\perp r\perp z|y$. There exists no nodes between $r$ and $y$ and between $y$ and $l$.\\
Consider nodes $l_-$, $r_+$. Analysis follows from Cases 3(c) and 3(d).
Thus $r_+, l_-\in N(r)\cap N(z)\cap N(l)$. $r_+$ eliminates $r$ and $l_-$ eliminates $l$. If $z$ is also eliminated, $l^{t+1} = l, r^{t+1} = r$ and $d_{z^{t+1}, r^{t+1}}\leq d_{r_+',r}\leq d_{max} + \eta_{max}$. If $z$ is not eliminated, $l^{t+1} = z, r^{t+1} = r$, $d_{l^{t+1},r^{t+1}} = d_{r,z}\leq d_{max} + \eta_{max}$  $d_{z^{t+1}, r^{t+1}}\leq d_{r_+',r}\leq d_{max} + \eta_{max}$.

Thus at each iteration, we visit one node and remove it from the set of nodes that get visited in subsequent iterations until we get $(l,r,z)$ such that at least 2 of the nodes are in the same leaf cluster. Note that the maximum distance in the above analysis is $3d_{max} + 3\eta_{max}$. However our threshold for the neighborhood set is $4d_{max} + 3\eta_{max}$. The extra $d_{max}$ is there to account for the fact that in the unidentifiable case, a parent node from a leaf cluster may have been confused with a leaf node. In that case, the leaf node is retained in the active set while the parent node is removed from the active set for the subsequent iterations. In order to account for that, we add a factor of $d_{max}$ to the neighborhood threshold.

\paragraph{Proof of correctness of \textsc{LeafClusterResolution}}
From the above analysis, we know that \textsc{LeafClusterResolution} is called with nodes in $C$ belonging in the same leaf cluster. 
The idea is to check if any on the nodes in $C$ are such that when they act as the center node, Equation \eqref{eq:err_est_x} has a solution. In order to do this, we consider 2 nodes in $C$ at a time and scan through all the nodes in their common neighborhood as the third node. We check if the third node is also in the same leaf cluster in which case we also see if the error for this node as the parent node is small. If it is not in the same leaf cluster, we just use it as the third node needed for Equation \eqref{eq:err_est_x}.
We first show that the routine to check if $X_{i_3}$ is in the same leaf cluster as $(X_{i_1}, X_{i_2})$ is correct:\\
If $X_{i_3}$ is in the same leaf cluster as $(X_{i_1}, X_{i_2})$, it is easy to see that any star/non-star test on $(X_{i_1}, X_{i_2}, X_{i_3}, j)$ always returns a non-star. When $X_{i_3}$ is not in the same leaf cluster as $(X_{i_1}, X_{i_2})$, then there exists a node $X_{i_3^+}$ adjacent to $X_{i_3}$ either away from the path connecting $X_{i_3}$ to $(X_{i_1}, X_{i_2})$ or on that path such that $(X_{i_1}, X_{i_2}, X_{i_3}, X_{i_3^+})$ forms a non-star where $(X_{i_3}, X_{i_3^+})$ forms a pair. It is easy to see that $d_{X_1', (X_{i_3^+})'}, d_{X_2', (X_{i_3^+})'} \leq 2d_{max} + 3\eta_{max}$. Therefore, $X_{i_3^+} \in N(X_{i_1})\cap N(X_{i_2}) \cap N(X_{i_3})$. Thus it is ruled out from being a parent candidate.\\
Now it is easy to see that if any leaf node is identifiable, it will have a non-zero error for Equation \eqref{eq:err_est_x}. For an unidentifiable leaf node, both the leaf and parent have a solution to Equation \eqref{eq:err_est_x} and one of them is randomly selected as the parent node. 

Any subsequent calls with nodes from the same leaf cluster always select the correct parent in line (2).

From the correctness of \textsc{LeafClusterResolution}, we conclude that \textsc{FindLeafParent} subroutine is correct. 
Once we have the correctness of \textsc{GetLeafParent}, the correctness of \textsc{FindTree} is easy to understand. We prove this by induction on the number of nodes. 

\textit{Base Case (n=2):}
Line 9 recovers the lone edge.

\textit{Inductive Case:}
Let us assume that the algorithm works for all $n<k$.
For $n = k+1$, by the correctness of \textsc{GetLeafParent}, the algorithm correctly recovers one leaf parent pair and adds that edge to the edge set. Once the leaf node is removed, the algorithm is effectively running on $k$ nodes and by the inductive assumption that is correct.

This completes the proof of correctness of the algorithm.

\subsection{Modification for finite sample domain}
In this section we present the necessary modifications needed to execute the algorithm using finite samples.\\
\textbf{Classifying 4 nodes as star/non-star using finite samples:} Let us denote $\kappa_{i',j'} = \exp{-d_{i',j'}}$, $\kappa_{max} = \exp(-d_{min})$. We denote the finite sample estimate of $\kappa_{i',j'}$ by $\hat{\kappa_{i',j'}}$

In the infinite sample setting, a set of 4 nodes $(X_{1}, X_{2}, X_{3}, X_{4})$ forms a non-star with $(X_{1}, X_{2})$ forming a pair if:
\begin{equation*}
\begin{aligned}
    \frac{\sqrt{\kappa_{1',3'}\kappa_{2',4'}\kappa_{1',4'}\kappa_{2',3'}}}{\kappa_{1',2'}\kappa_{3',4'}}\leq \kappa_{max}^2\\
    \frac{\sqrt{\kappa_{1',2'}\kappa_{3',4'}\kappa_{1',4'}\kappa_{2',3'}}}{\kappa_{1',3'}\kappa_{2',4'}}\geq 1/\kappa_{max}^2\\
    \frac{\sqrt{\kappa_{1',3'}\kappa_{4',2'}\kappa_{1',2'}\kappa_{4',3'}}}{\kappa_{1',4'}\kappa_{2',3'}}\geq 1/\kappa_{max}^2
\end{aligned}
\end{equation*}
The finite sample test is as follows:
\begin{equation*}
\begin{aligned}
    \frac{\sqrt{\hat{\kappa}_{1',3'}\hat{\kappa}_{2',4'}\hat{\kappa}_{1',4'}\hat{\kappa}_{2',3'}}}{\hat{\kappa}_{1',2'}\hat{\kappa}_{3',4'}}\leq (1+\kappa_{max}^2)/2\\
    \frac{\sqrt{\hat{\kappa}_{1',2'}\hat{\kappa}_{3',4'}\hat{\kappa}_{1',4'}\hat{\kappa}_{2',3'}}}{\hat{\kappa}_{1',3'}\hat{\kappa}_{2',4'}}\geq 1\\
    \frac{\sqrt{\hat{\kappa}_{1',3'}\hat{\kappa}_{4',2'}\hat{\kappa}_{1',2'}\hat{\kappa}_{4',3'}}}{\hat{\kappa}_{1',4'}\hat{\kappa}_{2',3'}}\geq 1
\end{aligned}
\end{equation*}

A set of 4 nodes $(X_{1}, X_{2}, X_{3}, X_{4})$ is classified as a star if:
\begin{equation*}
\begin{aligned}
    \frac{\sqrt{\hat{\kappa}_{1',3'}\hat{\kappa}_{2',4'}\hat{\kappa}_{1',4'}\hat{\kappa}_{2',3'}}}{\hat{\kappa}_{1',2'}\hat{\kappa}_{3',4'}}\geq (1+\kappa_{max}^2)/2\\
    \frac{\sqrt{\hat{\kappa}_{1',2'}\hat{\kappa}_{3',4'}\hat{\kappa}_{1',4'}\hat{\kappa}_{2',3'}}}{\hat{\kappa}_{1',3'}\hat{\kappa}_{2',4'}}\geq (1+\kappa_{max}^2)/2\\
    \frac{\sqrt{\hat{\kappa}_{1',3'}\hat{\kappa}_{4',2'}\hat{\kappa}_{1',2'}\hat{\kappa}_{4',3'}}}{\hat{\kappa}_{1',4'}\hat{\kappa}_{2',3'}}\geq (1+\kappa_{max}^2)/2
\end{aligned}
\end{equation*}

If neither of the above conditions is satisfied for any pair, the test fails and this set of 4 nodes is not classified as star/non-star.

{\bf Neighborhood Thresholding:}
In the finite sample setting, we allow for a slack in the threshold to ensure that, with high probability, the empirical neighborhood vector contains all the nodes from the underlying neighborhood vector. The empirical neighborhood vector is defined as follows:
\begin{equation*}
    N'(X_i) = sort({X_j: \hat{d}_{i',j'} \leq t_{emp}}, \text{ key } = \hat{d}_{i',j'}),
\end{equation*}
where the threshold is $t_{emp} = 0.5(4d_{max} + 3\eta_{max})$.

\section{Sample Complexity Upper Bound}\label{ap:sc_upper}

Let us define 2 events:
\begin{equation*}
    \mathcal{B}_1 = \{(E_{a'})_{i,i} < 0.1p_{min}, \forall a,i\}, 
    \mathcal{B}_2 = \{\|E_{a',b'}\| < \epsilon \forall a,b\}
\end{equation*}
For any $X_a, X_b$ we only consider nodes such that:
\begin{align*}
    &\sqrt{|det(\hat{P}_{a'|b'}\hat{P}_{b'|a'})|} > 0.5\exp(-4d_{\max})(1-q_{max})^{3(k-1)}(kp_{min})^{1.5k}\\
    &\implies\frac{|det(\hat{P}_{a',b'})|}{\sqrt{|det(\hat{P}_{a'}\hat{P}_{b'})|}} > 0.5\exp(-4d_{\max})(1-q_{max})^{3(k-1)}(kp_{min})^{1.5k}.
\end{align*}
In the event $\mathcal{B}_1$, $det(\hat{P}_{a'}), det(\hat{P}_{b'}) > (0.9p_{min})^k$, therefore we have:
\begin{align*}
    &|det(\hat{P}_{a',b'})|\geq 0.5\exp(-4d_{\max})(1-q_{max})^{3(k-1)}p_{min}^{1.5k}(0.9p_{min})^k
\end{align*}
Next we bound the minimum absolute  eigenvalue of $\hat{P}_{a',b'}$.
\begin{lemma}\label{le:min_eig}
For any $k\times k$ matrix $M$ such that $M_{i,j}\geq 0$, $\sum_{i,j}M_{i,j} = 1$ and $|det(M)|\geq c$ where $0<c\leq \left(\frac{1}{k}\right)^{k}$, then the minimum absolute eigenvalue of $M$ satisfies $c(k-1)^{k-1}\leq |\lambda_{min}(M)|\leq ck^{k-1}$.
\end{lemma}
\begin{proof}
Let $\lambda_1, \lambda_2\dots \lambda_k$ be the eigenvalues of $M$ such that $|\lambda_1|\geq |\lambda_2|\geq \dots \geq |\lambda_k|$. Standard results tell us that:
$$
\sum_{i} |\lambda_i|\leq  \sum_{i,j}M_{i,j} = 1, |det(M)| = \prod_{i}|\lambda_i|\geq c
$$
We are interested in the solution to the following optimization problem:

\begin{align}\label{eq:optimization}
    &\min &&|\lambda_k|\\
    &\text{s.t.}&& \sum_{i=1}^k |\lambda_i|\leq 1\\ 
    & && \prod_{i=1}^k|\lambda_i|\geq c,\\
    & && |\lambda_1|\geq |\lambda_2|\dots |\lambda_k|,
\end{align}
where $0<c\leq(1/k)^k$. Denote the optimal solution to the above problem by $\lambda_1^*, \lambda_2^*, \dots \lambda_k^*$. \\
\textbf{Claim:} $\sum_{i} |\lambda_i^*| = 1, \prod_{i=1}^k|\lambda_i^*|= c, |\lambda_1^*|= |\lambda_2^*|=\dots =|\lambda_{k-1}^*|$.\\
In order to prove this, we prove that if these do not hold true, there exists a smaller $|\lambda_k|$.\\
By contradiction, let us assume that $\sum_{i} |\lambda_i^*| = 1-\epsilon$ for some $0<\epsilon<1$. Then it is easy to see that $\exists \tilde{\lambda_i}, \epsilon'>0$ such that $|\tilde{\lambda_i}| = |\lambda_i^*| + \frac{\epsilon}{k-1}$ $\forall i \in \{1, 2 \dots ,k-1\}$ and $|\tilde{\lambda_k}| = |\lambda_i^*| - \epsilon'$ such that $\prod_{i=1}^k|\tilde{\lambda_i}| = c$. Therefore, $|\lambda_i^*|$ is not optimal. Thus, $\sum_{i} |\lambda_i^*| = 1$.\\
By contradiction, let us assume that $\prod_{i} |\lambda_i^*| = (1+\epsilon)c$ for some $0<\epsilon$. Consider $\tilde{\lambda_i}$ such that $\tilde{\lambda_i} = \lambda_i^*$ $\forall i \in \{1, 2 \dots ,k-1\}$ and $\tilde{\lambda_k} = \lambda_k^*/(1+\epsilon)$. Then $\tilde{\lambda_i}$ is feasible and has smaller objective value, thus $\prod_{i} |\lambda_i^*| = c$. \\
We prove the last part by contradiction too. Let us assume by contradiction that at least one of $|\lambda_i^*|$ is not equal for $i\in \{1, 2, \dots k-1\}$. Consider $\tilde{\lambda_i}$ such that $|\tilde{\lambda_i}| = \frac{\sum_{j = 1}^{k-1}|\lambda_j^*|}{k-1}$. Then, by the AM-GM inequality, we have that:
$$
\prod_{i=1}^{k-1}|\tilde{\lambda_i}| =  \left(\frac{\sum_{j = 1}^{k-1}|\lambda_j^*|}{k-1}\right)^{k-1} = (1+\epsilon) \prod_{i=1}^{k-1}|\lambda_i^*|
$$
for some $\epsilon>0$. Choosing $|\tilde{\lambda_k}| = |\lambda_k^*|/(1+\epsilon)$, we get a feasible $\tilde{\lambda_i}$ with a smaller objective function. This concludes the proof of the claim. \\
Thus, the solution to the optimization problem \ref{eq:optimization} satisfies:
$$
|\lambda_1^*|= |\lambda_2^*|=\dots =|\lambda_{k-1}^*| = \frac{1-\lambda_k^*}{k-1}, \left(\frac{1-\lambda_k^*}{k-1}\right)^{k-1}\lambda_k^* = c.
$$
Therefore, Equation \ref{eq:optimization} has the same solution as the following optimization problem:
\begin{align*}
    &\min&&|\lambda_k|\\
    &\text{s.t.}&& 0<|\lambda_k|\leq \frac{1}{k}\\ 
    & &&|\lambda_k|\left(\frac{1-|\lambda_k|}{k-1}\right)^{k-1}= c,
\end{align*}
where $0<c\leq(1/k)^k$. The solution to the above optimization problem satisfies $|\lambda_k^*|\left(\frac{1-|\lambda_k^*|}{k-1}\right)^{k-1} = c$. The solution exists because $|\lambda_k|\left(\frac{1-|\lambda_k|}{k-1}\right)^{k-1}$ is monotonically increasing in $|\lambda_k|$ and:
$$
|\lambda_k|\left(\frac{1-|\lambda_k|}{k-1}\right)^{k-1} = 0, \text{ when }|\lambda_k| = 0,
$$
$$
|\lambda_k|\left(\frac{1-|\lambda_k|}{k-1}\right)^{k-1} = \left(\frac{1}{k}\right)^{k}, \text{ when }|\lambda_k| = 1/k.
$$
Therefore, $|\lambda^*_k|$ satisfies:
$$
|\lambda^*_k| = c\left(\frac{(k-1)}{1-|\lambda_k^*|}\right)^{k-1}
$$
Since $0<|\lambda^*_k|\leq 1/k$, we have that $c(k-1)^{k-1}\leq|\lambda^*_k|\leq ck^{k-1}$
\end{proof}
Using Lemma \ref{le:min_eig}, the minimum absolute eigenvalue of $\hat{P}_{a',b'}$ is lower bounded by $|det(\hat{P}_{a',b'})|(k-1)^{k-1}$. Therefore, we have that:
\begin{equation}\label{eq:inv_bound}
    \|\hat{P}^{-1}_{a',b'}\|\leq \frac{1}{0.5\exp(-4d_{\max})(1-q_{max})^{3(k-1)}(kp_{min})^{1.5k}(0.9p_{min})^k(k-1)^{k-1}}\triangleq \frac{1}{z_1}
\end{equation}


\subsection{Sample Complexity for Existence of a solution to Equation \ref{eq:err_est_x}}
We are interested in the error in the estimate of $Q(x)$ as defined below:
\begin{align*}
    \hat{Q}(x)&= \|\frac{x^2}{k^2}(O - kI) - \frac{x}{k}(O\hat{P}'_{b} + \hat{P}'_{b}O - k\hat{P}'_{b} - I) + \hat{P}_{b',c'} \hat{P}_{a',c'}^{-1}\hat{P}_{a',b'}-\hat{P}'_{b}\|_F\\
     {Q}(x)&= \|\frac{x^2}{k^2}(O - kI) - \frac{x}{k}(O {P}'_{b} +  {P}'_{b}O - k {P}'_{b} - I) +  {P}_{b',c'}  {P}_{a',c'}^{-1} {P}_{a',b'}- {P}'_{b}\|_F
\end{align*}
We derive the error bound for the term $P_{b',c'} P_{a',c'}^{-1}P_{ab}$ when estimated using the respective empirical estimates.
\begin{align*}
    P_{b',c'} P_{a',c'}^{-1}P_{a',b'}
    &= (\hat{P}_{b',c'} + E_{b',c'}) (\hat{P}_{a',c'} + E_{a',c'})^{-1} (\hat{P}_{a',b'} + E_{a',b'})\\
    &= (\hat{P}_{b',c'} + E_{b',c'}) \left(\hat{P}_{a',c'}^{-1} + \sum_{m=1}^{\infty}(-\hat{P}_{a',c'}^{-1}E_{a',c'})^{m}\hat{P}_{a',c'}^{-1}\right)(\hat{P}_{a',b'} + E_{a',b'})\\
    &= \hat{P}_{b',c'} \hat{P}_{a',c'}^{-1}\hat{P}_{ab} 
    + E_{b',c'} \hat{P}_{a',c'}^{-1} \hat{P}_{a',b'} +  \hat{P}_{b',c'}\hat{P}_{a',c'}^{-1} E_{a',b'} + \hat{P}_{b',c'} \tilde{E}_{ac} \hat{P}_{a',b'} \\
    & + E_{b',c'}\hat{P}_{a',c'}^{-1}E_{a',b'} + E_{b',c'}\tilde{E}_{ac}\hat{P}_{a',b'} + \hat{P}_{b',c'}\tilde{E}_{ac}E_{a',b'} + E_{b',c'}\tilde{E}_{ac}E_{a',b'},
\end{align*}
here we use the notation $\tilde{E}_{ac} := \sum_{m=1}^{\infty}(-\hat{P}_{a',c'}^{-1}E_{a',c'})^{m}\hat{P}_{a',c'}^{-1}$. Using the triangle inequality and submultiplicative property of the spectral norm, we get that:
\begin{equation*}
\|\tilde{E}_{ac}\|_2 \leq \frac{\|\hat{P}_{a',c'}^{-1}\|_2^2\|E_{a',c'}\|_2}{1 - \|\hat{P}_{a',c'}^{-1}\|_2\|E_{a',c'}\|_2}
\end{equation*}
We choose such an $\epsilon$ in the event $\mathcal{B}_2$ that ensures that $\|P_{a',c'}^{-1}\|_2\|E_{a',c'}\|_2 < 0.5$. This gives us:
\begin{equation*}
\|\tilde{E}_{ac}\|_2 \leq 2\|\hat{P}_{a',c'}^{-1}\|_2^2\|E_{a',c'}\|_2
\end{equation*}
In the event $\mathcal{B}_2$, $\|E_{a',b'}\|_2, \|E_{b',c'}\|_2 \|E_{a',c'}\|_2< \epsilon$. In the event $\mathcal{B}_1$, from Equation \eqref{eq:inv_bound}, $\|\hat{P}_{a',c'}^{-1}\|_2\leq z_1^{-1}$. Therefore, $\|\tilde{E}_{ac}\|_2\leq 2z_1^{-2}\epsilon$.  Since $\hat{P}_{a',b'}, \hat{P}_{b',c'}$ are joint PMF matrices, we have that $\|\hat{P}_{a',b'}\|_2, \|\hat{P}_{b',c'}\|_2 < 1$. Substituting these along with triangle inequality and submultiplicative property of the spectral norm gives us the following:
\begin{align*}
    \|\hat{P}_{b',c'} \hat{P}_{a',c'}^{-1}\hat{P}_{a',b'} - P_{b',c'} P_{a',c'}^{-1}P_{a'b'}\|_2 \leq 3\epsilon z_1^{-1} + 8 \epsilon z_1^{-2}
\end{align*}

This gives us:
\begin{align*}
    \hat{Q}(x)&= \|\frac{x^2}{k^2}(O - kI) - \frac{x}{k}(O\hat{P}'_{b} + \hat{P}'_{b}O - k\hat{P}'_{b} - I) + \hat{P}_{b',c'} \hat{P}_{a',c'}^{-1}\hat{P}_{a',b'}-\hat{P}'_{b}\|_F\\
&\leq  Q(x) + (3x+1)\|E_{b'}\|_F + \|\hat{P}_{b',c'} \hat{P}_{a',c'}^{-1}\hat{P}_{a',b'} - P_{b',c'} P_{a',c'}^{-1}P_{ab}\|_F \\
&\implies |\hat{Q}(x) - Q(x)| \leq 4\sqrt{k}\epsilon+ 3\sqrt{k}\epsilon z_1^{-1} + 8 \sqrt{k}\epsilon z_1^{-2}\leq 15 \sqrt{k}\epsilon z_1^{-2}
\end{align*}
We need that $|\hat{Q}(x) - Q(x)|<t_0/2$. This is satisfied when:
\begin{equation}\label{eq:eps_1}
 \epsilon<\frac{t_0z_1^2}{30\sqrt{k}}   
\end{equation}

\subsection{Sample Complexity for Star/Non-Star test}
Consider a set of 4 nodes $\{X_1, X_2, X_3, X_4\}$ such that they form a non-star such that $\{X_1, X_2\}$ form a pair. 

\begin{equation}
    \begin{aligned}
    \frac{|det(P_{1,3}P_{2,4})|}{|det(P_{1,4}P_{2,3})|} = \frac{|det((\hat{P}_{1,3}+E_{1,3})(\hat{P}_{2,4}+E_{2,4}))|}{|det((\hat{P}_{1,4}+E_{1,4})(\hat{P}_{2,3}+E_{2,3}))|}
    \end{aligned}
\end{equation}
Using the analysis from \cite{tandon2021sga}, a set of 4 nodes is correctly classified if for any pair of nodes $\{a,b\}$ that are in each other's neighborhood sets, we have that $|det({P}_{a,b}) - det(\hat{P}_{a,b})|<\frac{z_1(1-\alpha)}{20}$, where $\alpha = \frac{1+\exp(-2d_{min})}{2}$. 
We can bound the difference in the empirical estimate of the determinant and the true determinant using the matrix perturbation result in Chapter 5 of \cite{bhatia2007perturbation} as follows:
$$|det({P}_{a,b}) - det(\hat{P}_{a,b})|\leq k\max\{\|{P}_{a,b}\|, \|\hat{P}_{a,b}\|\}^{k-1}\|E_{a,b}\|_2\leq  k\|E_{a,b}\|_2$$
Under event $\mathcal{B}_2$ we have that $\|E_{a,b}\|<\epsilon$. Thus the algorithm correctly classifies nodes as star/non-star when:
\begin{equation}\label{eq:eps_2}
    \epsilon<\frac{z_1(1-\alpha)}{20k}. 
\end{equation}
From Equations \eqref{eq:eps_1}, \eqref{eq:eps_2} we choose $\epsilon$ as follows:
\begin{equation}\label{eq:eps_final}
    \epsilon<\min\left\{\frac{z_1(1-\alpha)}{20k}, \frac{t_0z_1^2}{30\sqrt{k}}\right\}. 
\end{equation}
Next, we find the number of samples needed for $\mathcal{B}_1$ and $\mathcal{B}_2$ to hold true with high probability.
\begin{align*}
P(\mathcal{B}_1,\mathcal{B}_2) \geq 1 - P(\bar{\mathcal{B}}_1) - P(\bar{\mathcal{B}}_2)
\end{align*}
For a given $a, i$, by Hoeffding's inequality we have that:
\begin{align*}
    P((E_{a'})_{i,i})>0.1p_{min})\leq \exp(-2N(0.1p_{min})^2).
\end{align*}
By the union bound on all the nodes and all the alphabets we get: 
\begin{align*}
    P(\bar{\mathcal{B}}_1)\leq kn\exp(-2N(0.1p_{min})^2).
\end{align*}
In order to achieve $P(\bar{\mathcal{B}}_1)\leq \delta/2$, we have the following bound on the sample complexity:
\begin{equation}\label{eq:sample_1}
  N\geq\frac{50}{p^2_{min}}\log\left(\frac{2nk}{\delta}\right).  
\end{equation}
 
Next, we upper bound the probability $P(\bar{\mathcal{B}}_1)$. \\
The matrix Bernstein's inequality (\cite{tropp2015introduction}) states that for independent random matrices $S_1\dots S_N$ with dimension $d_1\times d_2$ such that $\E{S_i} = 0$, $\|S_i\|<L$ $\forall i$ and $Z = \sum_{i = 1}^NS_i$, then
$$
P(\|Z\|>t) \leq (d_1+d_2)\exp{\left(\frac{-t^2/2}{v(Z) + Lt/3}\right)}
$$
where $v(Z) = \max\{\|\sum_{i = 1}^N\E{S_iS_i^T}\|\}$.
In order to apply this in our setting, define $S_i = \mathbb{1}_{a',b'}^i - P_{a',b'}$ where $\mathbb{1}_{a',b'}^i$ is the indicator matrix for sample $i$ with a $1$ in the position corresponding to the value of $X_a'$ and $X_b'$ in that sample. \\
It is easy to see that $\E{S_i} = 0, \|S_i\|\leq 2$. Also, in this setting, $E_{a',b'} = \frac{1}{N}Z$. Next, we bound $v(Z)$.
\begin{align*}
    \E{S_iS_i^T} &= \E{(\mathbb{1}_{a',b'}^i - P_{a',b'})(\mathbb{1}_{a',b'}^i - P_{a',b'})^T}\\
    &= \E{(\mathbb{1}_{a',b'}^i)(\mathbb{1}_{a',b'}^i)^T} - \E{P_{a',b'}P_{a',b'}^T}\\
    \implies \|\sum_{i=1}^N \E{S_iS_i^T}\|&\leq 2N
\end{align*}
This bounds the probability of $\|E_{a',b'}\|>\epsilon$ as follows:
$$
P(\|E_{a',b'}\|>\epsilon) = P(\|Z\|>n\epsilon)\leq 2k\exp{\left(\frac{-N\epsilon^2}{4(1+\epsilon/3)}\right)}
$$

By the union bound on all the pair of nodes, we have:
\begin{align*}
    P(\bar{\mathcal{B}}_2)\leq kn(n-1)\exp\left(\frac{-N\epsilon^2}{4(1+\epsilon/3)}\right).
\end{align*}
For $P(\bar{\mathcal{B}}_2)\leq \delta/2$, the lower bound on the number of samples is given by
\begin{equation}\label{eq:sample_2}
    N\geq \frac{2(2+\epsilon/3)}{\epsilon^2}\log\left(\frac{2nk(n-1)}{\delta}\right)
\end{equation}
From Equations \eqref{eq:sample_1} and \eqref{eq:sample_2}, the algorithm outputs the correct tree if:
\begin{equation}
    N\geq \max\left\{\frac{50}{p_{min}^2}\log\left(\frac{2nk}{\delta}\right), \frac{2(2+\epsilon/3)}{\epsilon^2}\log\left(\frac{2nk(n-1)}{\delta}\right)\right\}
\end{equation}
From the value of $\epsilon$ as defines in Equation \eqref{eq:eps_final}, we  can see that the sample complexity is dominated by the second term. Substituting the value of $\epsilon$ from Equation \eqref{eq:eps_final}, we get that the sample complexity is of the following order:
\begin{align*}
    N = \mathcal{O}\Bigg(\max\Bigg\{&\tfrac{k^2\exp(8d_{\max})}{(1-q_{max})^{6(k-1)}(0.9p_{min}^{2.5})^{2k}(1-\exp{(-2d_{min})})^2(k-1)^{2(k-1)}}\Bigg. \Bigg.,\\
    &\Bigg.\Bigg.\tfrac{k \exp(16d_{\max})}{t_0^2 (1-q_{max})^{12(k-1)}(0.9p_{min}^{2.5})^{4k}(k-1)^{4(k-1)}}\Bigg\}\log\left(\tfrac{2nk(n-1)}{\delta}\right)\Bigg)
\end{align*}


\section{Sample Complexity Lower Bound}\label{ap:sc_lower}

\subsection{Preliminaries}
In this section, we present some definitions, and results that we will use for our lower bound proof. 
\paragraph{Information theoretic lower bound:} 
We now present the information theoretic lower bound for required samples in recovering a distribution.

We first define the symmetrized KL-divergence between two distributions $P$ and $Q$ as 
$$
J(P, Q) = \mathbb{E}_{\mathbf{X}\sim P}\log\left(\frac{P(\mathbf{X})}{Q(\mathbf{X})}\right) 
+ \mathbb{E}_{\mathbf{X}\sim Q}\log\left(\frac{P(\mathbf{X})}{Q(\mathbf{X})}\right).
$$

\begin{lemma}[Fano’s Inequality, Lemma 6.2 in Bresler et al.\cite{bresler2020learning}]\label{lemm:fano}
For $M\geq 2$, given the $(M+1)$  distributions $\{P_0,\dots, P_M\}$, for any estimator $\Psi: [k]^n\times N \to \{0,1,\dots, M\}$ that uses $N$ i.i.d. samples $\mathbf{X}'(1:N)$, and for any $\delta >0$ we have for
$$
N\leq (1-\delta) \frac{\log(M)}{\tfrac{1}{M+1}\sum_{k=1}^{M}  J(P^{(k)}, P^{(0)})},\quad \inf_{\Psi} \max_{0\leq k\leq M} P^{(j)}(\Psi(\mathbf{X}'(1:N))\neq j)  \geq \delta - \tfrac{1}{\log(M)}.
$$
\end{lemma}

The above inequality provides such a characterization in the minimax sense. In particular, it says among the $M$ distributions there exists at least one from which $N$ (as defined in the lemma) i.i.d. samples are required to identify that distribution correctly with probability at least $(1-\delta + \tfrac{1}{\log(M)})$.

\paragraph{Symmetric Graphical Models:} For symmetric graphical models~\cite{choi2011learning}, the marginals of all the random variables are uniform on the support and the conditional distribution for two random variables $X_i$, $X_j$ such that $(X_i, X_j)\in \sete$ is given by:
$$
P_{i| j} = \alpha_{i,j}I + (1-\alpha_{i,j})\frac{O}{k},
$$

where $O$ is the $k\times k$ matrix of all $1's$, $k$ is the support size, and $0<\alpha_{i,j}<1$. This characterization has the following property:
\begin{lemma}
Consider any 2 nodes $X_{i_1}$, $X_{i_t}$ in a symmetric graphical model such that the path between $X_{i_1}$ and $X_{i_t}$ is $X_{i_1}-X_{i_2}-\dots -X_{i_{t-1}}-X_{i_t}$. Then, the conditional PMF matrix of $X_{i_1}$ conditioned on $X_{i_t}$ is given as follows:
\begin{align*}
 P_{{i_1}| {i_t}} &= \alpha_{i_1,i_t}I + (1-\alpha_{i_1,i_t})\frac{O}{k} =\prod_{p = 1}^{t-1}\alpha_{i_p, i_{p+1}}I + \left(1-\prod_{p = 1}^{t-1}\alpha_{i_p, i_{p+1}}\right)\frac{O}{k},
\end{align*}
that is, $\alpha_{i_1,i_t} = \prod_{p = 1}^{t-1}\alpha_{i_p, i_{p+1}}$
\end{lemma}
We remark that when considering noisy random variables we have that:
$$
P_{i'| i} =(1-q_i)I + q_i\frac{O}{k}.
$$
For each node $X_i$, we define $\alpha_{i',i} = 1-q_i$. Therefore, we get:
$$
P_{i'| i} =\alpha_{i',i}I + (1-\alpha_{i',i})\frac{O}{k},
$$
such that $\alpha_{i',i}>0$ (as $q_i\leq q_{\max} <1$).

\paragraph{Circulant Matrices:}
Let $\mathcal{R}$ be a rotational operation of a vector $v \in \mathbb{R}^{k}$ which maps it to $v'=\mathcal{R}(v) \in \mathbb{R}^{k}$ with $v'(i) = v((i+1)\mathrm{mod}k)$ for all $1\leq i \leq k$. Then we have $v'' = \mathcal{R}^j(v)$ as $v''(i) = v((i+j)\mathrm{mod}k)$ for any $j \geq 1$, and for all $1\leq i \leq k$.Then a ciculant matrix created from vector $v$ is given as 
$Cir(v) = (v; \mathcal{R}(v); \mathcal{R}^2(v); \dots; \mathcal{R}^{(k-1)}(v))$.
For any circulant matrix in $\mathbb{R}^{k\times k}$ with vector $v$, denoted as $Cir(v)$, the determinant is given as 
$$\det(Cir(v)) = \prod_{j=0}^{k-1} \sum_{i=0}^{k-1} v_i \omega^{ji}.$$

The following lemma states that when a graphical model has the conditional PMF as circulant matrix for each edge, then if one node has uniform  marginal then all other nodes have uniform marginals as well.

\begin{lemma}
Consider a tree graphical model such that the conditional PMF matrix corresponding to every edge is a circulant matrix. Then, if the marginals of one of the nodes is uniformly distributed on the support, the marginals of all the remaining nodes are also uniform.
\end{lemma}
\begin{proof}
Suppose the node with uniform marginals is $X_1$. Suppose node $X_2$ has an edge with $X_1$ and $P(X_2|X_1)$ is a circulant matrix. Thus we have $P(X_2, X_1) = \frac{P(X_2|X_1)}{k}$. Therefore, $P(X_2, X_1)$ is also a circulant matrix. When the joint PMF matrix is circulant, all the rows and columns  the marginal distribution of both the random variables is uniform. Therefore, the marginal distribution of $X_2$ is also uniform. Thus the marginal distribution of all the nodes connected to $X_1$ is uniform. Once we know that the marginals of one hop neighbors of $X_1$ are uniform, we can infer the same about the two hop neighbors of $X_1$. This can further be extended for all the nodes in the graph.
\end{proof}

\underline{\em Simplifying the Quadratic Bound}:
Suppose the marginals of all the random variables are uniform, that is, $P'_b = \frac{1}{k} I$ and the underlying graphical model on $X_a, X_b, X_c$ is a chain with $X_a$ as the center node. We want to bound the following quadratic:
\begin{align*}
{Q}(x)&= \|\frac{x^2}{k^2}(O - kI) - \frac{x}{k}(O {P}'_{b} +  {P}'_{b}O - k {P}'_{b} - I) +  {P}_{b',c'}  {P}_{a',c'}^{-1} {P}_{a',b'}- {P}'_{b}\|_F.
\end{align*}
The conditional independence relation gives us $P_{b,c} = P_{b,a}P_{a}^{-1}P_{a_c}$. Recall that $E_a = (1-q_a)I + \frac{q_a}{k}O$ and similarly we have $E_b, E_c$. We have the following:
\begin{align*}
    {P}_{b',c'}  {P}_{a',c'}^{-1} {P}_{a',b'} &= E_bP_{b,c}E_c(E_aP_{a,c}E_c)^{-1}E_aP_{a,b}E_b\\
    &= E_bP_{b,a}P_a^{-1}P_{a,c}E_cE_c^{-1}P_{a,c}^{-1}E_a^{-1}E_aP_{a,b}E_b\\
    &= E_bP_{b,a}P_{a}^{-1}P_{a,b}E_b
\end{align*}

In the circulant setting, we have that $P_{a} = \tfrac{1}{k} I$ . This gives us ${P}_{b',c'}  {P}_{a',c'}^{-1} {P}_{a',b'} = kE_bP_{b,a}P_{a,b}E_b$.
Substituting these in the quadratic, we get:
\begin{align}\label{eq:q_simplify}
  {Q}(x)&= \|\frac{x^2}{k^2}(O - kI) - \frac{x}{k}(O {P}'_{b} +  {P}'_{b}O - k {P}'_{b} - I) +  {P}_{b',c'}  {P}_{a',c'}^{-1} {P}_{a',b'}- {P}'_{b}\|_F,\\  
  &=\|\left(\frac{x^2-2x+1}{k^2}\right)O - \left(\frac{x^2-2x+1}{k}\right)I - \frac{O}{k^2} + kE_bP_{b,a}P_{a,b}E_b\|_F,  \\
  &= \|\left(\frac{x-1}{k}\right)^2(O - kI) - \frac{O}{k^2} + kE_bP_{b,a}P_{a,b}E_b\|_F
\end{align}

\paragraph{Perturbed Symmetric Distribution:}
We now focus on a special case of circulant matrices which will be used in our lower bound construction later on. The conditional PMF for two nodes $a$ and $b$ in a perturbed symmetric distribution model  takes the following form:
$$
P_{b|a} = (\alpha - \delta)I + (1-\alpha)\frac{O}{k} + \Delta 
$$
$$
\Delta = \begin{bmatrix} 
0 & \delta & 0 &\dots &0 \\
0 & 0 & \delta &\dots &0 \\
\vdots & \vdots & \vdots & \vdots & \vdots \\
0 & 0 & 0 & \dots & \delta\\
\delta & 0 & 0 & \dots & 0\\
\end{bmatrix}.
$$
Note that this is a class that we define by perturbing the discrete symmetric model slightly.

We first consider the noiseless setting ($E_b = I$). In order to obtain the results for the noisy case, it is sufficient to replace $\alpha$ by $(1-q)\alpha$ and $\delta$ by $(1-q)\delta$.
For our model, we have that:
$$
P_{b,a} = \frac{1}{k}\left((\alpha - \delta)I + (1-\alpha)\frac{O}{k} + \Delta\right).
$$
Noting that $P_{b,a} = P_{a,b}^T$, $\Delta\Delta^T = \delta^2I$, $\Delta O = O\Delta^T = \delta O$, we get:
$$P_{b,a}P_{a,b} = \frac{1}{k^2}\left(((\alpha-\delta)^2+\delta^2)I+\frac{O}{k}(1-\alpha^2)+(\alpha - \delta)(\Delta^T+\Delta)\right)$$

\underline{\em  Lower bounding the Quadratic Bound}: Substituting this in Equation \eqref{eq:q_simplify} along with $E_b = I$, we get:
\begin{align*}
    Q^2(x) =& \|\left(\frac{x-1}{k}\right)^2(O - kI) - \frac{O}{k^2} + kE_bP_{b,a}P_{a,b}E_b\|_F^2\\
    =& \|\left(\frac{x-1}{k}\right)^2(O - kI) + ((\alpha-\delta)^2+\delta^2)\frac{I}{k}-\alpha^2\frac{O}{k^2}+\frac{(\alpha - \delta)}{k}(\Delta^T+\Delta) \|_F^2\\
\end{align*}
Each diagonal element (total $k$) of the matrix is $\left(\frac{x-1}{k}\right)^2 - \frac{(x-1)^2}{k}+\frac{(\alpha-\delta)^2+\delta^2}{k}-\frac{\alpha^2}{k^2}$.\\
Each element at the positions of the support $(\Delta + \Delta^T)$ (total $2k$) is $\left(\frac{x-1}{k}\right)^2-\frac{\alpha^2}{k^2}+\frac{\delta(\alpha-\delta)}{k}$.\\
Every remaining element (total $k^2-3k$) is $\left(\frac{x-1}{k}\right)^2-\frac{\alpha^2}{k^2}$.
To simplify the above equation, we define $\gamma = (1-x)^2 - \alpha^2$, $e = \delta(\alpha-\delta)$.
Each diagonal element is $\frac{\gamma}{k^2} - \frac{\gamma}{k}-\frac{2e}{k}$.\\
Each element at the positions of the support $(\Delta + \Delta^T)$ (total $2k$) is $\frac{\gamma}{k^2} + \frac{e}{k}$.\\
Every remaining element (total $k^2-3k$) is $\frac{\gamma}{k^2}$.
Thus, we get:
\begin{align*}
    Q^2(x) =& k\left(\frac{\gamma}{k^2} - \frac{\gamma}{k}-\frac{2e}{k}\right)^2 + 2k\left(\frac{\gamma}{k^2} + \frac{e}{k}\right)^2 + (k^2-3k)\frac{\gamma^2}{k^4}\\
    =&  \tfrac{1}{k^3}\left((k-1)\gamma + 2k e\right)^2 + \tfrac{2}{k^3}\left(\gamma + ke\right)^2 + \tfrac{k-3}{k^3}\gamma^2
\end{align*}
$Q^2(x)$ is minimized for $\gamma = -\frac{2ke}{k-1}$. Substituting this, we get:
\begin{equation}\label{eq:pos_err}
Q^2(x)\geq \frac{2(k-3)\delta^2(\alpha-\delta)^2k^2}{k-1}.
\end{equation}

\underline{\em Computing the determinant of conditional PMF}: Let us consider the perturbed symmetric distribution $C(v(\theta, \theta'))$ with the vector
$$v(\theta, \theta') = \left( (1-\theta'-(K-2)\theta), \theta', \underbrace{\theta, \dots,\theta}_{k-2 \text{ times}}\right).$$
For $\theta = \tfrac{1-\alpha}{k}$ and $\delta = (\theta'-\theta)$ we have $C(v(\theta, \theta')) = P_{b|a}$. We make this switch as this helps us computing the determinant easily. 

We now derive some of the necessary results which we will apply in our lower bound graph construction. The determinant of the matrix $C(v(\theta, \theta'))$ is derived first. We have for any $j =0$ to $k-1$, 
\begin{align*}
\sum_{i=0}^{k-1} v(\theta, \theta')_i \omega^{ji}
&= (1-\theta'-(k-2)\theta) + \theta'\omega^j + \theta \sum_{i=2}^{k-1}\omega^{ji}\\
&= (1-\theta'-(k-1)\theta) + (\theta' - \theta)\omega^j + \theta \sum_{i=0}^{k-1}\omega^{ji}\\
&=\begin{cases}
1 = (1-\theta'-(k-1)\theta) + (\theta' - \theta) + k\theta,\, j=0\\
(1-\theta'-(k-1)\theta) + (\theta' - \theta)\omega^j,\, j\neq 0
\end{cases}
\end{align*}
Therefore, we have following the derivations in \cite{circulant} 
\begin{align*}
    \det(P_{b|a})=\det(Cir(v(\theta, \theta'))) &= \prod_{j=1}^{k-1}\left((1-\theta'-(k-1)\theta) - (\theta - \theta')\omega^j\right) \\
    &= \frac{(1-\theta'-(k-1)\theta)^k}{(1-k\theta)} \prod_{j=0}^{k-1}\left(1 - \tfrac{(\theta - \theta')}{(1-\theta'-(k-1)\theta)}\omega^j\right)\\
    &= \frac{(1-\theta'-(k-1)\theta)^k- (\theta - \theta')^k}{(1-k\theta)}\\
    &= (1-k\theta)^{(k-1)}\left( \left(1 - \tfrac{\theta' - \theta}{1-k\theta}\right)^k - \left(\tfrac{\theta - \theta'}{1-k\theta}\right)^k \right)\\
    &= \alpha^{(k-1)}\left( \left(1 - \tfrac{\delta}{\alpha}\right)^k - \left(\tfrac{-\delta}{\alpha}\right)^k \right)
\end{align*}
In the last line we substitute $\alpha = (1-k\theta)$ and $\delta = (\theta' -\theta)$ to get back to the form common to other parts of the proof.

\subsection{Lower Bound for recovering the equivalence class of trees}
In this section we derive the lower bound on the sample complexity to recover the equivalence class when the underlying model has is totally unidentifiable (no leaf is distinguishable from it's parent). For this purpose, we consider the symmetric class of tree graphical models. 

\paragraph{Family of distributions:} With the above background, we are now ready to derive the lower bounds.  We consider the family of probability distributions which is structurally similar to Appendix A in \cite{tandon2021sga}, but uses discrete symmetric distribution instead of using Ising models. The family of distributions is given as $(P^{(i)}: i = 0, 1, \dots, t^2-1)$. The graph $P^{(0)}$  consists of $n = 2t+1$ nodes $(1,2,\dots, 2t+1)$. Here, we use odd number of nodes for simplifying exposition. There are $2t$ edges where node $j = 1, \dots, 2t$ are connected to node $(2t+1)$. Nodes $1, 2\dots t$ have distance $d_{max}$ from node $2t+1$ and are corrupted with probability $q_{max}$. Nodes $t+1, t+2\dots 2t$ have distance $d_{min}$ from node $2t+1$ and have $0$ probability of error. Node $2t+1$ also has $0$ probability of error.
This is shown in Figure \ref{fig:lowerbound1}. The edges have two different type of conditional as described below.  

\begin{gather*}
    P^{(0)}_{j'|(2t+1)'} = \alpha_{min}(1-q_{max}) I + (1-\alpha_{min}(1-q_{max}))\frac{O}{k},\, \forall j\in \{1, 2\dots t\},\\
    P^{(0)}_{j'|(2t+1)'} = \alpha_{max} I + (1-\alpha_{max})\frac{O}{k},\, \forall j\in \{t+1, t+2\dots 2t\}.
\end{gather*} 

For any $i = 1, \dots, t^2-1$, the distribution $P^{(i)}$  is constructed from  $P^{(0)}$ by disconnecting the edge $(i_a,2t+1)$, and adding edge $(i_b+t, 2t+1)$ where 
$
i_a = (1+ \lfloor\tfrac{i-1}{t}\rfloor), \text{ and } i_b = i -  \lfloor\tfrac{i-1}{t}\rfloor t. 
$
As noted in \cite{tandon2021sga}, the pair $(i_a,i_b)$ is unique for every $i =1,\dots, t^2-1$. We use another discrete symmetric distribution for all these edges: $(i_a,i_b)$ for any $i =1,\dots, t^2-1$. Specifically, the conditional pmf of the different edges of the $i$-th graphical model is  given below.
\begin{gather*}
    P^{(i)}_{j'|(2t+1)'} = \alpha_{min}(1-q_{max}) I + (1-\alpha_{min}(1-q_{max}))\frac{O}{k}\, \forall i\in \{1,2,3,\dots t\}\setminus \{i_a\},\\
    P^{(i)}_{j'|(2t+1)'} = \alpha_{max} I + (1-\alpha_{max})\frac{O}{k}\, \forall i\in {t+1, t+2\dots 2t},\\
    P^{(i)}_{i_a'|i_b'} = \alpha_{min}(1-q_{max}) I + (1-\alpha_{min}(1-q_{max}))\frac{O}{k}.
\end{gather*}

We finally note that all the graphs $P^{(i)}$ for $i \in \{0,1,\dots, t^2-1\}$ have a different equivalence class. In particular, we see that $P^{(0)}$ admits all possible permutation of star nodes (with node $i$ being the root, and remaining $2t$ nodes being the leaf nodes, for all $i \in \{1,\dots, 2t+1\}$). For $P^{(i)}$, the equivalence structure is given by two leaf clusters connected by a single edge. The nodes $\{1,\dots, 2t+1\}\setminus \{i_a, i_b\}$ forms one leaf cluster, while $\{i_a, i_b\}$ forms the other leaf cluster. As $(i_a, i_b)$ is unique for all $i\in \{1,\dots,t^2-1\}$, all the $t^2$ graphs under consideration have different equivalence classes (see, Figure~\ref{fig:lowerbound1}). Also, all the leaf nodes are indistinguishable from it's parents in each of these graphs. 

\begin{figure}[H]
    \centering
    \includegraphics[width=0.7\linewidth]{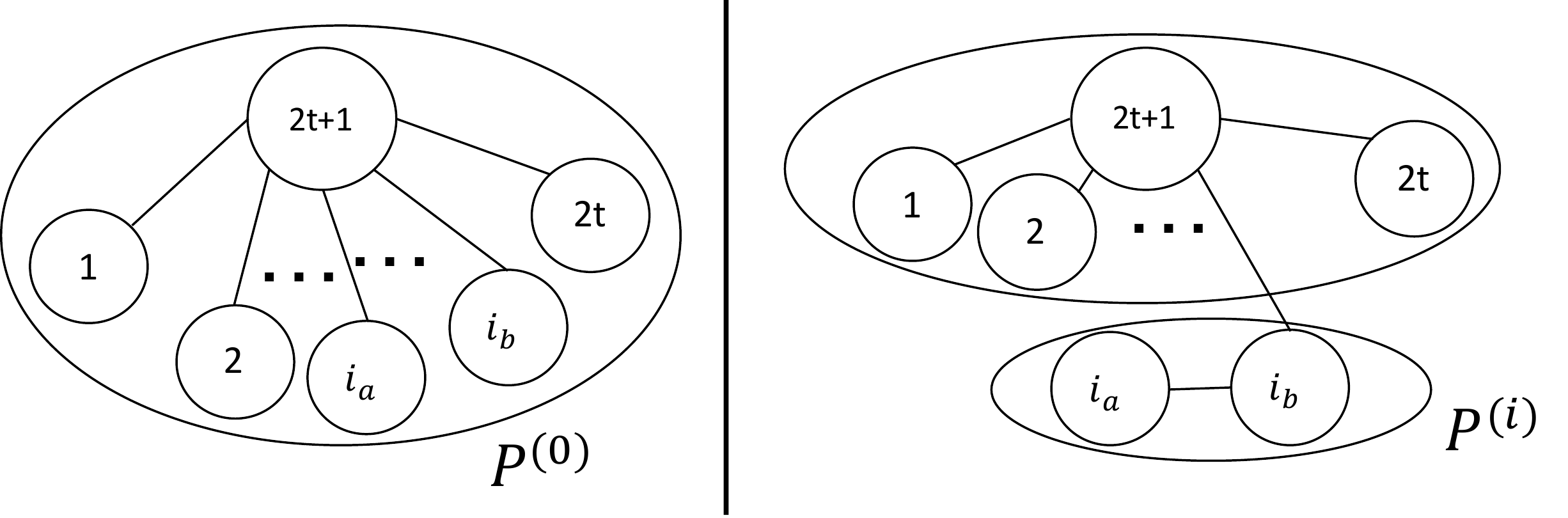}
    \caption{The family of distributions used for providing lower bound for completely unidentifiable case. The graphical model corresponding to $P^{(0)}$ a single recoverable leaf cluster. The graphical model corresponding to $P^{(i)}$, for each $i=1,\dots,t^2-1$,  has nodes $\{i_a,i_b\}$ as one recoverable leaf cluster, and the remaining nodes  as another recoverable leaf cluster.}
    \label{fig:lowerbound1}
\end{figure}

\paragraph{Symmetrized KL-divergence:} For the symmetrized KL-divergence $J(P^{(0)}, P^{(1)})$ computation we focus our attention on $i=1$, in which case $i_a=1$ and $i_b=(t+1)$. The computation remains identical for other $i \geq 2$ due to symmetry. 

For this purpose, we need to compute $\mathbb{E}_{\mathbf{X}\sim P^{(0)}}\log\left(\frac{P^{(0)}(\mathbf{X})}{P^{(1)}(\mathbf{X})}\right)$ and  $\mathbb{E}_{\mathbf{X}\sim P^{(1)}}\log\left(\frac{P^{(1)}(\mathbf{X})}{P^{(0)}(\mathbf{X})}\right)$.
Let us look at $\left(\frac{P^{(0)}(\mathbf{X})}{P^{(1)}(\mathbf{X})}\right)$. Recall that the nodes $t+1, t+2\dots 2t+1$ have 0 noise. 
We first see that the expression for $P^{(0)}(\mathbf{X})$ can be decomposed as follows due to the discrete symmetric conditional PMF and the graph structure:
$$
P^{(0)}(\mathbf{X}) = P^{(0)}(X_{2t+1}')\prod_{i=1}^{2t}P^{(0)}(X_{i}'|X_{2t+1}')
$$
Similarly, the decomposition for $P^{(1)}(\mathbf{X})$ is:
$$
P^{(1)}(\mathbf{X}) = P^{(1)}(X_{2t+1}')P^{(1)}(X_{1}'|X_{t+1}')\prod_{i=2}^{2t}P^{(0)}(X_{i}'|X_{2t+1}')
$$
Furthermore, due to the property of discrete symmetric model we have $P^{(0)}(X_{2t+1}') = P^{(1)}(X_{2t+1}') = 1/k$. This gives us:
$$
\frac{P^{(0)}(\mathbf{X})}{P^{(1)}(\mathbf{X})} = \frac{P^{(0)}(X_1'|X_{2t+1}')}{P^{(1)}(X_1'|X_{t+1}')}.
$$
Therefore,
$$
\mathbb{E}_{\mathbf{X}\sim P^{(0)}}\log\left(\frac{P^{(0)}(\mathbf{X})}{P^{(1)}(\mathbf{X})}\right) = \mathbb{E}_{\mathbf{X}\sim P^{(0)}}\log(P^{(0)}(X_1'|X_{2t+1}') - \mathbb{E}_{\mathbf{X}\sim P^{(0)}}\log(P^{(1)}(X_1'|X_{t+1}')))
$$

We find the symmetrized KL divergence between $P^{(0)}$ and  $P^{(1)}$.  
We primarily need the following four conditional PMF matrices for the calculation of the symmetrized KL divergence:
\begin{align}\label{eq:con_pmf_symm}
  P^{(0)}_{1'|(2t+1)'} &= (1-q_{\max})\alpha_{min}I + \left(1 - (1-q_{\max})\alpha_{min}\right)\frac{O}{k}\\
  P^{(0)}_{1'|(t+1)'} &= (1-q_{\max})\alpha_{min}\alpha_{max}I + \left(1 - (1-q_{\max})\alpha_{min}\alpha_{max}\right)\frac{O}{k}\\
  P^{(1)}_{1'|(2t+1)'} &= (1-q_{\max})\alpha_{min}\alpha_{max}I + \left(1 -  (1-q_{\max})\alpha_{min}\alpha_{max}\right)\frac{O}{k}\\
  P^{(1)}_{1'|(t+1)'} &= (1-q_{\max})\alpha_{min}I + \left(1 - (1-q_{\max})\alpha_{min}\right)\frac{O}{k}
\end{align}

For notational simplicity let us use $\alpha_{n} = (1-q_{\max})$.
\begin{align*}
    &\mathbb{E}_{\mathbf{X}\sim P^{(0)}}\log(P^{(0)}(X_1'|X_{2t+1}') \\
    &= \mathbb{E}_{{X_1', X_{2t+1}'}\sim P^{(0)}}\log(P^{(0)}(X_1'|X_{2t+1}')\\
    &= \sum_{(X_1',X_{2t+1}')\in\mathcal{S}^2} P^{(0)}(X_1',X_{2t+1}')\log(P^{(0)}(X_1'|X_{2t+1}')\\
    &= \frac{1}{k}    \sum_{(X_1',X_{2t+1}')\in\mathcal{S}^2} P^{(0)}(X_1'|X_{2t+1}')\log(P^{(0)}(X_1'|X_{2t+1}')\\
    &= \frac{1}{k}    \left(\sum_{(X_1'=X_{2t+1}')} P^{(0)}(X_1'|X_{2t+1}')\log(P^{(0)}(X_1'|X_{2t+1}') + \sum_{(X_1'\neq X_{2t+1}')} P^{(0)}(X_1'|X_{2t+1}')\log(P^{(0)}(X_1'|X_{2t+1}')\right)\\
    &= \frac{1}{k}\left(k\left(\alpha_{min}\alpha_{n}  + \frac{1-\alpha_{min}\alpha_{n} }{k}\right)\log\left(\alpha_{min}\alpha_{n}  + \frac{1-\alpha_{min}\alpha_{n} }{k}\right)\right)\\
    &+ \frac{1}{k} \left((k^2-k)\left( \frac{1-\alpha_{min}\alpha_{n} }{k}\right)\log\left(\frac{1-\alpha_{min}\alpha_{n} }{k}\right)\right)
\end{align*}
For the second term we have similarly,
\begin{align*}
    &\mathbb{E}_{\mathbf{X}\sim P^{(0)}}\log(P^{(1)}(X_1'|X_{t+1}') \\
    &= \frac{1}{k}    \left(\sum_{(X_1'=X_{t+1}')} P^{(0)}(X_1'|X_{t+1}')\log(P^{(1)}(X_1'|X_{t+1}') + \sum_{(X_1'\neq X_{t+1}')} P^{(0)}(X_1'|X_{t+1}')\log(P^{(1)}(X_1'|X_{t+1}')\right)\\
    &= \frac{1}{k}\left(k\left(\alpha_{min}\alpha_{max}\alpha_{n}  + \frac{1-\alpha_{min}\alpha_{max}\alpha_{n} }{k}\right)\log\left(\alpha_{min}\alpha_{n}  + \frac{1-\alpha_{min}\alpha_{n} }{k}\right)\right)\\
    &+ \frac{1}{k} \left((k^2-k)\left( \frac{1-\alpha_{max}\alpha_{min}\alpha_{n} }{k}\right)\log\left(\frac{1-\alpha_{min}\alpha_{n} }{k}\right)\right)
\end{align*}

Recall the p.m.f. for a tree structured graphical model with vertex set $V$ and edge set $E$, and alphabet $\mathcal{X} = [K]^{|V|}$, is 
$$
P(\mathbf{X}) = \prod_{i \in V} P_{i}(X_i) \prod_{(i,j) \in E} \frac{P_{i,j}(X_i, X_j)}{P_{i}(X_i),P_{j}(X_j)},
$$
In the symmetric setting, we get that:
$$
P(\mathbf{X}) = \tfrac{1}{k} \prod_{(i,j) \in E} P_{i,j}(X_i| X_j).
$$
Computing the symmetrized KL divergence involves calculating the following 4 terms which can be done using Equation \eqref{eq:con_pmf_symm}:
\begin{align*}
    \mathbb{E}_{P^{(0)}}{\log(P^{(0)}(X'_{2t+1}|X'_{1}))} = &\left(\alpha_{min}\alpha_{n} +\tfrac{(1-\alpha_{min}\alpha_{n} )}{k}\right)\log\left(\alpha_{min}\alpha_{n} +\tfrac{(1-\alpha_{min}\alpha_{n} )}{k}\right)\\
    +&\left(\tfrac{(k-1)}{k}(1-\alpha_{min}\alpha_{n} )\log\left(\tfrac{1-\alpha_{min}\alpha_{n} }{k}\right)\right)\\
    \mathbb{E}_{P^{(0)}}{\log(P^{(1)}(X'_{t+1}|X'_{1}))} = &\left(\alpha_{min}\alpha_{max}\alpha_{n} +\tfrac{(1-\alpha_{min}\alpha_{max}\alpha_{n} )}{k}\right)\log\left(\alpha_{min}\alpha_{n} +\tfrac{(1-\alpha_{min}\alpha_{n} )}{k}\right)\\
    +&\left(\tfrac{(k-1)}{k}(1-\alpha_{min}\alpha_{max}\alpha_{n} )\log\left(\tfrac{1-\alpha_{min}\alpha_{n} }{k}\right)\right)\\
    \mathbb{E}_{P^{(1)}}{\log(P^{(1)}(X'_{t+1}|X'_{1}))} = &\left(\alpha_{min}\alpha_{n} +\tfrac{(1-\alpha_{min}\alpha_{n} )}{k}\right)\log\left(\alpha_{min}\alpha_{n} +\tfrac{(1-\alpha_{min}\alpha_{n} )}{k}\right)\\
    +&\left(\tfrac{(k-1)}{k}(1-\alpha_{min}\alpha_{n} )\log\left(\tfrac{1-\alpha_{min}\alpha_{n} }{k}\right)\right)\\
    \mathbb{E}_{P^{(1)}}{\log(P^{(0)}(X'_{2t+1}|X'_{1}))} = &\left(\alpha_{min}\alpha_{max}\alpha_{n} +\tfrac{(1-\alpha_{min}\alpha_{max}\alpha_{n} )}{k}\right)\log\left(\alpha_{min}\alpha_{n} +\tfrac{(1-\alpha_{min}\alpha_{n} )}{k}\right)\\
    +&\left(\tfrac{(k-1)}{k}(1-\alpha_{min}\alpha_{max}\alpha_{n} )\log\left(\tfrac{1-\alpha_{min}\alpha_{n} }{k}\right)\right)\\
\end{align*}

This gives us:
\begin{align*}
    J(P^{(0)}, P^{(1)}) = \mathbb{E}_{P^{(0)}}{\log\tfrac{(P^{(0)}(X'_{2t+1}|X'_{1}))}{(P^{(1)}(X'_{t+1}|X'_{1}))} + \mathbb{E}_{P^{(1)}}{\log\tfrac{(P^{(1)}(X'_{t+1}|X'_{1}))}{(P^{(0)}(X'_{2t+1}|X'_{1}))}}}
\end{align*}
Substituting these quantities from above and simplifying, we get:
\begin{align*}
     J(P^{(0)}, P^{(1)}) &= 2\alpha_{min}\alpha_{n}(1-\alpha_{max})\left(\tfrac{k-1}{k}\right)\log\left(1+\tfrac{k\alpha_{min}\alpha_{n}}{1-\alpha_{min}\alpha_{n}}\right)\\
    &= 2\exp(-\tfrac{d_{\max}}{k-1})(1-q_{\max})(1-\exp(-\tfrac{d_{\min}}{k-1}))\left(\tfrac{k-1}{k}\right)\log\left(1+\tfrac{k\exp(-\tfrac{d_{\max}}{k-1})(1-q_{\max})}{1-\exp(-\tfrac{d_{\max}}{k-1})(1-q_{\max})}\right)\\
    &\leq 2(k-1)\exp(-\tfrac{2d_{\max}}{k-1})(1-q_{\max})^2(1-\exp(-\tfrac{d_{\min}}{k-1}))
\end{align*}
We have the maximum distance between two nodes given as $d_{\max} = -(k-1)\log(\alpha_{min})$ and $d_{\min} = -(k-1)\log(\alpha_{max})$. The noise is related as $\alpha_{n} = (1-q_{max})$. Substituting, these terms above provides us the second equality. Using $\log(1+x)\leq x $ gives the final inequality. 

\paragraph{Lower Bound Proof - Part I:}
We are now in a position to prove the first part of Theorem~\ref{th:lb}.

By the application of Lemma~\ref{lemm:fano}, and expressions of $J(P^{(0)}, P^{(k)})$ we obtain that for attaining a probability error of at most  $\delta > 0$ we require at least $N$ samples where
\begin{align*}
    N &> (1- \delta + \tfrac{1}{\log(n)}) \frac{2\log(n)}{\tfrac{n^2}{n^2+1} 2(k-1)\exp(-\tfrac{2d_{\max}}{k-1})(1-q_{\max})^2(1-\exp(-\tfrac{d_{\min}}{k-1}))}\\
    &\geq  \frac{(1- \delta)\exp(\tfrac{2d_{\max}}{k-1}) \log(n)}{(k-1)(1-q_{\max})^2(1-\exp(-\tfrac{d_{\min}}{k-1}))}
\end{align*}


\subsection{Lower bound for recovering $\mathcal{T}_{T^*}^{sub}$ when $\mathcal{T}_{T^*}^{sub}\subset \mathcal{T}_{T^*}$}
In this section, we focus on the dependence of $t_0$ which can not be captured when the graph is completely unidentifiable. Therefore, we create graphs using perturbed symmetric distribution where the graph is partly identifiable (a subset of leaf nodes is distinguishable from it's parent).

\paragraph{Family of distributions:}
We consider graphical models with random variables whose support size is $k\geq 4$. We construct a family of $n+1$ star structured distributions on $n+1$ nodes (as shown in Figure \ref{fig:lowerbound2}), $P^{(0)}, P^{(1)}, \dots, P^{(n)}$, such that $P^{(0)}$ is completely identifiable while $P^{(i)}$ is such that leaf node $i$ and the center node $0$ is unidentifiable.

We next provide the details of the family of graphical models. 
\begin{figure}[H]
    \centering
    \includegraphics[width=0.7\linewidth]{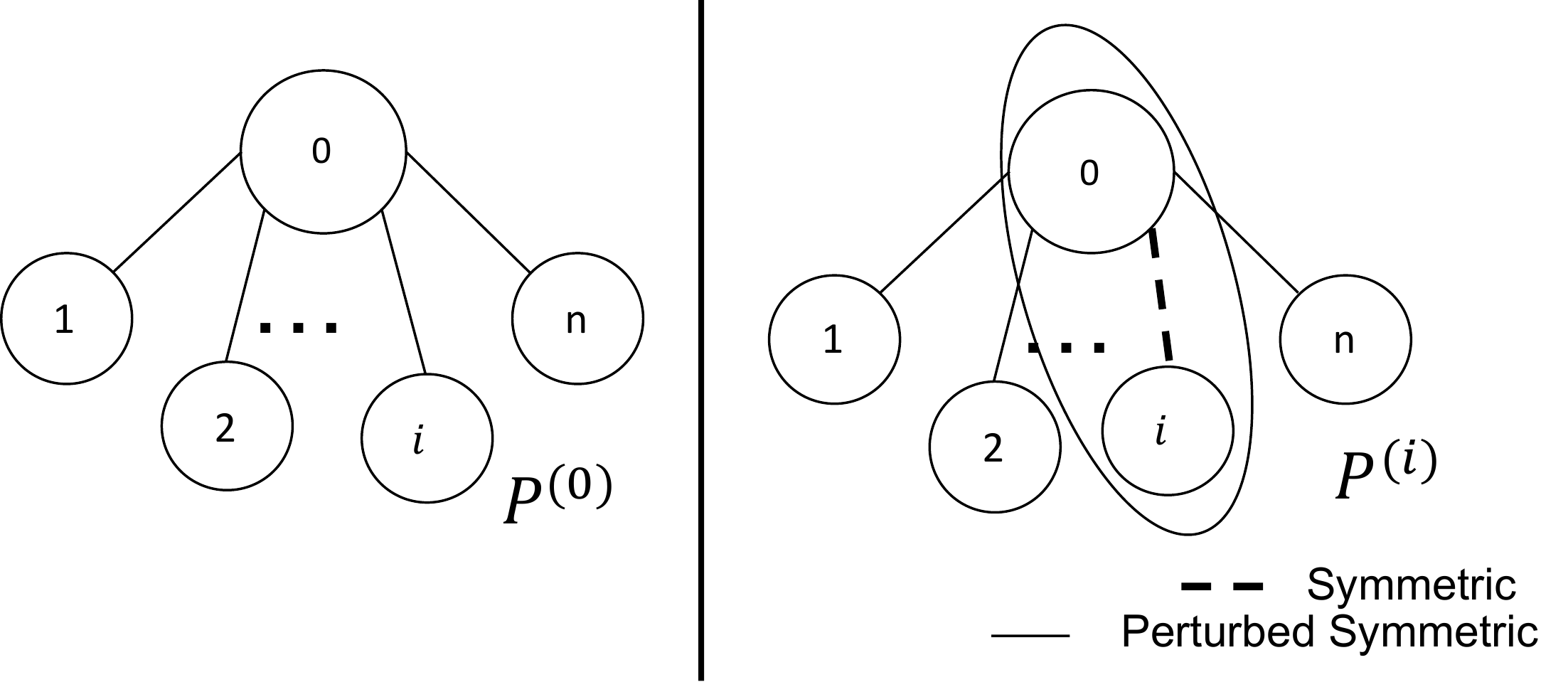}
    \caption{The family of distributions used for providing lower bound with $t_0$ dependence. The graphical model corresponding to $P^{(0)}$ is completely identifiable. The graphical model corresponding to $P^{(i)}$, for each $i=1,\dots,n$,  has edge $\{i,0\}$ which forms a recoverable leaf cluster, and the rest are all identifiable.}
    \label{fig:lowerbound2}
\end{figure}

For $P^{(0)}$, the conditional distribution matrices are as follows:
$$
P^{(0)}_{j|0} = (\alpha - \delta)I + (1-\alpha)\frac{O}{k} + \Delta, \forall j\in[n],
$$
where 
$$
\Delta = \begin{bmatrix} 
0 & \delta & 0 &\dots &0 \\
0 & 0 & \delta &\dots &0 \\
\vdots & \vdots & \vdots & \vdots & \vdots \\
0 & 0 & 0 & \dots & \delta\\
\delta & 0 & 0 & \dots & 0\\
\end{bmatrix}
$$
For $P^{(i)}$, the conditional distribution matrices are as follows:
\begin{align*}
P^{(i)}_{j|0} &= (\alpha - \delta)I + (1-\alpha)\frac{O}{k} + \Delta, \forall j\in[n], j\neq i.\\
P^{(i)}_{i|0} &= \alpha I + (1-\alpha)\frac{O}{k}.
\end{align*}
Recall from Equation \eqref{eq:pos_err}, this conditional distribution ensures that in $P^{(0)}$, all the leaves can be identified. It also ensures that in $P^{(i)}$ all the leaves other than $i$ can be identified.
It is easy to see that $(\alpha - \delta)I + (1-\alpha)\frac{O}{k} + \Delta = C(v(\theta, \theta'))$ for $\theta = \frac{1-\alpha}{k}, \theta' = \frac{1-\alpha}{k} + \delta$. The marginals of all the random variables in all the distributions are uniform on the support. Given the graph structure and the uniform marginals, the joint PMF of the random variables can be decomposed as follows:
\begin{align}\label{eq:pmf_decomp}
  P^{(0)}(\mathbf{X}) = \frac{1}{k}\prod_{j=1}^nP^{(0)}(X_j|X_0),\\
  P^{(i)}(\mathbf{X}) = \frac{1}{k}\prod_{j=1}^nP^{(i)}(X_j|X_0).
\end{align}
Recall that $P^{(0)}_{X_j|X_0}$ is the matrix form of conditional distribution whereas 
$P^{(0)}(X_j|X_0)$ is the scalar value of the conditional PMF for any $X_j$ and $X_0$.

\paragraph{KL Divergence Computation}
We now calculate the symmetrized KL divergence between $P^{(0)}$ and $P^{(i)}$ for $i\neq 0$ denoted by $J(P^{(0)}, P^{(i)})$. 
$$
J(P^{(0)}, P^{(i)}) = \mathbb{E}_{\mathbf{X}\sim P^{(i)}}\log\frac{P^{(i)}(\mathbf{X})}{P^{(0)}(\mathbf{X})} + \mathbb{E}_{\mathbf{X}\sim P^{(0)}}\log\frac{P^{(0)}(\mathbf{X})}{P^{(i)}(\mathbf{X})}
$$
Substituting $P^{(0)}(\mathbf{X}), P^{(i)}(\mathbf{X})$ from equation \ref{eq:pmf_decomp} and noting that $P^{(0)}(X_j|X_0) = P^{(i)}(X_j|X_0)$ $\forall j\neq i$, we get that:
\begin{align*}
    J(P^{(0)}, P^{(i)}) &= \mathbb{E}_{P^{(i)}}\log\frac{ P^{(i)}(X_i|X_0)}{ P^{(0)}(X_i|X_0)} + \mathbb{E}_{P^{(0)}}\log\frac{P^{(0)}(X_i|X_0)}{ P^{(i)}(X_i|X_0)}
\end{align*}
Therefore to compute $J(P^{(0)}, P^{(i)})$, we need $\mathbb{E}_{ P^{(i)}}\log P^{(i)}(X_i|X_0)$, $\mathbb{E}_{ P^{(i)}}\log P^{(0)}(X_i|X_0)$, $\mathbb{E}_{ P^{(0)}}\log P^{(0)}(X_i|X_0)$ and $\mathbb{E}_{ P^{(0)}}\log P^{(i)}(X_i|X_0)$. We first calculate $\mathbb{E}_{ P^{(i)}}\log P^{(i)}(X_i|X_0)$. Note that $P^{(i)}(X_i= x_i|X_0=x_0)$ takes only 2 values - $\alpha+(1-\alpha)/k$(whenever $x_i = x_0$, that is, for $k$ combinations of $x_i$, $x_0$), $(1-\alpha)/k$ (whenever $X_i\neq X_0$, that is, for $k^2-k$ combinations of $x_i$, $x_0$).
\begin{align*}
    \mathbb{E}_{ P^{(i)}}\log P^{(i)}(X_i|X_0) =& \sum_{x_i, x_0\in \sets\times \sets}P^{(i)}(X_i = x_i, X_0 = x_0)\log P^{(i)}(X_i = x_i|X_0 = x_0)\\
    =& \sum_{x_i=x_0}P^{(i)}(X_i = x_i, X_0 = x_0)\log P^{(i)}(X_i = x_i|X_0 = x_0) \\
    &+ \sum_{x_i\neq x_0}P^{(i)}(X_i = x_i, X_0 = x_0)\log P^{(i)}(X_i = x_i|X_0 = x_0)\\
    =& \sum_{x_i=x_0}P^{(i)}(X_i = x_i|X_0 = x_0)P^{(i)}(X_0 = x_0)\log P^{(i)}(X_i = x_i|X_0 = x_0) \\
    &+ \sum_{x_i\neq x_0}P^{(i)}(X_i = x_i|X_0 = x_0)P^{(i)}(X_0 = x_0)\log P^{(i)}(X_i = x_i|X_0 = x_0)\\
    =& k \left(\alpha + \frac{1-\alpha}{k}\right) \frac{1}{k} \log\left(\alpha + \frac{1-\alpha}{k}\right)+k(k-1)\frac{1-\alpha}{k}\frac{1}{k}\log\left(\frac{1-\alpha}{k}\right)\\
    =& \left(\alpha + \frac{1-\alpha}{k}\right)\log\left(\alpha + \frac{1-\alpha}{k}\right) + \frac{k-1}{k}(1-\alpha)\log\left(\frac{1-\alpha}{k}\right).
\end{align*}

We next calculate $\mathbb{E}_{ P^{(i)}}\log P^{(0)}(X_i|X_0)$. $P^{(0)}(X_i|X_0)$ takes 3 different values - $\left(\alpha + \frac{1-\alpha}{k}-\delta\right)$ (for $k$ combinations of $x_i$, $x_0$), $\frac{1-\alpha}{k}+\delta$ (for $k$ combinations of $x_i$, $x_0$), $\frac{1-\alpha}{k}$ (for $k^2-2k$ combinations of $x_i$, $x_0$).
\begin{align*}
    \mathbb{E}_{ P^{(i)}}\log P^{(0)}(X_i|X_0) =& k \left(\alpha + \frac{1-\alpha}{k}\right) \frac{1}{k} \log\left(\alpha + \frac{1-\alpha}{k}-\delta\right) + k \left(\frac{1-\alpha}{k}\right) \frac{1}{k} \log\left(\frac{1-\alpha}{k}+\delta\right)+\\
    &+k(k-2)\frac{1-\alpha}{k}\frac{1}{k}\log\left(\frac{1-\alpha}{k}\right)\\
    =& \left(\alpha + \frac{1-\alpha}{k}\right)\log\left(\alpha + \frac{1-\alpha}{k}-\delta\right) + \frac{1-\alpha}{k}\log\left(\frac{1-\alpha}{k}+\delta\right)\\
    &+ \frac{k-2}{k}(1-\alpha)\log\left(\frac{1-\alpha}{k}\right)\\
\end{align*}

Evaluating the remaining terms on similar lines gives us:
\begin{align*}
    \mathbb{E}_{X_i,X_0\sim P^{(0)}}\log P^{(0)}(X_i|X_0) =& \left(\alpha + \frac{1-\alpha}{k}-\delta\right)\log\left(\alpha + \frac{1-\alpha}{k}-\delta\right) \\ &+\left(\frac{1-\alpha}{k}+\delta\right)\log\left(\frac{1-\alpha}{k}+\delta\right)
    + \frac{k-2}{k}(1-\alpha)\log\left(\frac{1-\alpha}{k}\right),\\
    \mathbb{E}_{X_i,X_0\sim P^{(0)}}\log P^{(i)}(X_i|X_0) =& \left(\alpha + \frac{1-\alpha}{k}-\delta\right)\log\left(\alpha + \frac{1-\alpha}{k}\right) \\
    &+ \left(\frac{k-1}{k}(1-\alpha)+\delta\right)\log\left(\frac{1-\alpha}{k}\right).
\end{align*}

This gives us:
\begin{align*}\label{eq:kl_non_noisy}
    J(P^{(0)}, P^{(i)}) =& \delta\left[\log\left(1+\frac{k\delta}{1-\alpha}\right) - \log\left(1-\frac{k\delta}{k\alpha+(1-\alpha)}\right)\right]\\
    &\leq  k\delta^2\left(\frac{1}{1-\alpha} + \frac{1}{1+(k-1)\alpha}\right)\\
    &\leq  \tfrac{(k-1)}{8k(k-3)\alpha^2} \left(\frac{1}{1-\alpha} + \frac{1}{1+(k-1)\alpha}\right)\times t_0^2,  \quad \text{ for }  t_0 \leq \tfrac{k\sqrt{k-3}\alpha^2}{\sqrt{2(k-1)}}, k \geq 4.
\end{align*}
The second last inequality holds as for $\log((1+ax)/(1-bx)) \leq (a+b)x$ for $x>0$, $a>0$, $b>0$, and $b\leq a$.

We now reason about the final inequality.
We have $Q^2(x) \geq \tfrac{2(k-3)k^2}{(k-1)}\delta^2(\alpha - \delta)^2$ for $k\geq 4$. If we have $\delta < \alpha/4$ then we have  $Q^2(x) \geq \tfrac{(k-3)k^2}{8(k-1)}\delta^2 \alpha^2$. But we are dealing with the situation when $Q^2(x) \geq t_0^2$. This means we must choose $\delta$ in a way such that 
$t_0^2 \leq \tfrac{(k-3)k^2}{8(k-1)}\delta^2\alpha^2$. Let $\delta = \tfrac{\sqrt{(k-1)}}{k\sqrt{8(k-3)}\alpha} t_0$. This choice satisfies $\delta \leq \alpha /4$ for $t_0 \leq \tfrac{k\sqrt{(k-3)}\alpha^2}{\sqrt{2(k-1)}}$. Hence, replacing $\tfrac{\sqrt{(k-1)}}{k\sqrt{8(k-3)}\alpha} t_0$  gives the final inequality for the symmetrized KL divergence above.

As we have $\delta \leq \alpha/4$ and $k \geq 4$, we can simplify the determinant term as 

\begin{align*}
    &\det(P^{(i)}_{i|0}) = \alpha^{(k-1)}\left( \left(1 - \tfrac{\delta}{\alpha}\right)^k - \left(\tfrac{-\delta}{\alpha}\right)^k \right)\\
    &\det(P^{(i)}_{i|0}) \leq \alpha^{(k-1)},\quad \det(P^{(i)}_{i|0}) \geq \alpha^{(k-1)}\tfrac{3^k-1}{4^k}
\end{align*}
Since the distance is bounded by $d_{min}$ and $d_{max}$, it enforces:
\begin{align*}
    &d_{max} \geq -(k-1)\log(\alpha) - \log(\tfrac{3^k-1}{4^k}) \geq -(k-1)\log(\alpha) - k\log\left(\tfrac{3}{4}\right), \\
    &d_{min} \leq -(k-1)\log(\alpha)\\
    & \alpha \geq 2\exp(- d_{\max}/(k-1)),\, \alpha \leq \exp(-d_{\min}/(k-1)).
\end{align*}

If we use $\alpha = \exp(-d_{\min}/(k-1))$ for our construction, the symmetrized KL divergence in terms of the distance bounds, for $k\geq 4$ and $t_0 \leq \tfrac{k\sqrt{k-3}\alpha^2}{\sqrt{2(k-1)}}$, is 
\begin{align*}
    J(P^{(0)}, P^{(i)}) \leq & \tfrac{(k-1)}{8k(k-3)\alpha^2} \left(\tfrac{1}{1-\alpha} + \tfrac{1}{1+(k-1)\alpha}\right)\times t_0^2\\
    \leq & \tfrac{(k-1)}{8k(k-3)\exp(-2 d_{\min}/(k-1))} \left(1+ \tfrac{1}{1-\exp(-d_{\min}/(k-1))}\right)\times t_0^2
\end{align*}

\paragraph{Lower Bound Proof - Part II:}
We now derive the second part of Theorem~\ref{th:lb}, thus concluding its proof. 

Plugging the above symmetrized KL bound in Lemma~\ref{lemm:fano} we obtain that for a probability error of at most  $\delta > 0$ we require at least $N$ samples where
\begin{align*}
    N &> (1- \delta + \tfrac{1}{\log(n)}) \frac{\log(n)}{\tfrac{n}{n+1} \tfrac{(k-1)}{8k(k-3)\exp(-2 d_{\min}/(k-1))} \left(1+ \tfrac{1}{1-\exp(-d_{\min}/(k-1))}\right)\times t_0^2 }\\
    &\geq  \frac{(1- \delta)\exp(-\tfrac{2d_{\min}}{k-1})(1-\exp(-\tfrac{d_{\min}}{k-1})) 8k(k-3) \log(n)}{(k-1)(2-\exp(-\tfrac{d_{\min}}{k-1})) t_0^2}
\end{align*}
Therefore, we have $N = \Omega\left(\frac{(1- \delta)\exp(-\tfrac{2d_{\min}}{k-1})(1-\exp(-\tfrac{d_{\min}}{k-1})) k \log(n)}{ t_0^2}\right)$

Instead using $\alpha = \tfrac{1}{2}\exp(-d_{\max}/(k-1))$ in our construction, following similar steps, we obtain 
$$N= \Omega\left(\frac{(1- \delta)\exp(-\tfrac{2d_{\max}}{k-1})(1-\exp(-\tfrac{d_{\max}}{k-1})) k \log(n)}{ t_0^2}\right).$$

Combining these two we obtain the final lower bound in this setting ($k\geq 4$ and $t_0 \leq \tfrac{\sqrt{3}}{4\sqrt{10}}k\exp(-2\tfrac{d_{\max}}{k-1})$)as 
$$N= \Omega\left(\max_{d\in \{d_{\max}, d_{\min}\}}\frac{(1- \delta)\exp(-\tfrac{2d}{k-1})(1-\exp(-\tfrac{d}{k-1})) k \log(n)}{ t_0^2}\right).$$


\section{Experiments}\label{ap:exps}


We present the performance of our algorithm for the perturbed symmetric model. \textit{All the experiments in this section are for $k=4$}.

\subsection{Varying $q_{max}$}
Now, we study the impact of the probability of error on the performance of the algorithm.
\paragraph{Setting:} 
(i) Number of nodes = 7.\\
(ii) Graph Shape = \{Chain, Star\}\\
(iii) Distance of all the adjacent nodes = $\exp(-0.7)$. \\
(iv) Error probability is uniformly sampled from $[0,q_max]$, where, $q_{max}\in\{0, 0.2, 0.4\}$.\\
(v) $\delta = 0.04$\\
(vi) Assume access to $q_{max}$, $d_{min}$ but not to $d_{max}$, $t_0$.\\
(vii) Number of iterations = 100\\
\textbf{Takeaway:} The convergence is slower for higher $q_{max}$ as demonstrated in Figure \ref{fig:app_k_4_q_max}.
\begin{figure}
    \centering
    \includegraphics[scale = 0.3]{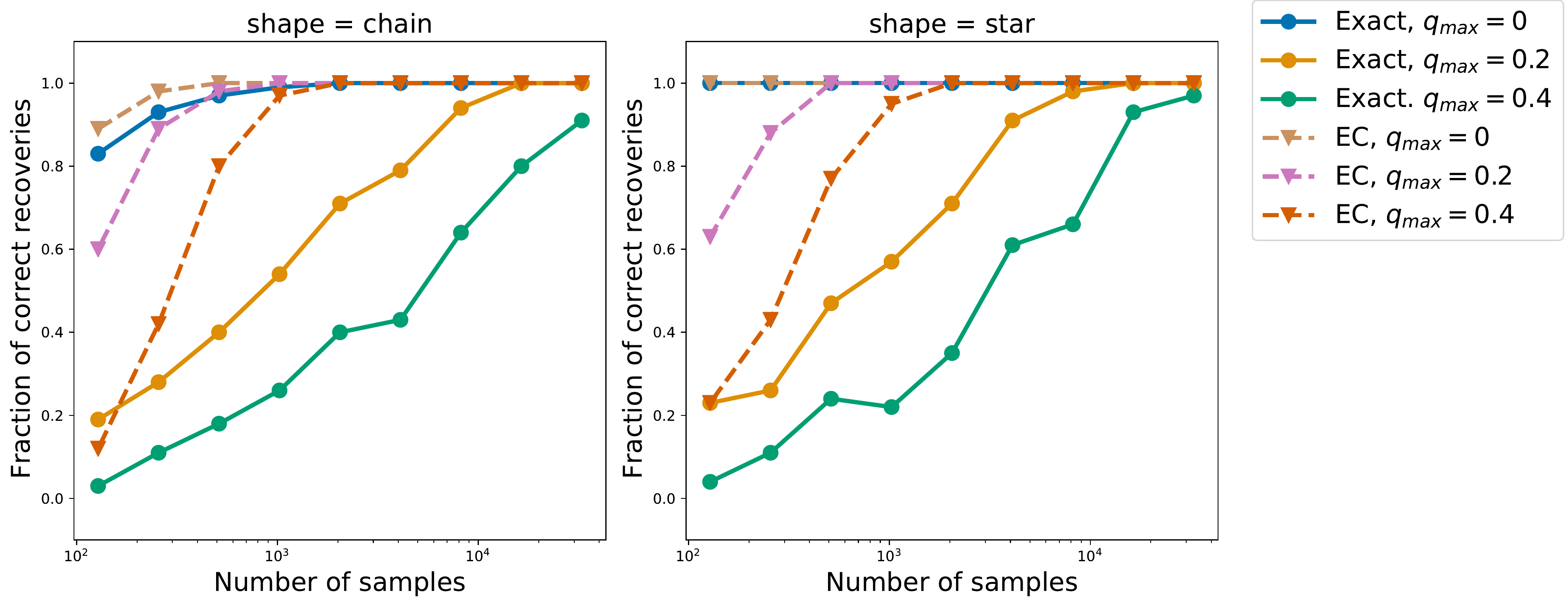}
    \caption{Comparing the performance of our
    algorithm for different values of $q_{max}\in \{0, 0.2, 0.4\}$ and different graph shapes - chain, star. Setting: $d_{min} = d_{max} = \exp(-0.7)$, $\delta = 0.04$ $\#$ of nodes$=7$. We provide results for two cases: i) when the exact underlying tree is recovered, ii) when a tree from the equivalence class is recovered.}
    \label{fig:app_k_4_q_max}
\end{figure}

\subsection{Varying $d$}
Finally, we present the results for different values of $d$.
\paragraph{Setting:} 
(i) Number of nodes = 7.\\
(ii) Graph Shape = \{Chain, Star\}.\\
(iii) Distance of all the adjacent nodes  $\in \{\exp(-0.5),\exp(-0.7),\exp(-0.92)\}$. \\
(iv) Error probability is uniformly sampled from $[0, 0.2]$.\\
(v) $\delta = 0.02$\\
(vi) Assume access to $q_{max}$, $d_{min}$ but not to $d_{max}$, $t_0$.\\
(vii) Number of iterations = 100\\
\textbf{Takeaway:} The algorithm performs the best for intermediate values of $d$. When the distance is too high or too low, the convergence is slower. Interestingly, the performance for exact recovery and equivalence class recovery show different trends - exact recovery is more difficult when the distance is large whereas the recovery of the  equivalence class is more difficult when the distance is small.
The results are presented in Figure \ref{fig:app_k_4_d}.
\begin{figure}
    \centering
    \includegraphics[scale = 0.3]{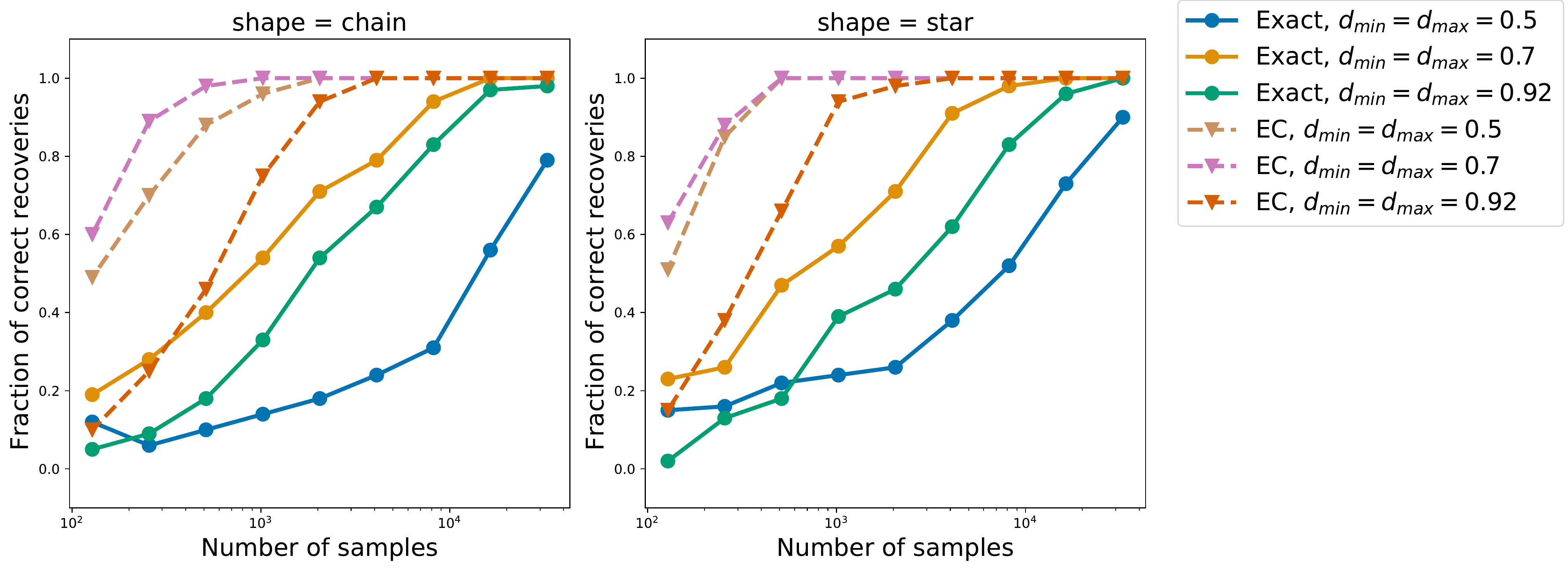}
    \caption{Comparing the performance of our
    algorithm for different values of $d\in \{\exp(-0.5), \exp(-0.7), \exp(-0.92)\}$ and different graph shapes - chain, star. Setting: $q_{max} = 0.2$, $\delta = 0.02$ $\#$ of nodes$=7$. We provide results for two cases: i) when the exact underlying tree is recovered, ii) when a tree from the equivalence class is recovered.}
    \label{fig:app_k_4_d}
\end{figure}

\end{document}